\documentclass[10pt]{article}
\usepackage[margin=1in]{geometry}
\usepackage{yub}
\usepackage[numbers]{natbib}
\usepackage{parskip}
\usepackage{enumitem}
\usepackage{tikz}
\usepackage{hyperref}       %

\usepackage{algorithm}
\usepackage[noend]{algorithmic}

\newcommand{\alglinelabel}{%
  \addtocounter{ALC@line}{-1}%
  \refstepcounter{ALC@line}%
  \label%
}

\allowdisplaybreaks

\usepackage{makecell}
\definecolor{light-gray}{gray}{0.85}
\usepackage{colortbl}

\newcommand{\kl}{{\rm KL}}

\newcommand{\cO}{\mc{O}}
\newcommand{\tO}{\wt{\mc{O}}}

\newcommand{\ith}{$i^{\text{th}}$ }
\renewcommand{\epsilon}{\eps}

\newcommand{\efce}{{\sf EFCE}}
\newcommand{\trig}{{\sf Tr}}
\newcommand{\bal}{{\sf bal}}
\newcommand{\efceregret}{{\rm Reg}^{\trig}}
\newcommand{\phiregret}{{\rm Reg}^\Phi}

\renewcommand{\bar}{\overline}
\newcommand{\conv}{{\rm conv}}
\newcommand{\ext}{{\sf ext}}

\def\cS{{\mathcal S}}
\def\cX{{\mathcal X}}
\def\cY{{\mathcal Y}}
\def\cA{{\mathcal A}}
\def\cB{{\mathcal B}}
\def\cF{{\mathcal F}}

\def\cC{{\mathcal C}}

\def\cB{{\mathcal B}}

\mathchardef\mhyphen="2D

\newenvironment{talign*}
 {\csname align*\endcsname}
 {\endalign}

\hypersetup{
    colorlinks,
    linkcolor={blue!70!black},
    citecolor={blue!70!black},
}
\colorlet{linkequation}{blue}

\def\shownotes{0}  %
\ifnum\shownotes=1
\newcommand{\authnote}[2]{{\scriptsize $\ll$\textsf{#1 notes: #2}$\gg$}}
\else
\newcommand{\authnote}[2]{}
\fi

\title{Efficient Phi-Regret Minimization in Extensive-Form Games via Online Mirror Descent}

\author{
  Yu Bai\thanks{Salesforce Research. Email: \texttt{yu.bai@salesforce.com}}
  \and
  Chi Jin\thanks{Princeton University. Email: \texttt{chij@princeton.edu}}
  \and 
  Song Mei\thanks{UC Berkeley. Email: \texttt{songmei@berkeley.edu}}
  \and
  Ziang Song\thanks{Peking University. Email: \texttt{songziang@pku.edu.cn}}
  \and
  Tiancheng Yu\thanks{MIT. Email: \texttt{yutc@mit.edu}}
}

\begin{document}

\maketitle

\begin{abstract}
A conceptually appealing approach for learning Extensive-Form Games (EFGs) is to convert them to Normal-Form Games (NFGs). This approach enables us to directly translate state-of-the-art techniques and analyses in NFGs to learning EFGs, but typically suffers from computational intractability due to the exponential blow-up of the game size introduced by the conversion. In this paper, we address this problem in natural and important setups for the \emph{$\Phi$-Hedge} algorithm---A generic algorithm capable of learning a large class of equilibria for NFGs. We show that $\Phi$-Hedge can be directly used to learn Nash Equilibria (zero-sum settings), Normal-Form Coarse Correlated Equilibria (NFCCE), and Extensive-Form Correlated Equilibria (EFCE) in EFGs. We prove that, in those settings, the \emph{$\Phi$-Hedge} algorithms are equivalent to standard Online Mirror Descent (OMD) algorithms for EFGs with suitable dilated regularizers, and run in polynomial time. This new connection further allows us to design and analyze a new class of OMD algorithms based on modifying its log-partition function. In particular, we design an improved algorithm with balancing techniques that achieves a sharp $\widetilde{\mathcal{O}}(\sqrt{XAT})$ EFCE-regret under bandit-feedback in an EFG with $X$ information sets, $A$ actions, and $T$ episodes. To our best knowledge, this is the first such rate and matches the information-theoretic lower bound.

\end{abstract}

\section{Introduction}

Extensive form games (EFGs) is a natural formulation for multi-player games with imperfect information and sequential play, which models real-world games such as Poker~\citep{brown2018superhuman,brown2019superhuman}, Bridge~\citep{tian2020joint}, Scotland Yard~\citep{schmid2021player}, Diplomacy~\citep{bakhtin2021no} and has many other important applications such as cybersecurity~\citep{kakkad2019comparative}, auction~\citep{myerson1981optimal}, marketing~\citep{herbig1991game}. %
In multi-player general-sum EFGs, computing an approximate Nash equilibrium (NE)~\citep{nash1950equilibrium} is PPAD-hard~\citep{daskalakis2009complexity} and thus likely intractable. A reasonable and computationally tractable solution concept in general-sum EFGs is the \emph{extensive-form correlated equilibria} (EFCE) \citep{von2008extensive,gordon2008no,celli2020no,farina2021simple}. It is known that, as long as each player runs an uncoupled dynamics minimizing a suitable EFCE-regret, their average joint policy will converge to an EFCE~\citep{greenwald2003general}.   %

Existing algorithms of minimizing the EFCE-regret are mostly built upon the \emph{regret decomposition} techniques~\citep{zinkevich2007regret}, which utilize the structure of the game and the set of policy modifications~\citep{celli2020no,morrill2021efficient,farina2021simple,song2022sample}. For example,~\citet{morrill2021efficient}~decomposes the EFCE-regret to local regrets at each information set (infoset) with each of them handled by a local regret minimizer;~\citet{farina2021simple} utilizes the trigger structure of the policy modification set to decompose the regret to external-like regrets.

There are at least two alternative approaches to designing regret minimization algorithms for EFGs. The first is to convert a EFG to a normal-form game (NFG) and use NFG-based algorithms such as $\Phi$-Hedge~\citep{greenwald2003general}. This approach typically admits simple algorithm designs and sharp regret bounds by directly translating existing results in NFGs~\citep{stoltz2007learning}. However, the conversion introduces an exponential blow-up in the game size, and makes such algorithms computationally intractable in general. The computational efficiency of these NFG-based algorithms is recently investigated by~\citet{farina2022kernelized} in the external regret minimization problem, who provided an efficient implementation of an NFG-based algorithm using ``kernel tricks''. The second is to use Online Mirror Descent (OMD) algorithms via suitably designed regularizers over the parameter space. This approach has been successfully implemented in minimizing the external regret~\citep{kroer2015faster} but not yet generalized to the EFCE-regret, as it remains unclear how to design suitable regularizers for the policy modification space.

In this paper, we develop the first line of EFCE-regret minimization algorithms along both lines of approaches above, and identify an equivalence between them. We consider EFCE-regret minimization in EFGs with $X$ infosets, $A$ actions, and maximum $L_1$-norm of sequence-form policies bounded by $\lone{\Pi}$ (cf. Section~\ref{sec:EFG-prelim} for the formal definition). Our contributions can be summarized as follows:
\begin{itemize}[leftmargin=1.5pc]
\item We present an efficient implementation of the $\Phi$-Hedge algorithm for minimizing the extensive-form trigger regret, by recursively evaluating the gradient of a log-partition function (Section~\ref{sec:phi-hedge-implementation}). The implementation further reveals that this algorithm (via reparametrization) is equivalent to an OMD algorithm with dilated regularizers, which we term as EFCE-OMD (Section~\ref{sec:efce-omd}).

\item We show that EFCE-OMD achieves trigger regret bound $\tO(\sqrt{\lone{\Pi}T})$ under full feedback and $\tO(\sqrt{X\lone{\Pi}AT})$ under bandit feedback (Section~\ref{sec:phi-hedge-regret}). Notably, the proofs are done using the corresponding NFG analysis straightforwardly, and is independent of the actual implementation.

\item We design an improved algorithm Balanced EFCE-OMD, and show that it achieves a sharp $\tO(\sqrt{XAT})$ trigger regret under bandit feedback (Section~\ref{sec:sharp-rate}). This improves over EFCE-OMD by a factor of $\lone{\Pi}$ and is the first to match the information-theoretic lower bound. The algorithm works by modifying the above log-partition function using a variety of \emph{balancing} techniques, and is equivalent to another OMD algorithm (but no longer an NFG algorithm).

\item As another example of our framework, we show that the $\Phi$-Hedge algorithm for vanilla (external) regret minimization in EFGs, along with its efficient implementation via ``kernelization'' developed recently in~\citep{farina2022kernelized}, is actually equivalent to standard OMD with dilated entropy (Section~\ref{sec:equivalence}).

\end{itemize}

\subsection{Related work}

\paragraph{$\Phi$-regret minimization and correlated equilibrium}

The $\Phi$-regret minimization framework was introduced in \citet{greenwald2003general} and \citet{stoltz2007learning}. In particular, \citet{greenwald2003general} showed that uncoupled no $\Phi$-regret dynamics leads to $\Phi$-correlated equilibria, a generalized notion of correlated equilibria introduced by~\citet{aumann1974subjectivity}. \citet{stoltz2007learning} then developed a family of $\Phi$-regret minimization algorithms using the fixed-point method (including the $\Phi$-Hedge algorithm considered in this paper), and derived explicit regret bounds. Two important special cases of $\Phi$-regret are the internal regret and swap regret in normal-form games \cite{stoltz2005internal, blum2007external}. A recent line of work developed algorithms with $\cO({\rm polylog} T)$ swap regret bound in normal-form games \cite{anagnostides2021near, anagnostides2022uncoupled}.

\paragraph{Regret minimization in EFG from full feedback}

A line of work considers external regret minimization in EFGs from full feedback \citep{zinkevich2007regret,celli2019learning,burch2019revisiting,farina2020faster,zhou2020lazy}. In particular, \citet{zhou2020lazy} achieves $\tO(\sqrt{ X T})$ external regret. The recent work of \citet{farina2022kernelized}~develops the first algorithm to achieve $\tO(\lone{\Pi} {\rm polylog} T)$ external regret in EFGs by converting it to an NFG and invoking the fast rate of Optimistic Hedge~\citep{daskalakis2021near}, along with an efficient implementation via the ``kernel trick''. Our $\Phi$-regret framework covers their algorithm as a special case, and we further show that their algorithm (along with its efficient implementation) is equivalent to the standard OMD with dilated entropy.

The notion of Extensive-Form Correlated Equilibria (EFCE) in EFGs was introduced in \citet{von2008extensive}. Optimization-based algorithms for computing computing EFCEs in multi-player EFGs from full feedback have been proposed in \citet{huang2008computing, farina2019correlation}.

\citet{gordon2008no}~first proposed to use uncoupled EFCE-regret minimization dynamics to compute EFCE; however, they do not explain how to efficiently implement each iteration of the dynamics. Recent works \cite{celli2020no, farina2021simple,morrill2021efficient, song2022sample} developed uncoupled EFCE regret minimization learning dynamics with efficient implementation; All of these algorithms are based on counterfactual regret decomposition~\citep{zinkevich2007regret} and minimizing each trigger regret (first considered by~\citet{dudik2012sampling,gordon2008no}) using a different regret minimizer. \citet{celli2020no} decomposed the regret to each laminar subtree, but they did not give an explicit regret bound. \citet{farina2021simple} decomposed the regret to each trigger sequence and used CFR type algorithm to minimize the regret on each trigger sequence and achieved an $\tilde O(\sqrt{X^2 T})$ EFCE-regret bound. \citet{morrill2021efficient, song2022sample} decomposed the regret to each information set and use regret minimization algorithms with time-selection functions~\cite{blum2007external, khot2008minimizing} to minimize the regret on each information set, giving $\tO(\sqrt{X^2 T})$ and $\tO(\sqrt{X T})$ regret bounds respectively. In this paper, we show that the simple $\Phi$-Hedge algorithm, which has an efficient implementation and an intuitive interpretation, can also achieve the state-of-art $\tO(\sqrt{X T})$ regret bound in the full feedback setting.

\paragraph{Regret minimization in EFG from bandit feedback}

Minimizing the external regret in EFGs from bandit feedback is considered in a more recent line of work~\citep{lanctot2009monte,farina2020stochastic, farina2021model, farina2021bandit,zhou2019posterior,zhang2021finding,kozuno2021model,bai2022near}. \citet{dudik2009sampling} consider sample-based learning of EFCE in succinct extensive-form games; however, their algorithm relies on an approximate Markov-Chain Monte-Carlo sampling subroutine that does not lead to a sample complexity guarantee. 

A concurrent work by~\citet{song2022sample} also achieves $\tO(X / \eps^2)$ sample complexity for learning EFCE under bandit feedback (when only highlighting $X$) using the Balanced $K$-EFR algorithm. Our work achieves the same linear in $X$ sample complexity, but using a very different algorithm (Balanced EFCE-OMD). We also remark that the algorithm of~\citep{song2022sample} cannot minimize the EFCE-regret against adversarial opponents from bandit feedback like our algorithm, as their algorithm requires playing multiple episodes against a fixed opponent, which is infeasible when the opponent is adversarial.

\section{Preliminaries}

\subsection{$\Phi$-regret minimization and $\Phi$-Hedge algorithm}
\label{sec:phi-reg-min}

Consider a generic linear regret minimization problem on a \emph{policy set} $\Pi\subset \R^d_{\ge 0}$ with respect to a \emph{policy modification set} $\Phi\subset \R^{d\times d}$. Here $\Pi$ is a convex compact subset of $\R^d$, and $\Phi$ is a convex compact subset of $\R^{d\times d}$, where each $\phi\in\Phi$ is a \emph{policy modification function} which is a linear transformation from $\R^d$ to $\R^d$ that maps $\Pi$ to itself ($\phi(\Pi)\subseteq\Pi$). 
For any algorithm that plays policies $\set{\mu^t}_{t=1}^T$ within $T$ rounds and receives loss functions $\set{\ell^t}_{t=1}^T\subset \R_{\ge 0}^d$, the $\Phi$-regret is defined as
\begin{align}
\label{equation:phi-regret}
 \phiregret(T) \defeq \sup_{\phi\in\Phi} \sum_{t=1}^T \<\mu^t - \phi\mu^t, \ell^t\>.
\end{align}
The $\Phi$-regret subsumes the vanilla regret (i.e. external regret) as a special case by taking $\Phi$ to be the set of all constant modifications $\Phi^{\ext}\defeq \{\phi_{\mu_\star}:\mu_\star\in\Pi\}$ where $\phi_{\mu^\star}\mu = \mu^\star$ for all $\mu\in\Pi$. Another widely studied example is the \emph{swap regret}~\citep{blum2007external} (and the closely related \emph{internal regret}~\citep{foster1998asymptotic}) for normal-form games, where $\Pi=\Delta_{d}$ is the probability simplex over $d$ actions, and $\Phi$ is the set of all stochastic matrices (i.e. those mapping $\Delta_{d}$ to itself). A primary motivation for minimizing the $\Phi$-regret is for computing various types of \emph{Correlated Equilibria} (CEs) in multi-player games using the online-to-batch conversion (see e.g.~\citep{cesa2006prediction}), which has been established in many games and has been a cornerstone in the online learning and games literature.

\paragraph{$\Phi$-Hedge algorithm}
A widely used strategy for minimizing the $\Phi$-regret is to use any (black-box) linear regret minimization algorithm on the $\Phi$ set to produce a sequence of $\{\phi^t\}_{t=1}^T\subset \Phi$, combined with the \emph{fixed point technique} (e.g.~\citep{stoltz2005internal})---Output policy $\mu^t$ that satisfies the fixed-point equation $\phi^t\mu^t=\mu^t$ in each round $t$. In the common scenario where $\Phi$ is the convex hull of a finite number of \emph{vertices}, i.e. $\Phi=\conv(\Phi_0)$ where $\Phi_0$ is a finite subset of $\Phi$, a standard regret minimization algorithm over $\Phi$ is Hedge (a.k.a. Exponential Weights)~\citep{arora2012multiplicative}, leading to the \emph{$\Phi$-Hedge} algorithm (Algorithm~\ref{algorithm:phi-regret-minimization}).

\begin{algorithm}[h]
   \caption{$\Phi$-Hedge}
   \label{algorithm:phi-regret-minimization}
   \begin{algorithmic}[1]
     \REQUIRE Finite vertex set $\Phi_0\subset \R^{d\times d}$ such that $\conv(\Phi_0)=\Phi$; Learning rate $\eta$.
     \STATE Initialize $p^1\in\Delta_{\Phi_0}$ with $p^1_\phi = 1 / |\Phi_0|$ for $\phi\in\Phi_0$. 
     \FOR{iteration $t=1,\dots,T$}
     \STATE Compute $\phi^{t} = \sum_{\phi\in\Phi_0} p_\phi^{t} \phi$. \alglinelabel{line:phi-hedge-phi-update}
     \STATE Set policy $\mu^t$ to be the fixed point of equation $\mu^t = \phi^t \mu^t$. 
     \alglinelabel{line:fixed-point-phi-hedge}
     \STATE Receive loss function $\ell^t\in\R^d_{\ge 0}$, suffer loss $\<\mu^t, \ell^t\>$.
     \STATE Update $p_\phi^{t+1} \propto_\phi p_\phi^{t} \cdot \exp\{ - \eta \< \phi \mu^t, \ell^t \> \}$. \alglinelabel{line:phi-hedge-p-update}
     \ENDFOR
 \end{algorithmic}
\end{algorithm}

It is a standard result (\citep{stoltz2007learning}, see also Lemma~\ref{lemma:regret-bound-phi-hedge}) that Algorithm~\ref{algorithm:phi-regret-minimization} achieves $\Phi$-regret bound
\begin{align}\label{eqn:phiregret_bound_simple}
    \phiregret(T) \le \frac{\log |\Phi_0|}{\eta} +  \frac{\eta}{2} \sum_{t=1}^T \sum_{\phi\in\Phi_0} p_\phi^{t} \paren{\langle \phi \mu^t, \ell^t \rangle}^2.
\end{align}
By choosing $\eta>0$, this result implies a quite desirable bound $\phiregret(T)\le L\sqrt{2\log|\Phi_0| \cdot T}$ in the full-feedback setting (assuming bounded loss $\<\phi\mu^t, \ell^t\>\le L$), and can also be used to prove regret bounds in the bandit-feedback setting.

\subsection{Extensive-form games (EFGs) and extensive-form trigger regret}\label{sec:EFG-prelim}

\def\TFSDP{{\rm TFSDP}}

In this paper, we consider $m$-player imperfect-information extensive-form games (EFGs) with perfect-recall (see Appendix~\ref{appendix:efg-def} for detailed definitions). For the purpose of this work, we consider an alternative formulation of EFGs---Tree-Form Adversarial Markov Decision Processes (TFAMDP). This model is equivalent to studying EFGs from the perspective of a single player, while treating all other players as adversaries who can change both transitions and rewards in each round.

\paragraph{Tree-form adversarial MDP} We consider an episodic, tabular TFAMDP which consists of the followings $(H, \{ \cX_h\}_{h \in [H]}, \cA, \mathcal{T}, \{ p_{h}^t\}_{h \in \{ 0 \} \cup [H], t \ge 1}, \{ R_h^t \}_{h \in [H], t \ge 1} )$. Here $H \in \mathbb{N}_+$ is the horizon length; $\cX_h$ is the space of information sets (henceforth \emph{infosets}) at step $h$ with size $\vert \cX_h \vert = X_h$ and $\sum_{h = 1}^H X_h = X$; $\cA$ is the action space with size $\abs{\cA}=A$. Next, $\mathcal{T} = \{\cC(x, a)\}_{(x, a) \in \cX \times \cA}$ defines the tree structure over the infosets and actions, where $\cC(x_h, a_h) \subset \cX_{h+1}$ denotes the set of immediate children of $(x_h, a_h)$. Furthermore, $\{\cC(x_h, a_h)\}_{(x_h, a_h) \in \cX_h \times \cA}$ forms a partition of $\cX_{h+1}$. It directly follows from the tree structure of TFAMDP that the player has \emph{perfect recall}, i.e., for any infoset $x_h \in \cX_h$, there is a unique history $(x_1, a_1, \ldots ,x_{h-1}, a_{h-1})$ that leads to $x_h$. Furthermore, $p_0^t(\cdot) \in \Delta_{\cX_1}$ is the initial distribution over $\cX_1$ at episode $t$; $p_h^t( \cdot | x_h, a_h)$ is the transition probability from $(x_h, a_h)$ to its immediate children $\cC(x_h, a_h)$ at episode $t$; $R^t_h( \cdot |x_h, a_h)$ is the distribution of the stochastic reward $r \in [0, 1]$ received at $(x_h, a_h)$ at episode $t$, with expectation $\bar{R}^t_h(x_h, a_h)$.

At the beginning of episode $t$, an adversary will first choose the initial distribution $p_0^t$, transition $\{ p_{h}^t\}_{h \in [H]}$, and reward distribution $\{ R_h^t \}_{h \in [H]}$. Then in the \emph{bandit feedback} setting, at each step $h$, the player observes the current infoset $x_h$, takes an action $a_h$, receives a bandit feedback of the reward $r^t_h\sim R^t_h(\cdot|x_h, a_h)$, and the environment transitions to the next state $x_{h+1}\sim p_h^t(\cdot|x_h, a_h)$.

\paragraph{Policies} 
We use $\mu = \set{\mu_{h}(\cdot|x_{h})}_{h \in [H], x_{h}\in\cX_{h}}$ to denote a policy, where each $\mu_{h}(\cdot|x_{h})\in\Delta_\cA$ is the action distribution at infoset $x_{h}$. We say $\mu$ is a \emph{deterministic} policy if $\mu_{h}(\cdot|x_{h})$ takes some single action with probability $1$ for any $(h,x_{h})$. Let $\Pi$ denote the set of all possible policies. We denote the \emph{sequence form} representation of policy $\mu \in \Pi$ by
\begin{align}
\textstyle \mu_{1:h}(x_h, a_{h}) \defeq \prod_{h' = 1}^h \mu_{h'}(a_{h'} \vert x_{h'}),  \label{eqn:sequence-form-policy}
\end{align}
where $(x_1, a_1, \ldots, x_{h-1}, a_{h-1})$ is the unique history of $x_h$. 
We also identify $\mu$ as a vector in $\R^{XA}_{\ge 0}$, whose $(x_h, a_h)$-th entry is equal to its sequence form $\mu_{1:h}(x_h, a_{h})$. Let $\lone{\Pi}\defeq \max_{\mu\in\Pi} \lone{\mu}$, which admits bound $\lone{\Pi}\le X$ but can in addition be smaller (cf. Appendix~\ref{appendix:efg-properties}).

\paragraph{Expected loss function}  Given any policy $\mu^t$ at round $t$, the total expected loss received at round $t$ (which equals to $H$ minus the total rewards within round $t$) is given by %
\[
\textstyle  \< \mu^t, \ell^t\> \defeq \sum_{h, x_h, a_h} \mu_{1:h}^t(x_h, a_h) \ell_h^t(x_h, a_h),
\]
where the loss function for the $t$-th round is given by $\ell^t = \{\ell_{h}^t(x_{h}, a_{h})\}_{h,x_{h},a_{h}}\in\R^{XA}_{\ge 0}$: 
\begin{align}
\textstyle  \ell_{h}^t(x_{h}, a_{h}) \defeq p_0^t(x_1) \prod_{h' = 1}^{h-1} p_{h'}^t(x_{h'+1}| x_{h'}, a_{h'}) [1 - \bar{R}_{h}^t(x_h, a_h)], \label{equation:l-definition} 
\end{align}
where $(x_1, a_1, \ldots, x_{h-1}, a_{h-1})$ is the unique history that leads to $x_h$. In the \emph{full feedback} setting, the learner is further capable of observing the full loss vector $\ell^t\in\R^{XA}_{\ge 0}$ at the end of each round $t$.

\paragraph{Subtree and subtree policies} 
For any $g\le h$, $x_{g} \in \cX_{g}, x_{h} \in \cX_{h}$, and any action $a_g, a_h \in \cA$, we say $x_h$ or $(x_h, a_h)$ is in the subtree rooted at $x_g$, written as $x_{h} \succeq x_{g}$ or $(x_{h}, a_h) \succeq x_{g}$, if $x_{g}$ is either equal to $x_h$ or is a part of the unique preceding history $(x_1, a_1, \ldots ,x_{h-1}, a_{h-1})$ which leads to $x_{h}$. Similarly, we say $x_h$ or $(x_{h}, a_h)$ is in the subtree of $(x_{g}, a_{g})$, written as $x_{h} \succ (x_{g}, a_{g})$ or $(x_{h}, a_h) \succeq (x_{g}, a_{g})$, if $(x_{g}, a_{g})$ is either equal to $(x_h, a_h)$ (only in the latter case), or is a part of the unique preceding history $(x_1, a_1, \ldots ,x_{h-1}, a_{h-1})$ which leads to $x_{h}$.

For any $g \in [H]$, and any infoset $x_g \in \cX_{g}$, we use $\mu^{x_g} = \{ \mu^{x_g}_{h}(\cdot | x_h) \in \Delta_{\cA}: x_h \succeq x_g \}$ to denote a subtree policy rooted at $x_g$. 
We use $\Pi^{x_g}$ and $\mc{V}^{x_g}$ to denote the set of all subtree policies and the set of all \emph{deterministic} subtree policies rooted at $x_g$. We denote the sequence form representation of $\mu^{x_g} \in \Pi^{x_g}$ by:
\begin{equation*}
\textstyle \mu^{x_g}_{g:h}(x_h,a_h) = 
\begin{cases}
\prod_{h' = g}^h \mu^{x_g}_{h'}(a_{h'}|x_{h'}) ~~~&\text{if}~ (x_h, a_h) \succeq x_g, \\
0 ~~~&\text{otherwise}.
\end{cases}
\end{equation*}
Similarly, we can also identify any subtree policy $\mu^{x_g} \in \Pi^{x_g}$ as a vector in $\R^{XA}_{\ge 0}$, whose $(x_h, a_h)$-th entry is equal to its sequence form $\mu^{x_g}_{g:h}(x_h,a_h)$ (which is non-zero only on the subtree rooted at $x_g$).

\paragraph{Extensive-form trigger regret} %
The notion of trigger regret is introduced in \citep{gordon2008no,celli2020no,farina2021simple}. An \emph{(extensive-form) trigger modification} $\phi_{x_ga_g\to m^{x_g}}$ is a policy modification that modifies any policy $\mu \in \Pi$ as follows: When $x_g$ is visited and $a_g$ is about to be taken (by $\mu$), we say $x_ga_g$ is \emph{triggered}\footnote{For notational convenience, our definition here does not include triggering at the root of the tree (i.e. $x_ga_g=\emptyset$), where the subtreee policy $m^{\emptyset}\in\Pi$ is just a policy for the entire game tree. However, all our results can be directly extended to this case without changing the rates.}, in which case the subtree policy rooted at $x_g$ is then replaced by $m^{x_g} \in \Pi^{x_g}$. One can verify that the trigger modification $\phi_{x_ga_g\to m^{x_g}}$ can be written as a linear transformation that maps from $\Pi$ to $\Pi$:
\begin{align*}
    \phi_{x_ga_g\to m^{x_g}} \defeq (I - E_{\succeq x_ga_g}) + m^{x_g} e_{x_ga_g}^\top \in \R^{XA\times XA}.
\end{align*}
Here, $E_{\succeq x_ga_g}$ is a diagonal matrix with diagonal entry $1$ at all $(x_h, a_h)$ satisfying $(x_h, a_h) \succeq (x_g, a_g)$, and zero otherwise, and
$e_{x_ga_g}\in \R^{XA}$ is an indicator vector whose only non-zero entry is $1$ at $(x_g,a_g)$. We say $\phi_{x_g a_g \to v^{x_g}}$ is a deterministic trigger modification if $v^{x_g} \in \mc{V}^{x_g}$ is a deterministic subtree policy. We denote the set of all deterministic trigger modifications and its convex hull as $\Phi^{\trig}_{0}$ and $\Phi^{\trig}$ respectively, where %
\begin{equation}\label{eqn:Phi-trig}
    \Phi^{\trig}_{0} \defeq \bigcup_{g,x_g,a_g} \,\,\bigcup_{v^{x_g}\in\mc{V}^{x_g}} \,\, \set{\phi_{x_ga_g\to v^{x_g}}}, \qquad \Phi^{\trig} = {\rm conv}\set{ \Phi^{\trig}_{0} }.
\end{equation}
The \emph{(extensive-form) trigger regret} is then defined as the difference in the total loss when comparing against the best extensive-form trigger modification in hindsight. We note that the trigger regret is a special case of $\Phi$-regret~\eqref{equation:phi-regret} with $\Phi = \Phi^{\trig}$.
\begin{definition}[Extensive-Form Trigger Regret]
For any algorithm that plays policies $\mu^t\in\Pi$ at round $t\in[T]$, the extensive-form trigger regret (also the EFCE-regret) is defined as
\begin{align}\label{eqn:efceregret}
 \efceregret(T) \defeq \max_{\phi\in\Phi^{\trig}} \sum_{t=1}^T \<\mu^t - \phi\mu^t, \ell^t\>. 
\end{align}
\end{definition}

\paragraph{From trigger regret to Extensive-Form Correlated Equilibrium (EFCE)} The importance of extensive-form trigger regret is in its connection to computing EFCE: 
By standard online-to-batch conversion \citep{celli2020no, farina2021simple}, if all players have low trigger regret (with $\efceregret_i(T)$ for the \ith player), then the average joint policy $\wb{\pi}$ is an $\epsilon$-EFCE, where $\epsilon = \max_{i\in[m]} \efceregret_i(T)/T$ (cf. Appendix~\ref{appendix:online-to-batch}). We remark in passing by taking $\Phi=\Phi^{\ext}$, low $\Phi$-regret implies learning (Normal-Form) Coarse Correlated Equilibria in EFGs, as well as Nash Equilibria in the two-player zero-sum setting~\citep{bai2022near}.

\section{Efficient $\Phi$-Hedge for Trigger Regret Minimization}
\label{sec:phi-hedge-efce}

In this section, we study the $\Phi$-Hedge algorithm (Algorithm~\ref{algorithm:phi-regret-minimization}) for minimizing the trigger regret. Naively, Algorithm~\ref{algorithm:phi-regret-minimization} requires maintaining and updating $p^t\in\Delta_{\Phi_0}$ (cf. Line~\ref{line:phi-hedge-p-update}), whose computational cost is linear in $|\Phi_0^\trig|$ which can be exponential in $X$ in the worst case\footnote{$|\Phi^{\trig}_0|$ is at least the number of deterministic policies of the game, which could be $A^{O(X)}$ in the worst case.}. We begin by deriving an efficient implementation of the iterate $\phi^t\in\Phi$ (of Line~\ref{line:phi-hedge-phi-update}) directly by exploiting the structure of $\Phi^{\trig}_{0}$.

\subsection{Efficient implementation of $\Phi^\trig$-Hedge algorithm}\label{sec:phi-hedge-implementation}

We first use a standard trick to convert the computation of $\phi^t$ (Line~\ref{line:phi-hedge-phi-update} \&~\ref{line:phi-hedge-p-update}, Algorithm~\ref{algorithm:phi-regret-minimization}) in $\Phi$-Hedge to evaluating the gradient of a suitable log-partition function. This is stated in the lemma below (for any generic $\Phi_0$), whose proof can be found in Appendix~\ref{appendix:proof-free-energy-trick}.

\begin{lemma}[Conversion to log-partition function]
\label{lem:free-energy-trick}
Define the log-partition function $F^{\Phi_0}: \R^{d \times d} \to \R$
\begin{align}
F^{\Phi_0}(M) \defeq \log \sum_{\phi\in\Phi_0} \exp\{ - \<\phi, M\>\}.
\end{align}
Then Line~\ref{line:phi-hedge-phi-update} of $\Phi$-Hedge (Algorithm~\ref{algorithm:phi-regret-minimization}) has a closed-form update for all $t\ge 1$:
\begin{align}\label{eqn:phi-update-grad-F}
\phi^{t} = -\grad F^{\Phi_0}\paren{ \eta\sum_{s=1}^{t-1} M^s } = - \frac{ \sum_{\phi\in\Phi_{0}} \exp\big\{-\eta \<\phi, \sum_{s=1}^{t-1} M^s\>\big\} \phi }{\sum_{\phi\in\Phi_{0}} \exp\big\{-\eta\<\phi, \sum_{s=1}^{t-1} M^s\>\big\}},~~~~ M^t \defeq \ell^t (\mu^t)^\top.
\end{align}
\end{lemma}
Eq.~\eqref{eqn:phi-update-grad-F} suggests a strategy for evaluating $\phi^t=-\grad F^{\Phi_0}(\eta\sum_{s=1}^{t-1}M^s)$---So long as the vertex set $\Phi_0$ has some structure that allows efficient evaluation of the sum of exponentials on the numerators and denominators (i.e. faster than naive sum), $\phi^t$ may be computed directly in sublinear in $|\Phi_0|$ time, and there is no need to maintain the underlying distribution $p^t\in\Delta_{\Phi_0}$.

The following lemma enables such an efficient computation for the log-partition function $F^\trig\defeq F^{\Phi^\trig}$ (and its gradient) associated with the trigger modification set $\Phi=\Phi^\trig$. This lemma (proof deferred to Appendix~\ref{appendix:proof-evaluate-partition-function-gradient}) is a consequence of the specific structure of $\Phi_0$ (cf.~\eqref{eqn:Phi-trig}), whose elements are indexed by a sequence $x_ga_g$ and a deterministic subtree policy $v^{x_g}\in\mc{V}^{x_g}$.
\begin{lemma}[Recursive expression of $F^{\trig}$ and $\grad F^{\trig}$]
\label{lem:evaluate-partition-function-gradient}
For any loss matrix $M \in \R^{XA \times XA}$, the \emph{EFCE log-partition function} can be written as
\begin{align}
\label{eqn:F-trig-M}
\textstyle   F^{\trig}(M) =  \log \sum_{g, x_g, a_g} \exp\Big\{  - \<I - E_{\succeq x_ga_g}, M\> + F_{x_g a_g, x_g}(M)  \Big\},
\end{align}
where for any $x_h \succeq x_g$,
\begin{equation}\label{eqn:F-M-FTRL-form}
 \textstyle   F_{x_g a_g, x_h}(M) \defeq \log \sum_{a_h} \exp\Big\{  - M_{x_ha_h, x_g a_g} +  \sum_{x_{h+1} \in \cC(x_h,a_h)} F_{x_g a_g, x_{h+1}}(M) \Big\}. 
\end{equation}
Furthermore, define $\lambda = (\lambda_{x_g a_g})_{x_g a_g \in \cX \times \cA} \in \Delta_{XA}$ and $m = ( m_{x_g a_g} )_{x_g a_g \in \cX \times \cA}$ with $m_{x_ga_g}\in\Pi^{x_g}$ (and also identified as a vector in $\R^{XA}$) as
\begin{align}
& \lambda_{x_ga_g} \propto_{x_ga_g} \exp\Big\{   - \< I - E_{\succeq x_g a_g}, M\> + F_{x_ga_g,x_g}(M) \Big\},\label{eqn:lambda-M-FTRL-form} \\
&\textstyle  m_{x_ga_g, h}(a_h \vert x_h) \propto_{a_h}\exp\Big\{  - M_{x_ha_h,x_ga_g} + \sum_{x_{h+1} \in \cC(x_h, a_h)} F_{x_ga_g, x_{h+1}}(M) \Big\},\label{eqn:m-M-FTRL-form}
\end{align}
then we have
\begin{align}
\label{equation:phi-lambda-m}
 \textstyle   -\grad F^{\trig}(M)=\phi(\lambda, m) \defeq \sum_{g, x_g, a_g} \lambda_{x_ga_g} (I - E_{\succeq x_g a_g} + m_{x_ga_g} e_{x_g a_g}^\top).
\end{align}
\end{lemma}
Above, $\lambda = (\lambda_{x_g a_g})_{x_g a_g \in \cX \times \cA} \in \Delta_{XA}$ is a probability distribution over $\cX\times\cA$, and $m = ( m_{x_g a_g} )_{x_g a_g \in \cX \times \cA} \in \mc{M} \equiv \prod_{g,x_g a_g} \Pi^{x_g a_g}$ is a collection of subtree policies $m_{x_g a_g}$, where each $m_{x_g a_g} \in \Pi^{x_g}$ is a subtree policy that specifies an action distribution $m_{x_g a_g, h}(a_h | x_h)$ for every $x_h \succeq x_g$, and can be identified with a vector in $\R^{XA}$ (c.f. Section \ref{sec:EFG-prelim}). 

The recursive structure in Lemma~\ref{lem:evaluate-partition-function-gradient} offers a roadmap for evaluating $(\lambda,m)$ and thus $\grad F^{\trig}(M)$ in $O(X^2A^2)$ time (formal statement in Appendix~\ref{appendix:runtime}). Applying Lemma~\ref{lem:evaluate-partition-function-gradient} with $M = \eta \sum_{s = 1}^{t-1} M^s$ gives an efficient implementation of~(\ref{eqn:phi-update-grad-F}), i.e. the $\Phi$-Hedge algorithm with $\Phi=\Phi^{\trig}$. For clarity, we summarize this in Algorithm~\ref{algorithm:efce-ftrl}. We remark that the parameters $(\lambda^{t}, m^{t})$ therein can also be expressed in terms of $(\lambda^{t-1}, m^{t-1})$ and $M^{t-1}$, which we present in Algorithm~\ref{algorithm:efce-omd} (the equivalent ``OMD'' form) in Appendix~\ref{sec:efce-omd}. We also note that the fixed point equation $\phi^t \mu = \mu$ in Line~\ref{line:fixed-point} can be solved in $O(X^2A^2)$ time~\citep[Corollary 4.15]{farina2021simple}.

\begin{algorithm}[t]
\caption{EFCE-OMD (FTRL form; equivalent OMD form in Algorithm~\ref{algorithm:efce-omd})}
\label{algorithm:efce-ftrl}
\begin{algorithmic}[1]
\REQUIRE Learning rate $\eta>0$. 
\FOR{$t = 1, 2, \ldots, T$}
\STATE For each $x_g a_g \in \cX \times \cA$, from the reverse order of $x_h$, compute $m^{t}_{x_ga_g, h}(a_h \vert x_h)$ and $F^{t}_{x_ga_g, x_h}$
\begin{align}
& \textstyle m^{t}_{x_ga_g, h}(a_h \vert x_h) \propto_{a_h} \exp\Big\{ - \eta\sum_{s=1}^{t-1} M^s_{x_ha_h, x_ga_g} + \sum_{x_{h+1} \in \cC(x_h, a_h)} F^{t}_{x_ga_g, x_{h+1}} \Big\}, \label{equation:mt} \\
&~\textstyle F^{t}_{x_ga_g, x_h} = \log  \sum_{a_h}  \exp\Big\{ -\eta  \sum_{s=1}^{t-1}M^s_{x_ha_h, x_ga_g} + \sum_{x_{h+1} \in \cC(x_h, a_h)} F^{t}_{x_ga_g, x_{h+1}} \Big\} \label{equation:ft},
\end{align}
\STATE Compute $\lambda_{x_g a_g}^{t}$ as
\begin{align}
&\textstyle \lambda_{x_ga_g}^{t} \propto_{x_ga_g}  \exp\Big\{ - \eta\< I - E_{\succeq x_g a_g}, \sum_{s=1}^{t-1}M^s\> + F^{ t}_{x_ga_g,x_g} \Big\}.  \label{equation:lambdat}
\end{align}
\STATE Compute $\phi^t = \phi(\lambda^t, m^t)$ where $\phi$ is in Eq. (\ref{equation:phi-lambda-m}).
\STATE Compute the policy $\mu^t$, which is a solution of the fixed point equation $\phi^t \mu^t =\mu^t$. \alglinelabel{line:fixed-point}
\STATE Receive loss $\ell^t = \{\ell^t_h(x_h, a_h)\}_{(x_h, a_h) \in \cX \times \cA} \in \R^{XA}_{\ge 0}$.
\STATE Compute matrix loss $M^t = \ell^t (\mu^t)^\top \in \R^{XA \times XA}_{\ge 0}$. 
    \ENDFOR
 \end{algorithmic}
\end{algorithm}

\subsection{Equivalence to FTRL and OMD}
\label{sec:phi-hedge-dilated-entropy}

We now show that Algorithm~\ref{algorithm:efce-ftrl} is equivalent to FTRL and OMD with suitable dilated entropies and divergences (hence the name EFCE-OMD). 
We define the trigger dilated entropy function and trigger dilated KL divergence function over $(\lambda,m)\in\Delta_{XA}\times \mc{M}$ as
\[
\begin{aligned}
\textstyle H^{\trig}(\lambda, m) \defeq &~ \textstyle H(\lambda) + \sum_{g, x_g, a_g} \lambda_{x_g a_g} H_{x_g}(m_{x_g a_g}), \\
\textstyle D^{\trig}( \lambda, m\| \lambda', m') \defeq &~ \textstyle D_{\kl}(\lambda \| \lambda') + \sum_{g, x_g, a_g} \lambda_{x_g a_g} D_{x_g}(m_{x_g a_g} \| m_{x_g a_g}'),
\end{aligned}
\]
where $H(\cdot)$ and $D_{\kl}(\cdot \| \cdot)$ are the (negative) Shannon entropy and KL divergence; and for any $x_g$, $H_{x_g}(\cdot)$ is the dilated entropy, and $D_{x_g}(\cdot\| \cdot)$ is the dilated KL divergence \cite{hoda2010smoothing}, both for the subtree rooted at $x_g$ (detailed definitions in Appendix~\ref{appendix:def-dilated}).

\begin{lemma}[Equivalent formulations of $\Phi^{\trig}$-hedge]
\label{lem:equivalence-algorithm}
For any sequence of loss functions $\{M^t\}_{t\ge 1}$, the iterates $(\lambda^t, m^t)$ in Algorithm~\ref{algorithm:efce-ftrl} (i.e. \eqref{equation:mt}-\eqref{equation:lambdat}) are equivalent to (i.e. satisfy) the following FTRL update on $H^{\trig}$ and OMD update on $D^{\trig}$:
\begin{align}
& \textstyle (\lambda^{t}, m^{t}) =~ \argmin_{\lambda, m} \Big[\eta \< \phi(\lambda, m), \sum_{s=1}^{t-1} M^s\> + H^{\trig}(\lambda, m) \Big], \label{equation:phit-H-form} \\
& \textstyle (\lambda^{t}, m^{t}) =~ \argmin_{\lambda, m}\Big[ \eta \< \phi(\lambda, m), M^{t-1}\> + D^{\trig}(\lambda, m \| \lambda^{t-1}, m^{t-1}) \Big].  \label{equation:phit-D-form}
\end{align}
\end{lemma}
The proof of Lemma~\ref{lem:equivalence-algorithm} follows directly by the concrete forms of $(\lambda^t,m^t)$ in \eqref{equation:mt}-\eqref{equation:lambdat}, and can be found in Appendix~\ref{appendix:proof-equivalence-algorithm}.

\subsection{Regret bound under full feedback and bandit feedback}

\label{sec:phi-hedge-regret}

We now present the regret bounds of Algorithm~\ref{algorithm:efce-ftrl}. We emphasize that these regret bounds are simple consequence of the generic bound for $\Phi$-Hedge in~\eqref{eqn:phiregret_bound_simple}, and their proofs do not depend on the actual implementation of Algorithm~\ref{algorithm:efce-ftrl} developed in the preceding two subsections. We first consider the full feedback setting, where the full expected loss vector $\ell^t\in\R^{XA}_{\ge 0}$ is received after each episode. %
\begin{theorem}[Regret bound of EFCE-OMD under full feedback]
\label{thm:iteration-complexity-full}
Running Algorithm \ref{algorithm:efce-ftrl} with $\eta = \cO(\sqrt{\|\Pi\|_1 \iota/(H^2T)})$ achieves the following trigger regret bound
\begin{align*}
    \efceregret(T) \le \cO\big( \sqrt{H^2 \| \Pi\|_1 \iota T} \big),
\end{align*}
where $\iota\defeq \log(XA)$ is a log factor.
\end{theorem}

The proof of Theorem~\ref{thm:iteration-complexity-full} is simply by applying~\eqref{eqn:phiregret_bound_simple} and observing that $\log (\Phi^{\trig}_0) \le \| \Pi \|_1 \log A + \log (XA)$ (see Appendix \ref{app:iteration-complexity-full-proof}). This theorem shows that the $\Phi^\trig$-Hedge algorithm gives $\tO(\sqrt{XT})$ trigger regret bound, which matches the information-theoretic lower bound $\Omega(\sqrt{XT})$~\cite[Theorem 2]{zhou2018lazy} up to a $\tO({\rm poly}(H))$ factor, and is slightly better than the $\tO(\sqrt{XAT})$ upper bound of~\citep[Corollary F.3]{song2022sample} though their definition of EFCE-regret is slightly stricter (thus higher) than ours.

In the bandit feedback setting, the learner only observes her own rewards and infosets. In this case we replace $\ell^t$ in Algorithm~\ref{algorithm:efce-ftrl} with the following loss estimator (with IX bonus $\gamma$) proposed in~\citep{kozuno2021model}:
\begin{align}
\label{eq:bandit-loss-estimator}
 \textstyle     \widetilde{\ell}_h^{t}(x_h, a_h) \defeq \indic{(x_h^t, a_h^t)=(x_h, a_h)} (1 - r_h^t) / (\mu_{1:h}^t (x_h, a_h)+\gamma ).
\end{align}
We show that EFCE-OMD achieves the following guarantee in the bandit feedback setting (proof in Appendix \ref{app:proof-phi-hedge-efce-bandit}). The proof follows by plugging the loss estimator $\wt{\ell}^t$ into~\eqref{eqn:phiregret_bound_simple} and additionally bounding concentrations (which we remark is a better strategy than using a naive bandit-based loss estimator in the corresponding NFG space).
\begin{theorem}[Regret bound of EFCE-OMD under bandit feedback]\label{thm:phi-hedge-efce-bandit}
Run Algorithm \ref{algorithm:efce-ftrl} with loss estimator $\{\wt\ell^t\}_{t=1}^T$~\eqref{eq:bandit-loss-estimator}, $\eta =  \sqrt{\| \Pi \|_1 \log A/(HXAT)}$, and $\gamma=\sqrt{\lone{\Pi}\iota/(XAT) }$. Then we have the following trigger regret bound with probability at least $1-\delta$:
\begin{align*}
\efceregret(T) \le \cO \big( \sqrt{HXA \| \Pi \|_1 \iota \cdot T} \big),
\end{align*}
where $\iota = \log (3XA/\delta)$ is a log term.
\end{theorem}
To our best knowledge, Theorem~\ref{thm:phi-hedge-efce-bandit} gives the first trigger regret bound against adversarial opponents and bandit feedback.
This $\tO(\sqrt{XA\lone{\Pi}T})$ rate is $\sqrt{XA}$ worse than Theorem~\ref{thm:iteration-complexity-full} (ignoring $H$ and log factors), and is at most $\tO(\sqrt{X^2AT})$ using $\lone{\Pi}\le X$.

\section{Balanced EFCE-OMD for bandit feedback}
\label{sec:sharp-rate}

We now build upon the EFCE-OMD algorithm (Algorithm~\ref{algorithm:efce-ftrl}) to develop a new algorithm, \emph{Balanced EFCE-OMD} (Algorithm~\ref{algorithm:balanced-efce-ftrl}), and show that it achieves near-optimal extensive-form trigger regret guarantee under bandit feedback. Here we discuss the two key modifications in the algorithm design.

\begin{algorithm}[t]
   \caption{Balanced EFCE-OMD (FTRL form; equivalent OMD form in Algorithm~\ref{algorithm:balanced-efce-omd})}
   \label{algorithm:balanced-efce-ftrl}
   \begin{algorithmic}[1]
    \REQUIRE Learning rate $\eta$, balanced exploration policy $\{\mu^{\star, h}\}_{h \in [H]}$. 
    \FOR{$t = 1, 2, \ldots, T$}
    \STATE For each $x_g a_g \in \cX \times \cA$, from the reverse order of $x_h$, compute $m^{t}_{x_ga_g, h}(a_h \vert x_h)$ and $F^{\star,t}_{x_ga_g, x_h}$
\begin{align*}
m^{t}_{x_ga_g, h}(a_h \vert x_h) \propto_{a_h}&~  \exp\Big\{ \mu^{\star,h}_{g:h}(x_h, a_h) \Big(- \eta \sum_{s=1}^{t-1}\wt{M}^s_{x_ha_h, x_ga_g} + \sum_{x_{h+1} \in \cC(x_h, a_h)} F^{\star,t}_{x_ga_g, x_{h+1}} \Big) \Big\}, \\
F_{x_g a_g, x_h}^{\star,t} \defeq&~  \frac{1}{\mu^{\star,h}_{g:h}(x_h, a_h)}\log \sum_{a_h \in \cA} \exp\Big\{ \mu^{\star,h}_{g:h}(x_h, a_h) \\
&~\times \big[ - \eta\sum_{s=1}^t\wt{M}^s_{x_ha_h, x_g a_g} +  \sum_{x_{h+1} \in \cC(x_ha_h)} F_{x_g a_g, x_{h+1}}^{\star,t}\big] \Big\}.
\end{align*}
\STATE Compute $\lambda_{x_g a_g}^{t+1}$ as
\begin{align}
\lambda_{x_ga_g}^{t} \propto_{x_ga_g}  \exp\Big\{ \frac{1}{XA}  \Big( - \eta\langle I - E_{\succeq x_g a_g}, \sum_{s=1}^{t-1}\wt{M}^s \rangle + F^{\star, t}_{x_ga_g,x_g} \Big) \Big\} .
\end{align}
    \STATE Compute $\phi^t = \phi(\lambda^t, m^t)$, where $\phi$ is as defined in Eq. (\ref{equation:phi-lambda-m}).
    \STATE Find a $\mu^t$ to be a solution of the fixed point equation $\mu^t = \phi^t \mu^t$. 
    \STATE Play policy $\mu^t$, observe trajectory $(x_h^t, a_h^t, r_h^t)_{h \in [H]}$. 
    \STATE Form vector loss estimator $\wt \ell^{t, x_g a_g} = \{ \wt{\ell}^{t, x_g a_g}_h(x_h, a_h)\}_{x_h a_h}$ for each $(g, x_g a_g)$ as in Eq. (\ref{equation:adaptive-bandit-estimator}). 
\STATE Compute matrix loss estimator $\wt{M}^t = \sum_{g, x_g, a_g} \mu^t_{x_ga_g} \wt{\ell}^{t, x_ga_g} e_{x_g a_g}^\top$. 
    \ENDFOR
 \end{algorithmic}
\end{algorithm}

\paragraph{Key modification I: ``Rebalancing'' the log-partition function}

Building on the balancing technique of~\citep{bai2022near}, we start from Eq. (\ref{eqn:F-trig-M}) and (\ref{eqn:F-M-FTRL-form}) of the log partition function, and rescale the inner functions $F_{x_ga_g, x_h}$ using \emph{balanced exploration policies}~$\{\mu_{g:h}^{\star,h}(x_h, a_h)\}_{g, x_h, a_h}$ (see Definition~\ref{def:balanced-exploration-policy} for the formal definition), and rescale the outer function $F^{\trig}$ by $XA$. Concretely, for any matrix $M \in \R^{XA \times XA}$, we define the \emph{balanced EFCE log-partition function} as
\begin{align}
\label{equation:balanced-f-efce}
\textstyle   F^{\trig}_{\bal}(M) \defeq XA \log \sum_{g, x_g, a_g} \exp\Big\{ \frac{1}{XA} \big[- \<I - E_{\succeq x_ga_g}, M\> + F_{x_g a_g, x_g}^\star(M) \big] \Big\},
\end{align}
where for any $x_h \succeq x_g$ (using $\mu^{\star, h}_{g:h}\defeq \mu^{\star, h}_{g:h}(x_h, a_h)$ as shorthand, which depends on $x_h$ but not $a_h$),
\begin{align}
\label{equation:balanced-f-efce-xgag}
    F_{x_g a_g, x_h}^\star(M) \defeq&~ \frac{1}{\mu^{\star,h}_{g:h}}\log \sum_{a_h} \exp\Big\{ \mu^{\star,h}_{g:h} \big[ - M_{x_ha_h, x_g a_g} +  \sum_{x_{h+1} \in \cC(x_ha_h)} F_{x_g a_g, x_{h+1}}^\star(M)\big] \Big\}.
\end{align}

\paragraph{Key modification II: New loss estimator under bandit feedback} 
We use an \emph{adaptive} family of bandit-based loss estimators $\{\wt{\ell}^{t, x_ga_g} \}_{x_ga_g}\subset \R^{XA}_{\ge 0}$, one for each $(x_g,a_g)\in\cX\times\cA$, defined as
\begin{align}
\label{equation:adaptive-bandit-estimator}
  \wt{\ell}^{t, x_ga_g}_{h}(x_h, a_h) \defeq \frac{\indic{(x_h^t, a_h^t) = (x_h, a_h)} (1 - r_h^t)}{\mu^t_{1:h}(x_h, a_h) + \gamma (\mu^{\star, h}_{1:h}(x_h, a_h) + \mu^t_{x_ga_g} m^t_{x_ga_g, g:h}(x_h, a_h)\indic{x_h\succeq x_g})},
\end{align}
where $\mu^t_{x_ga_g}\defeq \mu^t_{1:g}(x_g,a_g)$ for shorthand. The main difference of~\eqref{equation:adaptive-bandit-estimator} over~\eqref{eq:bandit-loss-estimator} is in the adaptive IX bonus term on the denominator that scales with $\gamma$ but is different for each $x_ga_g$. We then place each $\mu^t_{x_ga_g} \wt{\ell}^{t,x_ga_g}$ into the $x_ga_g$-th column of a matrix loss estimator $\wt{M}^t$, or in matrix form,
\begin{align*}
\textstyle    \wt{M}^t \defeq \sum_{g, x_g, a_g} \mu^t_{x_ga_g} \wt{\ell}^{t, x_ga_g} e_{x_g a_g}^\top.
\end{align*}
With~\eqref{equation:balanced-f-efce}-\eqref{equation:adaptive-bandit-estimator} at hand, our algorithm Balanced EFCE-OMD is defined as the negative gradient of $F^{\trig}_{\bal}$ evaluated at the cumulative loss estimators:
\begin{align}
\textstyle     \phi^{t} = - \nabla F^{\trig}_{\bal}\Big( \eta \sum_{s = 1}^{t-1} \wt{M}^s \Big), ~~~ \forall t \ge 1, \label{equation:phit-F-form} 
\end{align}
and $\mu^{t}\in\Pi$ solves the fixed point equation $\phi^{t}\mu^{t}=\mu^{t}$. Similar as EFCE-OMD,~\eqref{equation:phit-F-form} also admits efficient implementations in both FTRL and OMD form (cf. Algorithm~\ref{algorithm:balanced-efce-ftrl} \&~\ref{algorithm:balanced-efce-omd}). The corresponding $(\lambda^t, m^t)$ is also equivalent to running a FTRL/OMD algorithm with respect to a \emph{balanced} dilated entropy/KL-divergence over $\phi\in\Phi^{\trig}$ (cf. Lemma~\ref{lem:equivalent-FTRL-OMD-balanced} and Appendix~\ref{sec:balanced-equivalence} for details).

\paragraph{Main result} We now present the theoretical guarantee of Algorithm~\ref{algorithm:balanced-efce-ftrl} (proof in Appendix~\ref{app:proof-balanced-efce-omd-bandit}).

\begin{theorem}\label{thm:balanced-efce-omd-bandit}
Balanced EFCE-OMD (Algorithm \ref{algorithm:balanced-efce-ftrl}) with $\eta = \sqrt{XA \iota  / H^4 T}$ and $\gamma = 2 \sqrt{XA \iota / H^2 T}$ achieves the following extensive-form trigger regret bound with probability at least $1-\delta$:
\begin{align*}
 \textstyle   \efceregret(T) \le  \cO \big( \sqrt{H^4XA T \iota} \big),
\end{align*}
where $\iota = \log (10XA/\delta)$ is a log term.
\end{theorem}

The $\wt{\cO}(\sqrt{XAT})$ trigger regret asserted in Theorem~\ref{thm:balanced-efce-omd-bandit} improves over Theorem~\ref{thm:phi-hedge-efce-bandit} by a factor of $\sqrt{\lone{\Pi}}$, and matches the information-theoretic lower bound up to $\text{poly}(H)$ and log factors\footnote{As the trigger regret is lower bounded by the vanilla (external) regret,~\citep[Theorem 6]{bai2022near} implies an $\Omega(\sqrt{XAT})$ lower bound for the trigger regret as well under bandit feedback.}. By the online-to-batch conversion (Appendix~\ref{appendix:online-to-batch}), Theorem~\ref{thm:balanced-efce-omd-bandit} also implies an $\tO(H^4XA/\eps^2)$ sample complexity for learning EFCE under bandit feedback (assuming same game sizes for all $m$ players). This improves over the best known $\tO(mH^6XA^2/\eps^2)$ sample complexity in the recent work of~\citet{song2022sample}\footnote{We remark though that the $1$-EFR algorithm of~\citep{song2022sample} actually finds an ``1-EFCE'' which is slightly stronger than our EFCE defined via trigger modifications.}.

\paragraph{Overview of techniques} 
The proof of Theorem~\ref{thm:balanced-efce-omd-bandit} is significantly more challenging than that of Theorem~\ref{thm:phi-hedge-efce-bandit}, even though the algorithm itself is designed by appearingly simple modifications. This happens since Algorithm~\ref{algorithm:balanced-efce-ftrl}, unlike Algorithm~\ref{algorithm:efce-ftrl}, no longer necessarily corresponds to any normal-form algorithm. The technical crux of the proof is to bound the nonlinear part of $F^{\trig}_{\bal}$ (with respect to the losses), which we do by carefully controlling a series of second-order terms utilizing the balanced policies within $F^{\trig}_{\bal}$ and the new adaptive IX bonus within $\{\wt{\ell}^{t,x_ga_g}\}_{x_ga_g}$ (Lemma~\ref{lem:reformulation-of-stability}-\ref{lem:bound-on-II-balanced}).

\section{Equivalence of OMD and Vertex MWU for regret minimization}
\label{sec:equivalence}

As another illustration of our framework, we now choose $\Phi=\Phi^{\ext}=\conv\{\Phi_0^{\ext}\}$ to be the set of \emph{external} policy modifications, which modify any policy to some deterministic policy. In this case, the $\Phi^{\ext}$-Hedge algorithm minimizes the external regret in EFGs. In this section, we show that $\Phi^{\ext}$-Hedge, same as the vertex MWU algorithm considered in~\citet{farina2022kernelized}, is actually equivalent to the OMD with dilated entropy \cite{hoda2010smoothing}. Let $\set{\ell^t}_{t\ge 1}\subset \R^{XA}_{\ge 0}$ be an arbitrary sequence of loss vectors.

\paragraph{Vertex MWU}
We use $\mc{V}$ to denote all the deterministic sequence-form policies, which can also be viewed as the vertex set of the policy set $\Pi$. 
A simple reformulation (cf. Appendix \ref{app:equivalence-vertex-hedge}) shows that $\Phi^{\ext}$-Hedge (Algorithm \ref{algorithm:phi-regret-minimization}) gives the vertex MWU algorithm considered by~\citet{farina2022kernelized} 
\begin{equation}
    \label{equation:vertex-ftrl}
    \begin{aligned}
  \textstyle      \mu^{t}  = \sum_{v \in \mc{V}} p^{t}_v \cdot v~~~~~~~\text{and}~~~~~~~ p_v^{t} \propto_v \exp\set{- \eta \<v, \sum_{s=1}^{t-1}\ell^s\>}.
    \end{aligned}
\end{equation}

\paragraph{OMD with dilated entropy}
Another popular algorithm for external regret minimization is the OMD algorithm on the sequence-form policy space with the dilated entropy~\citep{hoda2010smoothing,kroer2015faster}:
\begin{align}
    & \mu^{t} = \argmin_{\mu \in \Pi} \big[ \eta \<\mu, \ell^{t-1}\> + D_\emptyset(\mu\|\mu^{t-1}) \big], \label{equation:dilated-ent-omd} \\
    &  D_\emptyset(\mu\|\nu) \defeq \sum_{h=1}^H \sum_{x_h, a_h} \mu_{1:h}(x_h, a_h) \log \frac{ \mu_h(a_h|x_h)}{\nu_h(a_h|x_h)}.
\end{align}

\begin{theorem}[Equivalence of OMD and Vertex MWU]
\label{theorem:equivalence-ftrl}
For any sequence of loss vectors $\{\ell^t\}_{t\ge 1}$,
OMD with dilated entropy is equivalent to Vertex MWU, that is, \eqref{equation:dilated-ent-omd} and \eqref{equation:vertex-ftrl} give the same $\set{\mu^t}_{t\ge 1}$.
\end{theorem}
The proof of Theorem~\ref{theorem:equivalence-ftrl} can be found in Appendix~\ref{appendix:proof-equivalence}. Our proof also reveals that the efficient implementation of Vertex MWU developed by~\citet{farina2022kernelized} using the ``kernel trick'' is actually equivalent to the standard linear-time efficient implementation of OMD with dilated entropy.

\section{Conclusion}

In this paper, we present an efficient implementation of the $\Phi$-Hedge algorithm for minimizing the extensive form trigger regret. The algorithm is equivalent to OMD with dilated regularizers, and achieves efficient regret bounds under both full feedback and bandit feedback. We also design an improved algorithm Balanced EFCE-OMD, which achieves a sharp trigger regret bound under bandit feedback. We believe our work leads to many open questions, such as efficient implementations of $\Phi$-Hedge with more general $\Phi$ sets (e.g. the behavioral modifications considered in~\citep{morrill2021efficient,song2022sample}), or accelerated ${\rm polylog}(T)$ $\Phi$-regret bounds under full feedback by optimistic algorithms.

\section*{Acknowledgment}
S.M. is supported by NSF grant DMS-2210827. C.J. is supported by Office of Naval Research N00014-22-1-2253.

\bibliographystyle{abbrvnat}
\bibliography{bib}

\makeatletter
\def\renewtheorem#1{%
  \expandafter\let\csname#1\endcsname\relax
  \expandafter\let\csname c@#1\endcsname\relax
  \gdef\renewtheorem@envname{#1}
  \renewtheorem@secpar
}
\def\renewtheorem@secpar{\@ifnextchar[{\renewtheorem@numberedlike}{\renewtheorem@nonumberedlike}}
\def\renewtheorem@numberedlike[#1]#2{\newtheorem{\renewtheorem@envname}[#1]{#2}}
\def\renewtheorem@nonumberedlike#1{  
\def\renewtheorem@caption{#1}
\edef\renewtheorem@nowithin{\noexpand\newtheorem{\renewtheorem@envname}{\renewtheorem@caption}}
\renewtheorem@thirdpar
}
\def\renewtheorem@thirdpar{\@ifnextchar[{\renewtheorem@within}{\renewtheorem@nowithin}}
\def\renewtheorem@within[#1]{\renewtheorem@nowithin[#1]}
\makeatother

\renewtheorem{theorem}{Theorem}[section]
\renewtheorem{lemma}{Lemma}[section]
\renewtheorem{remark}{Remark}
\renewtheorem{corollary}{Corollary}[section]
\renewtheorem{observation}{Observation}[section]
\renewtheorem{proposition}{Proposition}[section]
\renewtheorem{definition}{Definition}[section]
\renewtheorem{claim}{Claim}[section]
\renewtheorem{fact}{Fact}[section]
\renewtheorem{assumption}{Assumption}[section]
\renewcommand{\theassumption}{\Alph{assumption}}
\renewtheorem{conjecture}{Conjecture}[section]

\tableofcontents

\appendix

\renewcommand{\efce}{\trig}
\section{Technical tools}
\label{app:phi-reg-min}

The following lemma is standard and gives a $\Phi$-regret bound of the $\Phi$-Hedge algorithm. 

\begin{lemma}[Regret bound for $\Phi$-Hedge]
\label{lemma:regret-bound-phi-hedge}
For strategy modification vertex set $\Phi_0$, step size $\eta$, and total steps $T$, running Algorithm \ref{algorithm:phi-regret-minimization} gives
\begin{align*}
    \phiregret(T) \le \frac{\log \abs{\Phi_0}}{\eta} +  \frac{\eta}{2} \sum_{t=1}^T \sum_{\phi\in\Phi_0} p_\phi^{t} \paren{\langle \phi \mu^t, \ell^t \rangle}^2.
\end{align*}
\end{lemma}
\begin{proof}%
We have
\begin{align*}
      \phiregret(T) &= \sup_{\phi \in \Phi} \sum_{t=1}^T \langle \mu^t - \phi \mu^t, \ell^t \rangle 
    \stackrel{(i)}{=}  \sup_{\phi \in \Phi} \sum_{t=1}^T \langle \phi^t \mu^t - \phi \mu^t, \ell^t \rangle \\
    & \stackrel{(ii)}{=} \sup_{p \in \Delta_{\Phi_0} } \sum_{t=1}^T \sum_{\phi\in\Phi_0} \paren{p_\phi^t \langle \phi\mu^t, \ell^t \rangle - p_\phi \langle \phi\mu^t, \ell^t \rangle}.
\end{align*}
Above, (i) uses the fixed point equation $\phi^t\mu^t=\mu^t$ (Line~\ref{line:fixed-point-phi-hedge}), and (ii) uses the fact that $\Phi=\conv\{\Phi_0\}$. Note that the above expression is exactly the regret of $\{p^t\}_{t=1}^T$, where the loss vector in the $t$-th round is $\{\langle \phi\mu^t, \ell^t \rangle \}_{\phi \in \Phi_0}$. Further, the update rule of $p^t$ (Line~\ref{line:phi-hedge-p-update}) coincides with Hedge algorithm. So by the standard regret bound for Hedge, see, e.g. (\citet{lattimore2020bandit}, Proposition 28.7), we have 
\begin{align*}
     \quad \phiregret(T)& = \sup_{p \in \Delta_{\Phi_0} } \sum_{t=1}^T \sum_{\phi\in\Phi_0} \paren{p_\phi^t \langle \phi\mu^t, \ell^t \rangle - p_\phi \langle \phi\mu^t, \ell^t \rangle} \\
    & \le \frac{\log \abs{\Phi_0}}{\eta} +  \frac{\eta}{2} \sum_{t=1}^T \sum_{\phi\in\Phi_0} p_\phi \paren{\langle \phi \mu^t, \ell^t \rangle}^2.
\end{align*}
This proves the lemma. 
\end{proof}

The following Freedman's inequality can be found in~\citep[Lemma 9]{agarwal2014taming}.
\begin{lemma}[Freedman's inequality]
  \label{lemma:freedman}
  Suppose random variables $\set{X_t}_{t=1}^T$ is a martingale difference sequence, i.e. $X_t\in\cF_t$ where $\set{\cF_t}_{t\ge 1}$ is a filtration, and $\E[X_t|\cF_{t-1}]=0$. Suppose $X_t\le R$ almost surely for some (non-random) $R>0$. Then for any $\lambda\in(0, 1/R]$, we have with probability at least $1-\delta$ that
  \begin{align*}
    \sum_{t=1}^T X_t \le \lambda \cdot \sum_{t=1}^T \E\brac{X_t^2 | \cF_{t-1}} + \frac{\log(1/\delta)}{\lambda}.
  \end{align*}
\end{lemma}
\section{Properties of the game}
\label{appendix:efg}

In this section we some properties of EFGs (using the TFAMDP definition in Section~\ref{sec:EFG-prelim}).

\subsection{Equivalence to classical definitions of EFGs}
\label{appendix:efg-def}
We first formally define Extensive-Form Games (EFGs). We then show that solving EFGs with adversarial opponents can be reduced to solving Tree-Formed AMDP. 

\subsubsection{Classical definition of EFGs}

We consider the problem of multi-player general-sum version of EFGs with adversarial opponents. We remark that in order to study this problem, it suffices to study the two-player zero-sum version. To apply the results obtained in the latter setting to the former setting, we simply view the second player as the collection of all other players (who play jointly against the first player), and view the zero-sum reward as the reward of the first player.

\paragraph{Partially observable Markov games}
Following the convention of \citep{kozuno2021model,bai2022near}, we consider EFGs under the model of finite-horizon, tabular, two-player zero-sum Markov Games with partial observability, which can be described as a tuple ${\rm POMG}(H, \cS, \cX, \cY, \cA, \cB, \P, r)$, where
\begin{itemize}[wide, itemsep=0pt, topsep=0pt]
\item $H$ is the horizon length;
\item $\cS=\bigcup_{h\in[H]} \cS_h$ is the (underlying) state space with $|\cS_h|=S_h$ and $\sum_{h=1}^H S_h=S$;
\item $\cX=\bigcup_{h\in[H]} \cX_h$ is the space of information sets (henceforth \emph{infosets}) for the \emph{max-player} with $|\cX_h|=X_h$ and $X\defeq \sum_{h=1}^H X_h$. At any state $s_h\in\cS_h$, the max-player only observes the infoset $x_h=x(s_h)\in\cX_h$, where $x:\cS\to\cX$ is the emission function for the max-player;
\item  $\cY=\bigcup_{h\in[H]} \cY_h$ is the space of infosets for the \emph{min-player} with $\abs{\cY_h}=Y_h$ and $Y\defeq \sum_{h=1}^H Y_h$. An infoset $y_h$ and the emission function $y:\cS\to\cY$ are defined similarly.
\item $\cA$, $\cB$ are the action spaces for the max-player and min-player respectively, with $\abs{\cA}=A$ and $\abs{\cB}=B$.
\item $\P=\{p_1(\cdot)\in\Delta(\cS_1)\} \cup \{p_h(\cdot|s_h, a_h, b_h)\in \Delta(\cS_{h+1})\}_{(s_h, a_h, b_h)\in \cS_h\times \cA\times \cB,~h\in[H-1]}$ are the transition probabilities, where $p_1(s_1)$ is the probability of the initial state being $s_1$, and $p_h(s_{h+1}|s_h, a_h, b_h)$ is the probability of transiting to $s_{h+1}$ given state-action $(s_h, a_h, b_h)$ at step $h$;
\item $r=\set{r(s_H, a_H, b_H)\in[0,1]}_{(s_H, a_H, b_H)\in\cS_H\times\cA\times \cB}$ are the (random) rewards received in the very last step with mean $\wb{r}(s_H, a_H, b_H)$. \footnote{We note in this formulation, reward depends on latent state, which can reveal information about latent state beyond information sets. Here, we consider the formulation where reward is only revealed in the very last step, to avoid such information leakage in the earlier steps. The more general formulation where rewards are received at every step can be translated to this formulation (by postponing all rewards to the last step), which will only incur mild (one or two) additional $H$ factors in the rates.} 
\end{itemize}

\paragraph{Policies, value functions}
As we consider partial observability, each player's policy can only depend on the infoset rather than the underlying state. A policy for the max-player is denoted by $\mu=\set{\mu_h(\cdot|x_h)\in\Delta(\cA)}_{h\in[H], x_h\in\cX_h}$, where $\mu_h(a_h|x_h)$ is the probability of taking action $a_h\in\cA$ at infoset $x_h\in\cX_h$. Similarly, a policy for the min-player is denoted by $\nu=\set{\nu_h(\cdot|y_h)\in\Delta(\cB)}_{h\in[H], y_h\in\cY_h}$. A trajectory for the max player takes the form $(x_1, a_1, x_2, \dots, x_H, a_H, r)$, where $a_h\sim \mu_h(\cdot|x_h)$, and the rewards and infoset transitions depend on the (unseen) opponent's actions and underlying state transition.

The overall game value for any (product) policy $(\mu, \nu)$ is denoted by $V^{\mu, \nu}\defeq \E_{\mu, \nu}\brac{ r(s_H, a_H, b_H) }$. The max-player aims to maximize the value, whereas the min-player aims to minimize the value.

\paragraph{Interaction protocol}
At the beginning of each episode $t$, the adversarial opponent chooses a policy $\nu^t$, which is not revealed to the player. 
Then, the initial state $s_1$ is sampled from $p_1(\cdot)$, and its corresponding information sets $x_1, y_1$ are revealed to each player respectively. At each step $h$, the system is in a underlying state $s_h$. The player chooses an action $a_h$ according to the observed infoset $x_h$, while the opponent simultaneously chooses an action $b_h$ according to policy $\nu^t$ and the observed infoset $y_h$. Afterwards, the environment transitions to a new state $s_{h+1}$ according to $p_h(\cdot|s_h, a_h, b_h)$. This episode ends when $(a_H, b_H)$ is played, and reward $r(s_H, a_H, b_H)$ is observed.

\paragraph{Tree structure and perfect recall}
We assume that our POMG has a \emph{tree structure}: For any $h$ and $s_h\in\cS_h$, there exists a unique history $(s_1, a_1, b_1, \dots, s_{h-1}, a_{h-1}, b_{h-1})$ of past states and actions that leads to $s_h$. We also assume that both players have \emph{perfect recall}: For any $h$ and any infoset $x_h\in\cX_h$ for the max-player, there exists a unique history $(x_1, a_1, \dots, x_{h-1}, a_{h-1})$ of past infosets and max-player actions that leads to $x_h$ (and similarly for the min-player). 
We similarly define $\cC(x_h, a_h)\subset\cX_{h+1}$ to be the set of all immediate children of $(x_h, a_h)$ at step $h+1$.

With the tree structure and perfect recall, under any product policy $(\mu, \nu)$, the probability of reaching state-action $(s_h, a_h, b_h)$ at step $h$ takes the form
\begin{align}
  \label{equation:reaching-probability-decomposition}
  \P^{\mu, \nu}(s_h, a_h, b_h) = p_{1:h}(s_h) \mu_{1:h}(x_h, a_h) \nu_{1:h}(y_h, b_h),
\end{align}
where we have defined the sequence-form transition probability as
\begin{align*}
  p_{1:h}(s_h)\defeq p_1(s_1)\prod_{h'\le h-1} p_{h'}(s_{h'+1}|s_{h'}, a_{h'}, b_{h'}),
\end{align*}
and the \emph{sequence-form policies} as
\begin{align*}
  \mu_{1:h}(x_h, a_h) \defeq \prod_{h'=1}^h \mu_{h'}(a_{h'}|x_{h'}), \quad  \nu_{1:h}(y_h, b_h) \defeq \prod_{h'=1}^h \nu_{h'}(b_{h'}|y_{h'}).
\end{align*}
Above, $\set{s_{h'}, a_{h'}, b_{h'}}_{h'\le h-1}$ are the histories uniquely determined from $s_h$, and $x_{h'}=x(s_{h'})$, $y_{h'}=y(s_{h'})$. 

We let $\Pi_{\max}$ denote the set of all possible policies for the max player ($\Pi_{\min}$ for the min player). In the sequence form representation, $\Pi_{\max}$ is a convex compact subset of $\R^{XA}$ specified by the constraints $\mu_{1:h}(x_h, a_h)\ge 0$ and $\sum_{a_h\in\cA}\mu_{1:h}(x_h, a_h)=\mu_{1:h-1}(x_{h-1}, a_{h-1})$ for all $(h, x_h, a_h)$, where $(x_{h-1}, a_{h-1})$ is the unique pair of prior infoset and action that reaches $x_h$ (understanding $\mu_0(x_0, a_0)=\mu_0(\emptyset)=1$).

\subsubsection{Reduction from classical definition of EFGs to TFAMDP}

In this section, we show that solving EFGs with adversarial opponents can be reduced to solving TFAMDP. Formally, we prove the following proposition.

\begin{proposition} \label{prop:reduction}
For any EFG (i.e. POMG with tree structure and perfect recall assumptions) $(H, \cS, \cX, \cY, \cA, \cB, \P, r)$ with adversarial opponents' policies $\{\nu^t\}_{t\ge 1}$, there exists an adversarial MDP $(\tilde{H}, \tilde{\cX}, \tilde{\cA}, \tilde{\mathcal{T}})$ with adversarial transition $\{\tilde{p}^t=\{p^t_h\}_{h\in\{0\}\cup [H]}\}_{t\ge1}$ and reward $\{\tilde{R}^t=\{\tilde{R}^t_h\}_{h\in[H]}\}_{t \ge 1}$, so that for any policy sequences $\{\mu^t\}_{t\ge1}$, their joint distributions over the learner's trajectory $P(x_1, a_1, x_2, a_2, \ldots, x_H, a_H, r)$ are exactly the same for all episodes $t \ge1$.
\end{proposition}

We remark that the joint distribution $P(x_1, a_1, x_2, a_2, \ldots, x_H, a_H, r)$ gives a complete description about what the first player can obtain from the dynamic systems in both models. The joint distributions being the same for two models means that information-theoretically, the learner has no way to distinguish the two models, thus proving their equivalence.

\begin{proof}[Proof of Proposition \ref{prop:reduction}]
In this section, we will use the notation in its original form $\cX, \cA, p, r$ to denote the quantity in EFGs while use their tilded form $\tilde{\cX}, \tilde{\cA}, \tilde{p}, \tilde{R}$ to denote the corresponding quantity in tree-form AMDP. It is not hard to see that in order to prove Proposition \ref{prop:reduction} for all $t \ge1$, it suffices to prove for a fixed $t$ it is true.

\paragraph{Construction of AMDP} we construct the corresponding AMDP using EFGs in the following way: we let $\tilde{H} = H$, $\tilde{\cX} = \cX$, $\tilde{\cA} = \cA$. Since EFG satisfies perfect recall assumption, which defines the immediate children function $\cC$. We use the precisely same child function to define the tree structure $\tilde{\mathcal{T}}$ in AMDP. We define the adversarial transition according to the following equations:
\begin{align*}
    \tilde{p}_1^t(x_1) :=& \sum_{s_1 \in x_1} p_1(s_1), \\
    \tilde{p}_h^t(x_{h+1}|x_h, a_h) :=& \frac{\sum_{s_{h+1} \in x_{h+1}} \sum_{b_{h+1} \in \cB} p_{1:h+1}(s_{h+1}) \nu^t_{1:h+1}(y_{h+1}(s_{h+1}), b_{h+1})}{\sum_{s_{h} \in x_{h}} \sum_{b_{h} \in \cB} p_{1:h}(s_{h}) \nu^t_{1:h}(y_{h}(s_{h}), b_{h})},
\end{align*}
where $y_{h+1}(s_{h+1})$ and $y_h(s_{h})$ are the infoset of opponent at $(h+1)$-th and $h$-th steps given state $s_{h+1}$ and $s_h$ respectively.
We also define the adversarial reward distribution $\tilde{R}_H^t( \cdot |x_H, a_H)$ such that it gives the following distribution over reward $r \in [0, 1]$ for any fixed $(x_H, a_H)$
\begin{equation*}
      r = r(s_H, a_H, b_H)  \text{~~with probability~~} \frac{p_{1:H}(s_{H}) \nu^t_{1:H}(y_{H}(s_{H}), b_H)}{\sum_{s'_{H} \in x_{H}} \sum_{b'_{H} \in \cB} p_{1:H}(s'_{H}) \nu^t_{1:H}(y_{H}(s'_{H}), b'_H)}.
\end{equation*}
And we set the adversarial reward $\tilde{R}_h^t( \cdot |x_h, a_h)$ to be zero (almost surely) for all $h\le H-1$ and all $(x_h, a_h, t)$.

\paragraph{Proof of equivalence} Denote $\tilde{P}^{\mu, t}$ as the probability of AMDP at episode $t$ with policy $\mu$; denote $P^{\mu, \nu^t}$ as the probability of EFGs under policy $\mu$ and $\nu^t$. It is very easy to check by induction over step $h$, that for any $h \in [H]$, and all policy $\mu$ simultaneously:
\begin{align*}
    \tilde{P}^{\mu, t}(x_h, a_h) = P^{\mu, \nu^t}(x_h, a_h) = \sum_{s_h \in x_h}  \sum_{b_h \in \cB} \mu_{1:h}(x_h, a_h) p_{1:h}(s_h) \nu^t_{1:h}(y_h(s_h), b_h).
\end{align*}
This proves that the joint distribution:
\begin{equation} \label{eq:joint_state}
    \tilde{P}^{\mu, t}(x_1, a_1, \ldots, x_H, a_H) = P^{\mu, \nu^t}(x_1, a_1, \ldots, x_H, a_H).
\end{equation}
Finally, the construction of adversarial reward is such that its conditional distribution given $(x_H, a_H)$ is exactly the same as the conditional distribution of the reward in the EFG:
\begin{equation*}
    \tilde{R}_H^t( r = r(s_H, a_H, b_H) |x_H, a_H) = P^{\mu, \nu^t}(r = r(s_H, a_H, b_H) | x_H, a_H),
\end{equation*}
which immediately gives that:
\begin{equation} \label{eq:conditional_reward}
    \tilde{P}^{\mu, t}(r_H | x_1, a_1, \ldots, x_H, a_H) = P^{\mu, \nu^t}(r| x_1, a_1, \ldots, x_H, a_H).
\end{equation}
Combining \eqref{eq:joint_state} and \eqref{eq:conditional_reward}, we finish the proof.
\end{proof}

\subsection{Online-to-batch conversion}
\label{appendix:online-to-batch}

We consider an EFG with $m$ players (e.g. using the definition in Section~\ref{appendix:efg-def}), where each player faces an equivalent Tree-Form Adversarial MDP. For any product policy $\pi=\set{\pi_i}_{i\in[m]}$, let $\ell^{\pi_{-i}}$ denote the expected loss function for the \ith player if that the other players play policy $\pi_{-i}$. We define a correlated policy $\bar\pi$ as a probability distribution over product policies, i.e. $\pi\sim\bar\pi$ gives a product policy $\pi$.

An EFCE of the game is defined as follows~\citep{celli2020no,farina2021simple}.
\begin{definition}[Extensive-form correlated equilibrium]
A correlated policy $\bar\pi$ is an \emph{$\epsilon$-approximate Extensive-Form Correlated Equilibrium} (EFCE) of the EFG if
\begin{align*}
     \max_{i\in[m]} \max_{\phi \in \Phi_{i}^\efce}
    \E_{\pi \sim \wb\pi}\paren{\<\phi \pi_i, \ell^{\pi_{-i}}\> - \<\pi_i, \ell^{\pi_{-i}}\>
    }
    \le \epsilon.
\end{align*}
We say $\bar\pi$ is an (exact) EFCE if the above is equality.
\end{definition}

When the game is played with product policies for $T$ rounds, suppose the product policy at round $t$ is $\pi^t$, the extensive-form trigger regret \eqref{eqn:efceregret} for the \ith player becomes
\begin{align*}
    \efceregret_i(T) = \max_{\phi \in \Phi^\efce_i} \sum_{t=1}^T  \<\phi \pi^t_i-\pi_i^t, \ell^{\pi_{-i}^t}\>.
\end{align*}
The following online-to-batch lemma for EFCE is standard, see e.g.~\citep{celli2020no}.
\begin{lemma}[Online-to-batch for EFCE]\label{lem:online-to-batch}
Let $\{ \pi^t= (\pi_i^t)_{i \in [n]} \}_{t \in [T]}$ be a sequence of product policies for all players over $T$ rounds. Then, for the average (correlated) policy $\bar\pi={\rm Unif}(\{\pi^t\}_{t=1}^T)$ is 
an $\epsilon$-EFCE, where $\epsilon = \max_{i\in[m]} \efceregret_i(T)/T$.
\end{lemma}

\subsection{Properties}
\label{appendix:efg-properties}

For any $h < h'$ and $x_{h} \in \cX_{h}$, we let $\cC_{h'}(x_{h}, a_{h})\equiv \{ x \in \cX_{h'}:x  \succ (x_{h}, a_{h})  \}$ and $\cC_{h'}(x_{ h}) \equiv \{ x \in \cX_{h'}:x \succeq x_{h}  \} = \cup_{a_{h} \in \cA} \cC_{h'}(x_{h}, a_{h})$ denote the infosets within the $h'$-th step that are reachable from (i.e. children of) $x_{h}$ or $(x_{h}, a_{h})$, respectively. For shorthand, let $\cC(x_{h}, a_{h}) \defeq \cC_{h+1}(x_{h}, a_{h})$ and $\cC(x_{h}) \defeq \cC_{h+1}(x_{h})$ denote the set of immediate children. 

We define $X_{\succeq x_h}$ for any $x_h\in\cX_h$ as %
\begin{align}
    X_{\succeq x_h} \defeq \sum_{h' = h}^H |\cC_{h'}(x_h)|.
\end{align}
It can be interpreted as the number of infosets in the subtree rooted at $x_h$.
\begin{lemma}
The $L^1$ norm of a sequence form is upper bounded by $\lone{\Pi} \le X$.
\end{lemma}
\begin{proof}
We prove the claim by induction over the root of the subtree, for $h=H,\cdots,1$. When $h=H$, for each infoset $x_h$, the sequence form is just a probability distribution, which sums up to $\lone{\Pi^{x_h}} = 1 = |X_{\succeq x_h}|$. If the claim holds for $h+1$, consider an infoset $x_h$ in the $h$-th level.  By induction hypothesis we have 
$$
\lone{\Pi^{x_h}} = \max_{a_h}\sum_{x_{h+1} \succeq (x_h,a_h)} \lone{\Pi^{x_{h+1}}} \le \max_{a_h}\sum_{x_{h+1} \succeq (x_h,a_h)} |X_{\succeq x_{h+1}}| \le |X_{\succeq x_h}|.
$$
So the equation above holds for any $x_h$. Setting $x_h = \emptyset$ gives $\| \Pi \|_1 \le X$ which completes the proof. 
\end{proof}

\begin{lemma}
\label{lem:triger-set-card}
We have $|\Phi_0^\efce| \le X A^{\|\Pi\|_1+1}$.
\end{lemma}

\begin{proof}
By Proposition 5.1 of \cite{farina2022kernelized}, $\mc{V} \le A^{\|\Pi\|_1}$. Since there are at most $XA$ different infoset-action pair to be trigger, we have $|\Phi_0^\efce| \le X A^{\|\Pi\|_1+1}$.
\end{proof}
\section{Proofs for Section \ref{sec:phi-hedge-implementation} \&~\ref{sec:phi-hedge-dilated-entropy}}
\label{app:phi-hedge-efce}

\subsection{Incremental (OMD) form of Algorithm~\ref{algorithm:efce-ftrl}}
\label{sec:efce-omd}

We first present an incremental update of $(\lambda^{t+1}, m^{t+1} )$ from $(\lambda^{t}, m^{t} )$ as in Algorithm~\ref{algorithm:efce-omd}. We set the initial values of these variables as 
\begin{align}\label{eqn:lambda-m-init}
 \lambda_{x_ga_g}^1 \propto_{x_ga_g} \exp\Big\{    F_{x_g}^0 \Big\}, \quad \quad m^1_{x_ga_g, h}(a_h \vert x_h) \propto_{a_h}\exp\Big\{ \sum_{x_{h+1} \in \cC(x_h, a_h)} F_{x_{h+1}}^0 \Big\},
\end{align}
where for any $x_h \succeq x_g$, $F_{ x_h}^0$ is recursively defined as 
\begin{equation*}
    F_{x_h}^0 \defeq \log \sum_{a_h \in \cA} \exp\Big\{ \sum_{x_{h+1} \in \cC(x_h,a_h)} F_{x_{h+1}}^0 \Big\}. 
\end{equation*}

When the summation is over an empty set $\cC(x_h,a_h)$, the sum should be understood as zero. Here $F_{x_h}^0$ has an intuitive meaning: it is the logarithm of the number of deterministic sequence-form policies starting from $x_h$, and can be computed by the above sum-product formulation recursively. %

Algorithm~\ref{algorithm:efce-omd} is computationally more efficient than Algorithm~\ref{algorithm:efce-ftrl}
 when the loss estimator is sparse. For example, with bandit feedback, we need to update the loss matrix for at most $H$ infoset-action pairs, and thus incur at most $H^3$ operations in Algorithm~\ref{algorithm:efce-omd}. On the contrary, Algorithm~\ref{algorithm:efce-ftrl} requires $O((XA)^2)$ operations to update the policy in each iteration. 
 
 We now prove that Algorithm~\ref{algorithm:efce-omd} and Algorithm~\ref{algorithm:efce-ftrl} are actually equivalent.
 
\begin{algorithm}[t]
\caption{EFCE-OMD (OMD form; equivalent FTRL form in Algorithm \ref{algorithm:efce-ftrl})}
\label{algorithm:efce-omd}
\begin{algorithmic}[1]
\REQUIRE Learning rate $\eta$. 
\STATE Initialize $\lambda_{x_g a_g}^1$, and $m_{x_ga_g, h}^1(a_h \vert x_h)$, for all $(g, x_g, a_g, h, x_h, a_h)$ with $g \le h$ using Eq. (\ref{eqn:lambda-m-init}). 
\FOR{$t = 1, 2, \ldots, T$}
\STATE Compute $\phi^t = \phi(\lambda^t, m^t)$ where $\phi$ is in Eq. (\ref{equation:phi-lambda-m}).
\STATE Compute the policy $\mu^t$, which is a solution of the fixed point equation $\mu = \phi^t \mu$. 
\STATE Receive loss $\ell^t = \{\ell^t_h(x_h, a_h)\}_{(x_h, a_h) \in \cX \times \cA} \in \R^{XA}_{\ge 0}$.
\STATE Compute matrix loss $M^t = \ell^t (\mu^t)^\top \in \R_{\ge 0}^{XA \times XA}$. 
\STATE For each $x_g a_g \in \cX \times \cA$, from the reverse order of $x_h$, compute $m^{t+1}_{x_ga_g, h}(a_h \vert x_h)$ and $\wt{F}^{t}_{x_ga_g, x_h}$
\begin{align*}
& m^{t+1}_{x_ga_g, h}(a_h \vert x_h) \propto_{a_h} m^{t}_{x_ga_g, h}(a_h \vert x_h) \exp\Big\{ - \eta M^t_{x_ha_h, x_ga_g} + \sum_{x_{h+1} \in \cC(x_h, a_h)} \wt{F}^{t}_{x_ga_g, x_{h+1}} \Big\},  \\
&~ \wt{F}^{t}_{x_ga_g, x_h} = \log  \sum_{a_h\in \cA} m^{t}_{x_ga_g, h}(a_h|x_h) \exp\Big\{ -\eta  M^t_{x_ha_h, x_ga_g} + \sum_{x_{h+1} \in \cC(x_h, a_h)} \wt{F}^{t}_{x_ga_g, x_{h+1}} \Big\} ,
\end{align*}
\STATE Compute $\lambda_{x_g a_g}^{t+1}$ as
\begin{align*}
& \lambda_{x_ga_g}^{t+1} \propto_{x_ga_g} \lambda_{x_ga_g}^{t}  \exp\Big\{ - \eta\langle I - E_{\succeq x_g a_g}, M^t\rangle  + \wt{F}^{ t}_{x_ga_g,x_g} \Big\}.  
\end{align*}
    \ENDFOR
 \end{algorithmic}
\end{algorithm}

\begin{lemma}
  \label{lem:efce-omd}
  Given the same sequence of $M^t$, 
  Algorithm~\ref{algorithm:efce-ftrl} and Algorithm~\ref{algorithm:efce-omd} outputs the same  $\lambda^t$ and $m^t$ and thus the same $\phi^t$.
\end{lemma}

\begin{proof}
We only need to prove for any $x_ga_g,x_h$ and $t$, $F^t_{x_ga_g, x_h} = \sum_{s=1}^t\wt{F}^{s}_{x_ga_g, x_h}$. Then $\lambda^t$ and $m^t$ will be the same in Algorithm~\ref{algorithm:efce-ftrl} and Algorithm~\ref{algorithm:efce-omd}.

We prove the above claim by induction. For the base case, the claim clearly holds if $h=H+1$ or $t=1$ by definition. Assume this holds at $t-1$ and $h+1$, then at the $h$-th step in Algorithm~\ref{algorithm:efce-omd}, 
\begin{align*}
    \wt{F}^{t}_{x_ga_g, x_h} = &\log  \sum_{a_h\in \cA} m^{t}_{x_ga_g, h}(a_h|x_h) \exp\Big\{ -\eta  M^t_{x_ha_h, x_ga_g} + \sum_{x_{h+1} \in \cC(x_h, a_h)} \wt{F}^{t}_{x_ga_g, x_{h+1}} \Big\}\\
    =&\log  \sum_{a_h\in \cA}  \exp\Big\{ -\eta  \sum_{s=1}^tM^s_{x_ha_h, x_ga_g} + \sum_{x_{h+1} \in \cC(x_h, a_h)} F^{t}_{x_ga_g, x_{h+1}} \Big\}\\
    &-\log  \sum_{a_h\in \cA}  \exp\Big\{ -\eta  \sum_{s=1}^{t-1}M^s_{x_ha_h, x_ga_g} + \sum_{x_{h+1} \in \cC(x_h, a_h)} F^{t-1}_{x_ga_g, x_{h+1}} \Big\}\\
    =& F^{t}_{x_ga_g, x_h}-F^{t-1}_{x_ga_g, x_h}.
\end{align*}

Thus $F^{t}_{x_ga_g, x_h} =  \sum_{s=1}^t\wt{F}^{s}_{x_ga_g, x_h}$. We completes the proof by noticing at $H+1$ step, $F^{t}_{x_ga_g, x_{H+1}} = \wt{F}^{t}_{x_ga_g, x_{H+1}} = 0$.
\end{proof}

\subsection{Proof of Lemma \ref{lem:free-energy-trick}}
\label{appendix:proof-free-energy-trick}

By Line~\ref{line:phi-hedge-p-update} of Algorithm~\ref{algorithm:phi-regret-minimization}, we have
\begin{equation}
p_\phi^{t+1} = \frac{p_\phi^{t} \cdot \exp\{ - \eta \langle \phi \mu^t, \ell^t \rangle  \}}{\sum_{\phi'} p_{\phi'}^{t} \cdot \exp\{ - \eta \langle {\phi'} \mu^t, \ell^t \rangle  \}} =  \frac{p_\phi^{t} \cdot \exp\{ - \eta \langle \phi,M^t \rangle  \}}{\sum_{\phi'} p_{\phi'}^{t} \cdot \exp\{ - \eta \langle \phi',M^t \rangle  \}}.
\end{equation}

Repeating this update and using the uniform initialization, we have 
\[
p_\phi^{t+1} = \frac{\exp\{ - \eta \langle \phi,\sum_{s=1}^t M^s \rangle \}}{\sum_{\phi'} \exp\{ - \eta \langle \phi',\sum_{s=1}^t M^s \rangle \}}. 
\]
As a result, we have
\begin{equation}
\phi^{t} = \sum_{\phi}p_\phi^{t}\phi = \frac{\sum_{\phi}\exp\{ - \eta \langle \phi,\sum_{s=1}^{t-1}M^s \rangle \}\phi}{\sum_{\phi}\exp\{ - \eta \langle \phi,\sum_{s=1}^{t-1}M^s \rangle \}} = -\grad F^{\Phi_0}\paren{ \eta\sum_{s=1}^{t-1} M^s }.
\end{equation}
This proves the lemma. 
\qed

\subsection{Proof of Lemma \ref{lem:evaluate-partition-function-gradient}}
\label{appendix:proof-evaluate-partition-function-gradient}

For any $x_h \succeq x_g$, we define $F_{x_g a_g, x_h}(M)$ by 
\begin{align*}
F_{x_g a_g, x_h}(M) \defeq \log \sum_{m_{x_ga_g}\in\mc{V}^{x_h}} \exp(-\langle m_{x_ga_g} e_{x_ga_g}^\top, M\rangle ).
\end{align*}

Note that for any $\phi \in \Phi_0^\trig$, there exists a unique $(g, x_g, a_g, m_{x_g a_g}) \in [H] \times \cX \times \cA \times \mc{V}^{x_g}$ such that $\phi = \phi_{x_g a_g \to m_{x_g a_g}}$. As a consequence, we have 
\begin{align*}
   F^{\efce}(M) =&~ \log \sum_{\phi\in\Phi^{\efce}_0} \exp(-\langle\phi, M\rangle ) \\
   =&~ \log\sum_{g,x_g,a_g} \sum_{m_{x_ga_g}\in\mc{V}^{x_g}} \exp(-\langle\phi_{x_ga_g\to m_{x_ga_g}}, M\rangle ) \\
   =& \log\sum_{g,x_g,a_g} \sum_{m_{x_ga_g}\in\mc{V}^{x_g}} \exp(-\langle I - E_{\succeq x_ga_g} + m_{x_ga_g} e_{x_ga_g}^\top, M\rangle )\\
   =&~ \log \sum_{g, x_g, a_g} \exp\Big\{  - \langle I - E_{\succeq x_ga_g}, M\rangle  + F_{x_g a_g, x_g}(M) \Big\}.
\end{align*}

It remains to evaluate $F_{x_g a_g, x_h}(M)$ recurrently, which is handled by the structure of $\mc{V}^{x_h}$ as follows:
\begin{align*}
&F_{x_g a_g, x_h}(M) \\
=& \log \sum_{m_{x_ga_g}\in\mc{V}^{x_h}} \exp(-\langle m_{x_ga_g} e_{x_ga_g}^\top, M\rangle ) \\
=& \log \sum_{a_h \in \cA}  \exp(- M_{x_ha_h, x_g a_g}) \prod_{x_{h+1} \in \cC(x_ha_h)} \sum_{m_{x_{h+1}a_{h+1}}\in\mc{V}^{x_{h+1}}}  \exp(-\langle m_{x_{h+1}a_{h+1}} e_{x_ga_g}^\top, M\rangle )\\
=& \log \sum_{a_h \in \cA} \exp\Big\{ - M_{x_ha_h, x_g a_g} + \sum_{x_{h+1} \in \cC(x_ha_h)} \log \sum_{m_{x_{h+1}a_{h+1}}\in\mc{V}^{x_{h+1}}} \exp(-\langle m_{x_{h+1}a_{h+1}} e_{x_ga_g}^\top, M\rangle ) \Big\}\\
=& \log \sum_{a_h \in \cA} \exp\Big\{ - M_{x_ha_h, x_g a_g} + \sum_{x_{h+1} \in \cC(x_ha_h)} F_{x_g a_g, x_{h+1}}(M) \Big\}.
\end{align*}
This proves Eq. (\ref{eqn:F-trig-M}) and (\ref{eqn:F-M-FTRL-form}). 

Calculating the gradient, we have
\begin{align}
    -\grad F^{\efce}(M) =& \frac{\sum_{g, x_g, a_g}\exp\Big\{  - \langle I - E_{\succeq x_ga_g}, M\rangle  + F_{x_g a_g, x_g}(M)  \Big\}\big[I - E_{\succeq x_ga_g}-\grad F_{x_g a_g, x_h}(M)\big]}{\sum_{g, x_g, a_g} \exp\Big\{  - \langle I - E_{\succeq x_ga_g}, M\langle \rangle + F_{x_g a_g, x_g}(M)  \Big\}} \nonumber \\
    =& \sum_{g, x_g, a_g} \lambda_{x_g, a_g} \big[I - E_{\succeq x_ga_g}-\grad F_{x_g a_g, x_h}(M)\big].\label{eqn:grad_F-tr}
\end{align}

It remains to compute $\grad F_{x_g a_g, x_h}(M)$. By the recurrent formula, we have
\begin{align*}
 &-\grad F_{x_g a_g, x_h}(M)  \\
 =& \frac{\sum_{a_h \in \cA}\exp\Big\{  - M_{x_ha_h, x_g a_g} +  \sum_{x_{h+1} \in \cC(x_ha_h)} F_{x_g a_g, x_{h+1}}(M) \Big\}\big[e_{x_ha_h}e_{x_ga_g}^\top -\sum_{x_{h+1} \in \cC(x_ha_h)} \grad F_{x_g a_g, x_{h+1}}(M) \big]}{\sum_{a_h \in \cA} \exp\Big\{  - M_{x_ha_h, x_g a_g} +  \sum_{x_{h+1} \in \cC(x_ha_h)} F_{x_g a_g, x_{h+1}}(M) \Big\}} \\ 
 =&\sum_{a_h \in \cA} m_{x_ga_g, h}(a_h \vert x_h) \big[e_{x_ha_h}e_{x_ga_g}^\top +\sum_{x_{h+1} \in \cC(x_ha_h)} (-\grad F_{x_g a_g, x_{h+1}})(M) \big].
\end{align*}
This gives a recursion formula for $- \grad F_{x_g a_g, x_h}(M)$. Solving this recursion formula, we get
$$
-\grad F_{x_g a_g, x_h}(M) = m_{x_ga_g} e_{x_g a_g}^\top,
$$
Plugging this into Eq. (\ref{eqn:grad_F-tr}) completes the proof. 
\qed

\subsection{Runtime of Algorithm~\ref{algorithm:efce-ftrl}}
\label{appendix:runtime}

Here we explain how Lemma~\ref{lem:evaluate-partition-function-gradient} and its execution in Algorithm~\ref{algorithm:efce-ftrl} is an $O(X^2A^2)$ time (in floating-point operations) efficient implementation of $-\grad F^{\trig}(M)$ for any matrix $M\in\R^{XA\times XA}$.

First, the function value $F^{\trig}(M)$ can be recursively evaluated using~\eqref{eqn:F-trig-M} \&~\eqref{eqn:F-M-FTRL-form}, where we first evaluate~\eqref{eqn:F-M-FTRL-form} for any $x_ga_g\in\cX\times\cA$ recursively in a bottom-up fashion over $\set{x_h:x_h\succeq x_g}$ (i.e. the subtree rooted at $x_g$) up until $x_h=x_g$, and then plug in the resulting values of $F_{x_ga_g, x_g}(M)$ into~\eqref{eqn:F-trig-M} to obtain $F^{\trig}(M)$. This process costs $O(XA)$ operations for each $x_g a_g$, so in total costs $O(X^2 A^2)$ operations. Second,~\eqref{eqn:lambda-M-FTRL-form}-\eqref{equation:phi-lambda-m} show that the gradient can be obtained without much extra cost: By~\eqref{equation:phi-lambda-m}, $-\grad F^{\trig}(M)$ is determined by the parameters $(\lambda, m)$, which then by~\eqref{eqn:lambda-M-FTRL-form} \&~\eqref{eqn:m-M-FTRL-form} are exactly the ratios of the recursive log-sum-exps which we already evaluated in the previous step, and thus can be directly yielded (with cost of the same-order) while evaluating $F^{\trig}(M)$. So the total runtime of the recursive computations in Lemma~\ref{lem:evaluate-partition-function-gradient} (i.e. Algorithm~\ref{algorithm:efce-ftrl}) is $O(X^2 A^2)$.

\subsection{Dilated entropy and dilated KL}
\label{appendix:def-dilated}
We introduce the subtree dilated entropy and subtree dilated KL divergence, a variant of dilated entropy and dilated KL divergence introduced in \cite{hoda2010smoothing,kroer2015faster,kozuno2021model}. Here the modification we made is that we allow these quantities to be rooted at any infoset $x_g$. The original version is recovered when the choose the full game tree as the subtree. These quantities were further used to define the trigger dilated entropy and trigger dilated KL divergence as in Section \ref{sec:phi-hedge-dilated-entropy}. 

\begin{definition}[Dilated entropy and Dilated KL divergence]
The dilated entropy $H_{x_g}$ rooted at $x_g$ of subtree policy $\mu^{x_g}\in\Pi^{x_g}$ is defined as
\begin{align}
H_{x_g}(\mu^{x_g}) \defeq \sum_{h=g}^H \sum_{(x_h, a_h) \succeq x_g} \mu^{x_g}_{g:h}(x_h, a_h) \log \mu^{x_g}_h(a_h | x_h).
\end{align}

The dilated KL divergence $D_{x_g}$ rooted at $x_g$ between two subtree policies $\mu^{x_g}, \nu^{x_g}\in\Pi^{x_g}$ is defined as
\begin{align}
    D_{x_g}(\mu^{x_g} \| \nu^{x_g}) \defeq \sum_{h=g}^H \sum_{(x_h, a_h) \succeq x_g} \mu^{x_g}_{g:h}(x_h, a_h) \log \frac{\mu^{x_g}_h(a_h | x_h)}{\nu^{x_g}_h(a_h | x_h)}.
\end{align}
\end{definition}

\subsection{Proof of Lemma \ref{lem:equivalence-algorithm}}
\label{appendix:proof-equivalence-algorithm}

We check solving optimization problem~\eqref{equation:phit-H-form} will result in exactly the same form of Algorithm~\ref{algorithm:efce-ftrl}. The OMD form~\eqref{equation:phit-D-form} is similar.

Using the definition of $H^{\efce}(\lambda, m)$, the objective function in \eqref{equation:phit-H-form} can be written as 
\begin{align*}
   H(\lambda )+ \sum_{g,x_ga_g}\lambda_{x_ga_g}\left[\eta\langle I - E_{\succeq x_g a_g}, \sum_{s=1}^tM^s\rangle  +\eta\langle m_{x_ga_g}, \sum_{s=1}^tM^s_{\cdot,x_ga_g}\rangle +H_{x_g}(m_{x_g a_g})\right].
\end{align*}

First fix $\lambda_{\cdot}$ and consider $m_{x_ga_g}$, which is just to minimize 
$
    \eta\langle m_{x_ga_g},\sum_{s=1}^tM^s_{\cdot,x_ga_g}\rangle  +H_{x_g}(m_{x_g a_g} )
$.

This is similar to  form studied in Appendix B of \cite{kozuno2021model} (or see Lemma~\ref{lem:dilated-ent-optimization} for a full proof), which implies that, the optimum is achieved at 
\begin{equation*}
    m^{t+1}_{x_ga_g, h}(a_h \vert x_h) \propto_{a_h} \exp\Big\{  \sum_{s=1}^t \big[ - \eta M^s_{x_ha_h, x_ga_g} + \sum_{x_{h+1} \in \cC(x_h, a_h)} F^{t}_{x_ga_g, x_{h+1}} \big] \Big\},
\end{equation*}
where
\begin{align*}
 F^{t}_{x_ga_g, x_h} =  \log  \sum_{a_h\in \cA}  \exp\Big\{ \sum_{s=1}^t \big[ -\eta M^s_{x_ha_h, x_ga_g}+ \sum_{x_{h+1} \in \cC(x_h, a_h)} F^{t}_{x_ga_g, x_{h+1}} \big]\Big\}.   
\end{align*}

Plug in the optimal $m_{x_ga_g}$, the object now becomes
\begin{align*}
  H(\lambda )+ \sum_{g,x_ga_g}\lambda_{x_ga_g}\left[\eta\langle I - E_{\succeq x_g a_g}, \sum_{s=1}^tM^s\rangle  -F^{t}_{x_ga_g, x_g}\right].
\end{align*}

This is a standard KL-regularized linear optimization problem on simplex. The optimum is achieved at 
\begin{equation*}
    \lambda_{x_ga_g}^{t+1} \propto_{x_ga_g} \exp\Big\{   - \eta\langle I - E_{\succeq x_g a_g}, \sum_{s=1}^tM^s\rangle  + F^{ t}_{x_ga_g,x_g}  \Big\}.
\end{equation*}
This gives the update of $\lambda^{t+1}$ and $m^{t+1}$ as in Algorithm \ref{algorithm:efce-ftrl}. This completes the proof. 
\qed

\section{Proofs for Section~\ref{sec:phi-hedge-regret}}

\subsection{Proof of Theorem \ref{thm:iteration-complexity-full}}\label{app:iteration-complexity-full-proof}

Using regret bound of $\Phi$-Hedge algorithm (Lemma \ref{lemma:regret-bound-phi-hedge}), we get 
\begin{align*}
    \efceregret(T) \le&~ \frac{\log |\Phi_0^\efce|}{\eta} + \frac{\eta}{2} \sum_{t=1}^T \sum_{\phi \in \Phi_0^\trig} p_\phi^{t} \paren{\langle \phi \mu^t, \ell^t \rangle}^2 \\
    \overset{(i)}{\le} &~ \frac{\log |\Phi_0^\efce|}{\eta} + \frac{\eta}{2} \sum_{t=1}^T \sum_{\phi \in \Phi_0^\trig} p_\phi^t H^2 = \frac{\log |\Phi_0^\efce|}{\eta} + \frac{\eta H^2 T}{2} .
\end{align*}
Here, (i) uses $\langle \phi \mu^t, \ell^t \rangle \in [0,H]$. Note that $\Phi^{\efce}_0 \defeq \bigcup_{g,x_ga_g} \bigcup_{v^{x_g}\in\mc{V}^{x_g}} \set{\phi_{x_ga_g\to v^{x_g}}}$ has cardinality upper bounded by $|\Phi_0^\efce| \le X A^{\|\Pi\|_1+1}$ by Lemma~\ref{lem:triger-set-card}. Substitute this into the regret bound, we have
\begin{align*}
      \efceregret(T) \le \frac{\log (X A^{\|\Pi\|_1+1})}{\eta} + \frac{\eta H^2 T}{2} 
      \le \frac{2\|\Pi\|_1\iota}{\eta} + \frac{\eta H^2 T}{2},
\end{align*}
where $\iota = \log(XA)$ is a log-term.
Choosing $\eta = 2\sqrt{\|\Pi\|_1 \iota/(H^2T)}$ gives $\efceregret(T) \le2\sqrt{H^2\|\Pi\|_1\iota T} $, which completes the proof. 
\qed

\subsection{Proof of Theorem \ref{thm:phi-hedge-efce-bandit}}\label{app:proof-phi-hedge-efce-bandit}

For the loss estimator $\wt \ell$, we have the following lemma (see Lemma 3 in \citet{kozuno2021model}):
\begin{lemma}\label{lemma:lemma3-in-kozuno}
Let $\delta \in (0,1)$ and $\gamma \in (0, \infty)$. Fix $h \in [H]$, and let $\alpha(x_h, a_h)\in [0, 2\gamma]$ be some constant for each $(x_h, a_h) \in \cX_h \times \cA$. Then with probability at least $1-\delta$, we have
\begin{align*}
    \sum_{t=1}^T \sum_{x_h\in \cX_h, a_h \in \cA} \alpha(x_h, a_h) \paren{\wt\ell_h^t(x_h, a_h) - \ell_h^t(x_h, a_h)} \le \log \frac{1}{\delta}.
\end{align*}
\end{lemma}

\begin{proof}[Proof of Theorem \ref{thm:phi-hedge-efce-bandit}]
We have
\begin{align*}
    \efceregret(T) =& \max_{\phi \in \Phi^\efce} \sum_{t=1}^T \langle \mu^t - \phi \mu^t, \ell^t \rangle \\
    \le & \underbrace{\sum_{t=1}^T \langle \mu^t , \ell^t - \wt \ell^t \rangle}_{\rm BIAS_1}   
    + \underbrace{\max_{\phi \in \Phi^\efce}\sum_{t=1}^T \langle \phi \mu^t, \wt \ell^t -  \ell^t \rangle }_{\rm BIAS_{2}}
    + \underbrace{\max_{\phi \in \Phi^\efce}\sum_{t=1}^T \langle \mu^t - \phi \mu^t, \wt \ell^t \rangle}_{\rm REGRET}.
\end{align*}

We use the following three lemmas to bound the terms above respectively. In these lemmas, $\iota = \log (3XA/\delta)$ is a log factor. 
\begin{lemma}[Bound on ${\rm BIAS}_1$]
  \label{lemma:bias_1}
  With probability at least $1-\delta/3$, we have
  $$
  {\rm BIAS}_1 \le H\sqrt{2T\iota}+\gamma XAT.
  $$
\end{lemma}

\begin{lemma}[Bound on ${\rm BIAS}_2$]
  \label{lemma:bias_2}
  With probability at least $1-\delta/3$, we have
  $$
  {\rm BIAS}_2 \le \|\Pi\|_1\iota/\gamma.
  $$
\end{lemma}

\begin{lemma}[Bound on ${\rm REGRET}$]
  \label{lemma:regret_sample}
  With probability at least $1-\delta/3$, we have
  $$
  {\rm REGRET} \le \log|\Phi_0^\efce|/\eta+\eta HXAT+\eta HXA\iota/\gamma..
  $$
\end{lemma}

Lemma \ref{lemma:bias_1}, \ref{lemma:bias_2}, and \ref{lemma:regret_sample} bound bias terms and regret term respectively. Using these lemmas, we have with probability at least $1-\delta$,%
\begin{align*}
    \efceregret(T) \le \frac{\log|\Phi_0^\efce|}{\eta} + \eta HXAT + \eta HXA \iota/ \gamma + \|\Pi\|_1 \iota /\gamma + H\sqrt{2T \iota} + \gamma XAT. 
\end{align*}
By Lemma~\ref{lem:triger-set-card}, $|\Phi^\efce_0| \le XA^{\|\Pi\|_1 + 1}$. We further have
$$
\efceregret(T) \le \frac{2\|\Pi\|_1 \iota}{\eta} + \eta HXAT + \eta HXA \iota/ \gamma + \|\Pi\|_1 \iota /\gamma + H\sqrt{2T \iota} + \gamma XAT.
$$
Choosing $\gamma = \sqrt{\|\Pi\|_1\iota/(XAT)}$ and $\eta = \sqrt{\|\Pi\|_1 \iota/(HXAT)}$ gives
\begin{align*}
    \efceregret(T) \le &~ 5\sqrt{ HXA\|\Pi\|_1T \iota} + XA \iota \sqrt{H}  + H\sqrt{2T\iota} \\
    \le &~ \mc{O}(\sqrt{HXA\|\Pi\|_1 \iota \cdot T} + XA\iota \sqrt{H}),
\end{align*}
where we uses $\|\Pi\|_1 \ge H$.
Notice that there is a “trivial” bound $\efceregret(T) \le HT$.
For $T \ge XA\iota/\|\Pi\|_1$,  we have $XA\iota\sqrt{H} \le \sqrt{HXA\|\Pi\|_1 \iota \cdot T}$, which gives $\efceregret \le \mc{O}(\sqrt{HXA\|\Pi\|_1 \iota \cdot T})$; For $T \le XA\iota/\|\Pi\|_1$,  we have $HT \le \sqrt{HXA\|\Pi\|_1 \iota \cdot T}$, which gives $\efceregret \le HT \le  \mc{O}(\sqrt{HXA\|\Pi\|_1 \iota \cdot T})$. Therefore, we always have
\begin{align*}
    \efceregret \le \mc{O}(\sqrt{HXA\|\Pi\|_1 \iota \cdot T}).
\end{align*}
This completes the proof. 
\end{proof}

Here, we give the proofs of the lemmas we used above.

\begin{proof}[Proof of Lemma \ref{lemma:bias_1}]
We further decompose $\textrm{BIAS}_1$ to two terms by 
$$
{\rm BIAS}_1 = \sum_{t=1}^T \<\mu^t,\ell^t - \wt\ell^t\>=\underset{\left( A \right)}{\underbrace{\sum_{t=1}^T{\left< \mu ^t,\ell ^t-\mathbb{E}\left\{ \widetilde{\ell }^t|\mathcal{F}^{t-1} \right\} \right>}}}+\underset{\left( B \right)}{\underbrace{\sum_{t=1}^T{\left< \mu ^t,\mathbb{E}\left\{ \widetilde{\ell }^t|\mathcal{F}^{t-1} \right\} -\widetilde{\ell }^t \right>}}}.
$$

To bound $(A)$, plug in the definition of loss estimator, 
\begin{align*}
    &~\sum_{t=1}^T{\left< \mu ^t,\ell ^t-\mathbb{E}\left\{ \widetilde{\ell }^t|\mathcal{F}^{t-1} \right\} \right>}\\
    =&~\sum_{t=1}^T{\sum_{h=1}^H{\sum_{x_h,a_h}{\mu_{1:h}^{t}(x_h,a_h)\left[ \ell_h ^t(x_h,a_h)-\frac{\mu _{1:h}^{t}(x_h,a_h)\ell_h ^t(x_h,a_h)}{\mu _{1:h}^{t}(x_h,a_h)+\gamma } \right]}}}
\\
=&~\sum_{t=1}^T{\sum_{h=1}^H{\sum_{x_h,a_h}{\mu_{1:h}^{t}(x_h,a_h)\ell_h ^t(x_h,a_h)\left[ \frac{\gamma }{\mu _{1:h}^{t}(x_h,a_h)+\gamma } \right]}}}
\\
\le&~ \gamma  \sum_{t=1}^T{\sum_{h=1}^H{\sum_{x_h,a_h}{ \ell_h^t(x_h, a_h)}}} \le \gamma XAT,
\end{align*}
where the last inequality is by $\ell_h^t(x_h, a_h) \in [0, 1] $.

To bound $(B)$, first notice 
\begin{align*}
    \left< \mu ^t,\widetilde{\ell }^t \right> =&~{\sum_{h=1}^H{\sum_{x_h,a_h}{\mu  _{1:h}^{t}(x_h,a_h)\frac{\indic{(x_h^t, a_h^t)=(x_h, a_h) } \cdot (1 - r_h^t)}{\mu _{1:h}^{t}(x_h,a_h)+\gamma } }}}
\\
\le&~ \sum_{h=1}^H{\sum_{x_h,a_h}{\ones\left\{ x_h=x_{h}^{t},a_h=a_{h}^{t} \right\}}}=\sum_{h=1}^H{1}=H.
\end{align*}
Then by Azuma-Hoeffding, with probability at least $1-\delta/3$, we have
$$
\sum_{t=1}^T{\left< \mu ^t,\mathbb{E}\left\{ \widetilde{\ell }^t|\mathcal{F}^{t-1} \right\} -\widetilde{\ell }^t \right>}\le H\sqrt{2T\log(3/\delta)} \le H\sqrt{2T\iota}.
$$
Combining the bounds for (A) and (B) gives the desired result.
\end{proof}

\begin{proof}[Proof of Lemma \ref{lemma:bias_2}] %

We have with probability at least $1-\delta/3$,
\begin{align*}
    & {\rm BIAS}_2 = \max_{\phi\in\Phi^\efce} \sum_{t=1}^T{\left< \phi\mu^t ,\widetilde{\ell }^t-\ell ^t \right>}
\\
=& \max_{\phi\in\Phi^\efce} \sum_{t=1}^T{\sum_{h=1}^H{\sum_{x_h,a_h}{(\phi \mu^t) _{1:h}(x_h,a_h)\left[ \wt\ell_h^t(x_h,a_h)-\ell_h ^t(x_h,a_h) \right]}}}
\\
=& \max_{\phi\in\Phi^\efce}  \sum_{h=1}^H{\sum_{x_h,a_h}{\frac{(\phi \mu^t) _{1:h}(x_h,a_h)}{\gamma }\sum_{t=1}^T{\gamma \left[ \wt\ell_h^t(x_h,a_h)-\ell_h ^t(x_h,a_h) \right]}}}
\\
\overset{\left( i \right)}{\le} &~\frac{\log \left( 3 XA/\delta \right)}{\gamma}\max_{\phi\in\Phi^\efce}\sum_{h=1}^H   {\sum_{x_h,a_h}{(\phi \mu^t)_{1:h}(x_h,a_h)}}
\\
\le &~\|\Pi\|_1\iota/\gamma ,
\end{align*}
where $(i)$ is by applying Lemma \ref{lemma:lemma3-in-kozuno} for each $(x_h,a_h)$ pair. To be more specific, we choose $\alpha(x_h',a_h') = \gamma\indic{(x_h',a_h') = (x_h, a_h)}$, then Lemma \ref{lemma:lemma3-in-kozuno} yields that
\begin{align*}
    \sum_{t=1}^T \gamma \left[ \wt\ell_h^t(x_h,a_h)-\ell_h ^t(x_h,a_h) \right] \le \log\frac{3XA}{\delta}
\end{align*}
with probability at least $1-\delta/(3XA)$. Then taking union bound gives the inequality in $(i)$.%
\end{proof}

\begin{proof}[Proof of Lemma \ref{lemma:regret_sample}]
Note that by Lemma \ref{lemma:regret-bound-phi-hedge}, we have
\begin{align*}
    \textrm{REGRET} \le \frac{\log \vert \Phi_0^\trig\vert}{\eta} + \frac{\eta}{2} \sum_{t=1}^T{\sum_{\phi \in \Phi_0^\trig} {p_{\phi}^{t}}(\langle \phi \mu^t, \wt\ell^t \rangle)^2}.
\end{align*}

To bound the second term, we have 
\begin{align*}
  &\sum_{t = 1}^T \sum_{\phi \in \Phi_0^\trig} {p_\phi^{t}}(\langle \phi \mu^t, \wt\ell^t \rangle)^2 \\
  \le&~ 2\sum_{h'\ge h} \sum_{t = 1}^T  \sum_{x_h,a_h} \sum_{(x_{h'},a_{h'}) \in \cC_{h'}(x_h, a_h)} \sum_{\phi \in \Phi_0^\trig} p_\phi^t (\phi \mu^t)_{1:h}(x_h,a_h) (\phi \mu^t)_{1:h'}(x_{h'},a_{h'})  \wt \ell_{h}^t(x_{h},a_{h})\wt \ell_{h'}^t(x_{h'},a_{h'})\\
  \le&~ 2\sum_{h'\ge h} \sum_{t = 1}^T  \sum_{x_h,a_h} \sum_{(x_{h'},a_{h'}) \in \cC_{h'}(x_h, a_h)} \sum_{\phi \in \Phi_0^\trig} \frac{p_\phi^t(\phi \mu^t)_{1:h}(x_h,a_h)}{\mu _{1:h}^{t}(x_h,a_h) + \gamma} \wt \ell_{h'}^t(x_{h'},a_{h'})\\
 \le&~ 2\sum_{h'\ge h}  \sum_{t=1}^T  \sum_{x_h,a_h} \sum_{(x_{h'},a_{h'})\in \cC_{h'}(x_h, a_h)} \sum_{\phi \in \Phi_0^\trig} \frac{p_\phi^t(\phi \mu^t)_{1:h}(x_h,a_h)}{\mu _{1:h}^{t}(x_h,a_h)} \wt \ell_{h'}^t(x_{h'},a_{h'})\\
 \overset{(i)}{=}&~ 2\sum_{h'\ge h} \sum_{t=1}^T \sum_{x_h,a_h} \sum_{(x_{h'},a_{h'})\in \cC_{h'}(x_h, a_h)} \wt \ell_{h'}^t(x_{h'},a_{h'})\\
 \overset{\left( ii \right)}{\le} &~ 2\sum_{h'\ge h}  \Big(\sum_{t=1}^T\sum_{x_h,a_h} \sum_{(x_{h'},a_{h'})\in \cC_{h'}(x_h, a_h)}{\ell}_{h'}^t(x_{h'},a_{h'}) +X_hA\iota/\gamma\Big)\\
 \le&~ 2 H X A T + 2 H X A \iota/\gamma,
\end{align*}
where $(i)$ uses that $\mu^t$ is the solution of the fixed point equation $\mu^t = \sum_{\phi \in \Phi_0^\trig} p_\phi^t \phi \mu^t$; $(ii)$ is by Lemma \ref{lemma:lemma3-in-kozuno}, which gives 
\begin{align*}
    \sum_{t=1}^T \sum_{(x_{h'},a_{h'})\in \cC_{h'}(x_h, a_h)} \gamma \paren{\wt \ell_{h'}^t(x_{h'},a_{h'}) - {\ell}_{h'}^t(x_{h'},a_{h'})} \le \log \frac{3XA}{\delta}
\end{align*}
with probability at least $1-\delta/(3XA)$ (choosing $\alpha(x_h',a_h') = \gamma \indic{(x_h', a_h') \in \cC_{h'}(x_h, a_h)}$ in the lemma). Then taking union bound yields that $(ii)$ holds with probability at least $1-\delta/3$. 

Finally, putting everything together, the lemma is proved. 
\end{proof}

\section{Omitted details in Section~\ref{sec:sharp-rate}}
\label{sec:sharp-rate-details}

\subsection{Balanced exploration policy}
\label{sec:balanced-exploration-policy}
We define the balanced exploration policies formally as introduced in Definition 2 in  \citet{bai2022near}. For any $1\le h < h' \le H$ and $(x_{h'}, a_{h'}) \in \cX_{h'} \times \cA$, we denote $\cC_h(x_{h'}, a_{h'}) \defeq \{ x_h \in \cX_h: x_{h} \succ (x_{h'}, a_{h'})\}$ to be the set of children of $(x_{h'}, a_{h'})$ at layer $h$. 

\begin{definition}[Balanced exploration policy]
\label{def:balanced-exploration-policy}
For any $1\le h\le H$, the \emph{balanced exploration policy for layer $h$}, denoted as $\mu^{\star, h}\in\Pi$, is defined as 
  \begin{align}
    \label{equation:balanced-policy}
    \mu^{\star, h}_{h'}(a_{h'} | x_{h'}) \defeq \left\{
    \begin{aligned}
      & \frac{ \abs{\cC_{h}(x_{h'}, a_{h'})} }{ \abs{\cC_h(x_{h'})} },& h'\in\set{1,\dots,h-1}, \\
      &1/A, & h'\in\set{h,\dots,H}.
    \end{aligned}
        \right.
  \end{align}
\end{definition}

In words, at time steps $h'\le h-1$, the policy $\mu^{\star, h}$ plays actions proportionally to their number of descendants \emph{within the $h$-th layer} of the game tree. Then at time steps $h'\ge h$, it plays the uniform policy.

A crucial property of the balanced exploration policy is the balancing property (Lemma C.4 in \citet{bai2022near}). We also include it here for convience.

\begin{lemma}[Balancing property of $\mu^{\star, h}$]
  \label{lemma:balancing}
  For policy $\mu\in\Pi$ and any $h\in[H]$, we have
  \begin{align*}
    \sum_{(x_h, a_h)\in \mc{X}_h\times \mc{A}} \frac{\mu_{1:h}(x_h, a_h)}{\mu^{\star, h}_{1:h}(x_h, a_h)} = X_hA.
  \end{align*}
  Moreover, for any $x_g$ as the root of a subtree, we have
  \begin{align*}
    \sum_{(x_h, a_h)\succeq x_g} \frac{\mu_{g:h}(x_h, a_h)}{\mu^{\star, h}_{g:h}(x_h, a_h)} = \abs{\cC_{h}(x_{g})} \cdot A.
  \end{align*}
\end{lemma}

\begin{corollary}\label{cor:balanced-policy-lower-bound}
We have 
\begin{align*}
    \mu^{\star, h}_{1:h}(x_h, a_h) \ge \frac{1}{X_hA}
\end{align*}
for any $h\in [H]$ and $(x_h,a_h)\in \cX_h \times \cA$.
\end{corollary}
\begin{proof}
Choosing some deterministic policy $\mu$ s.t. $\mu_{1:h}(x_h, a_h) = 1$ in Lemma \ref{lemma:balancing} and noticing each term in
the summation is non-negative, we have
\begin{align*}
    \frac{\mu_{1:h}(x_h, a_h)}{\mu^{\star, h}_{1:h}(x_h, a_h)} \le{X_hA}.
\end{align*}
\end{proof}

\subsection{Algorithms: Balanced EFCE-OMD (in FTRL form and OMD form)}

In this section, we present the algorithms omitted in Section~\ref{sec:sharp-rate}. We begin with the Balanced EFCE-OMD (in FTRL form) as in Algorithm~\ref{algorithm:balanced-efce-ftrl}. This algorithm is actually equivalent to the algorithm as in Eq.~(\ref{equation:phit-F-form}) because of the following lemma, whose proof is similar to Lemma~\ref{lem:evaluate-partition-function-gradient}.

\begin{lemma}
\label{lem:evaluate-balanced-partition-function-gradient}
For any loss matrix $M \in \R_{\ge 0}^{XA \times XA}$, recall that the \emph{balanced EFCE log-partition function} as defined in Eq. (\ref{equation:balanced-f-efce})). %
Let $\lambda = (\lambda_{x_g a_g})_{x_g a_g \in \cX \times \cA} \in \Delta_{XA}$ and $m = ( m_{x_g a_g} )_{x_g a_g \in \cX \times \cA} \in \mc{M}$ be
\begin{align}
& \lambda_{x_ga_g} \propto_{x_ga_g} \exp\Big\{ \frac{1}{XA}  \Big( - \eta\< I - E_{\succeq x_g a_g}, M\> + F^{\star}_{x_ga_g,x_g} \Big) \Big\}\\
& m_{x_ga_g, h}(a_h \vert x_h) \propto_{a_h}\exp\Big\{ \mu^{\star,h}_{g:h}(x_h, a_h) \Big(- \eta M_{x_ha_h, x_ga_g} + \sum_{x_{h+1} \in \cC(x_h, a_h)} F^{\star}_{x_ga_g, x_{h+1}} \Big) \Big\}.
\end{align}
then we have $-\grad F^{\trig}_{\bal}(M)=\phi(\lambda, m)$, where $\phi$ is as defined in Eq. (\ref{equation:phi-lambda-m}). 
\end{lemma}

\begin{algorithm}[t]
   \caption{Balanced EFCE-OMD (OMD form; equivalent FTRL form in Algorithm~\ref{algorithm:balanced-efce-ftrl})}
   \label{algorithm:balanced-efce-omd}
   \begin{algorithmic}[1]
    \REQUIRE Learning rate $\eta$, balanced exploration policy $\{\mu^{\star, h}\}_{h \in [H]}$. 
    \STATE Initialize $\lambda_{x_g a_g}^1 \propto_{x_g a_g} \exp\{ (X_{\succeq x_g}/X) \log A \}$, and $m_{x_ga_g, h}^1(a_h \vert x_h) = 1/A$, for all $(g, x_g, a_g, h, x_h, a_h)$ with $g \le h$. 
    \FOR{$t = 1, 2, \ldots, T$}
    \STATE Compute $\phi^t = \phi(\lambda^t, m^t)$, where $\phi$ is as defined in Eq. (\ref{equation:phi-lambda-m}).
    \STATE Find a $\mu^t$ to be a solution of the fixed point equation $\mu^t = \phi^t \mu^t$. 
    \STATE Play policy $\mu^t$, observe trajectory $(x_h^t, a_h^t, r_h^t)_{h \in [H]}$. 
    \STATE Form vector loss estimator $\wt \ell^{t, x_g a_g} = \{ \wt{\ell}^{t, x_g a_g}_h(x_h, a_h)\}_{x_h a_h}$ for each $(g, x_g a_g)$ as in Eq. (\ref{equation:adaptive-bandit-estimator}). 
\STATE Compute matrix loss estimator $\wt{M}^t = \sum_{g, x_g, a_g} \mu^t_{x_ga_g} \wt{\ell}^{t, x_ga_g} e_{x_g a_g}^\top$. 
\STATE For each $x_g a_g \in \cX \times \cA$, from the reverse order of $x_h$, compute $m^{t}_{x_ga_g, h}(a_h \vert x_h)$ and $F^{\star,t}_{x_ga_g, x_h}$
\begin{align*}
 m^{t+1}_{x_ga_g, h}(a_h \vert x_h) \propto_{a_h}&~ m^{t}_{x_ga_g, h}(a_h \vert x_h) \exp\Big\{ \mu^{\star,h}_{g:h}(x_h, a_h) \\
 &~\Big(- \eta \wt{M}^t_{x_ha_h, x_ga_g} + \sum_{x_{h+1} \in \cC(x_h, a_h)} \wt{F}^{\star,t}_{x_ga_g, x_{h+1}} \Big) \Big\}, \\
\wt{F}_{x_g a_g, x_h}^{\star,t} \defeq&~  \frac{1}{\mu^{\star,h}_{g:h}(x_h, a_h)}\log \sum_{a_h \in \cA} m^{t}_{x_ga_g, h}(a_h \vert x_h)\exp\Big\{ \mu^{\star,h}_{g:h}(x_h, a_h) \\
&~ \times \big[ - \wt{M}^t_{x_ha_h, x_g a_g} +  \sum_{x_{h+1} \in \cC(x_ha_h)} \wt{F}_{x_g a_g, x_{h+1}}^{\star,t}\big] \Big\}.
\end{align*}
\STATE Compute $\lambda_{x_g a_g}^{t+1}$ as
\begin{align}
\lambda_{x_ga_g}^{t+1} \propto_{x_ga_g} \lambda_{x_ga_g}^{t}  \exp\Big\{ \frac{1}{XA}  \Big( - \eta\< I - E_{\succeq x_g a_g}, \wt{M}^t\> + \wt{F}^{\star, t}_{x_ga_g,x_g} \Big) \Big\} .
\end{align}
    \ENDFOR
 \end{algorithmic}
\end{algorithm}

We also present an efficient update of $(\lambda^{t+1}, m^{t+1} )$ from $(\lambda^{t}, m^{t} )$, which gives the OMD form of the Balanced EFCE-OMD algorithm as in Algorithm~\ref{algorithm:balanced-efce-omd}. Notice the initialization of Balanced EFCE-OMD is different from EFCE-OMD (Algorithm~\ref{algorithm:efce-omd}) due to the presence of the balanced exploration policy.
Algorithm~\ref{algorithm:balanced-efce-ftrl} and Algorithm~\ref{algorithm:balanced-efce-omd} are indeed equivalent due to the following lemma, whose proof is similar to that of Lemma~\ref{lem:efce-omd}.
\begin{lemma}
  \label{lem:balanced-efce-omd}
  Given the same sequence of $M^t$, 
  Algorithm~\ref{algorithm:balanced-efce-ftrl} and Algorithm~\ref{algorithm:balanced-efce-omd} outputs the same  $\lambda^t$ and $m^t$ and thus the same $\phi^t$.
\end{lemma}

\subsection{Equivalence to FTRL and OMD}
\label{sec:balanced-equivalence}

Similar as Section~\ref{sec:phi-hedge-dilated-entropy}, we show that the Balanced EFCE-OMD algorithm (Algorithm~\ref{algorithm:balanced-efce-ftrl}) is equivalent to FTRL with the balanced trigger dilated entropy, and OMD with the balanced dilated KL divergence, both over the $(\lambda, m)$ parametrization.

We first introduce Balanced dilated entropy and balanced dilated KL divergence, and their trigger versions as below.

\begin{definition}[Balanced dilated entropy and balanced dilated KL divergence]
The balanced dilated entropy $H^{\bal}_{x_g}$ rooted at $x_g$ of subtree policy $\mu^{x_g}\in\Pi^{x_g}$ is defined as
\begin{align}
H^{\bal}_{x_g}(\mu^{x_g}) \defeq \sum_{h=g}^H \sum_{(x_h, a_h) \succeq x_g} \frac{\mu^{x_g}_{g:h}(x_h, a_h)}{\mu^{\star, h}_{g:h}(x_h, a_h)} \log \mu^{x_g}_h(a_h | x_h).
\end{align}
The balanced dilated KL divergence $D^{\bal}_{x_g}$ rooted at $x_g$ between two subtree policies $\mu^{x_g}, \nu^{x_g}\in\Pi^{x_g}$ is defined as
\begin{align}
    D^{\bal}_{x_g}(\mu^{x_g} \| \nu^{x_g}) \defeq \sum_{h=g}^H \sum_{(x_h, a_h) \succeq x_g} \frac{\mu^{x_g}_{g:h}(x_h, a_h)}{\mu^{\star, h}_{g:h}(x_h, a_h)} \log \frac{\mu^{x_g}_h(a_h | x_h)}{\nu^{x_g}_h(a_h | x_h)}.
\end{align}
\end{definition}

\begin{definition}[Balanced trigger dilated entropy and balanced trigger dilated KL divergence]
The balanced trigger dilated entropy function on $(\lambda, m) \in \Delta_{XA} \times \mc{M}$ is defined as
\begin{align}
\textstyle H_{\bal}^{\trig}(\lambda, m) =&~ XA \cdot H(\lambda) + \sum_{g, x_g, a_g} \lambda_{x_g a_g} H^{\bal}_{x_g}(m_{x_g a_g}). 
\end{align}
The balanced trigger dilated KL divergence function on $(\lambda, m), (\lambda', m') \in \Delta_{XA} \times \mc{M}$ is defined as
\begin{align}
\textstyle D_{\bal}^{\trig}( \lambda, m\| \lambda', m') =&~ XA \cdot D_{\kl}(\lambda \| \lambda') + \sum_{g, x_g, a_g} \lambda_{x_g a_g} D^{\bal}_{x_g}(m_{x_g a_g} \| m_{x_g a_g}').
\end{align}
\end{definition}

The following lemma shows that the Balanced EFCE-OMD (Algorithm \ref{algorithm:balanced-efce-ftrl} and \ref{algorithm:balanced-efce-omd}) are essentially FTRL with the balanced trigger dilated entropy, and OMD with the balanced tirgger dilated KL divergence. The proof of this lemma is similar to that of Lemma~\ref{lem:equivalence-algorithm}.

\begin{lemma}[Equivalent of Balanced EFCE-OMD to OMD/FTRL on $(\lambda, m)$]\label{lem:equivalent-FTRL-OMD-balanced}
For any sequence of loss functions $\{\wt M^t\}_{t\ge 1}$, the algorithm as in Eq.~(\ref{equation:phit-F-form}) is equivalent to (i.e. satisfy) the following FTRL update on $H_{\bal}^{\trig}$ and OMD update on $D_{\bal}^{\trig}$:
\begin{align}
& \textstyle (\lambda^{t+1}, m^{t+1}) =~ \argmin_{\lambda, m} \Big[\eta \< \phi(\lambda, m), \sum_{s=1}^t \wt M^s\> + H_{\bal}^{\trig}(\lambda, m) \Big],  \\
& \textstyle (\lambda^{t+1}, m^{t+1}) =~ \argmin_{\lambda, m}\Big[ \eta \< \phi(\lambda, m), \wt M^t\> + D_{\bal}^{\trig}(\lambda, m \| \lambda^t, m^t) \Big],
\end{align}
with $\phi^{t+1} = \phi(\lambda^{t+1}, m^{t+1})$. 
\end{lemma}
\section{Proof of Theorem~\ref{thm:balanced-efce-omd-bandit}}
\label{app:proof-balanced-efce-omd-bandit}

Here we restate the theorem for convenience. 

\begin{theorem}[Sample complexity under bandit feedback]
Run Balanced EFCE-OMD (Algorithm \ref{algorithm:balanced-efce-ftrl}) with  $\eta = \sqrt{XA\iota  / H^4 T}$ and $\gamma = 2 \sqrt{XA \iota / H^2 T}$. Then with probability at least $1- \delta$, we have the following extensive-form trigger regret bound,
\begin{align*}
    \efceregret(T) \le 200 \sqrt{XA H^4 T \iota},
\end{align*}
where $\iota = \log (10HXA/\delta)$ is a log factor.
\end{theorem}

\begin{proof}
By the fixed point property of our algorithm, we have the regret decomposition
\begin{align*}
    & \quad \efceregret(T) \\
    & = \sup_{\phi^\star\in\Phi^{\efce}} \sum_{t=1}^T \langle \mu^t - \phi^\star\mu^t, \ell^t\rangle  = \sup_{\phi^\star\in\Phi^{\efce}} \sum_{t=1}^T \langle \phi^t\mu^t - \phi^\star\mu^t, \ell^t\rangle  = \sup_{\phi^\star\in\Phi^{\efce}} \sum_{t=1}^T \langle \phi^t - \phi^\star, \ell^t(\mu^t)^\top\rangle  \\
    & \le \underbrace{\sup_{\phi^\star\in\Phi^{\efce}} \sum_{t=1}^T \langle \phi^t - \phi^\star, \wt{M}^t \rangle }_{\wt{\rm REGRET}^{\efce}(T)} + 
    \underbrace{ \sum_{t=1}^T \langle \phi^t, \ell^t(\mu^t)^\top - \wt{M}^t \rangle }_{{\rm BIAS}_1} +
    \underbrace{\sup_{\phi^\star\in\Phi^{\efce}} \sum_{t=1}^T \langle \phi^\star, \wt{M}^t - \ell^t(\mu^t)^\top \rangle }_{{\rm BIAS}_2}.
\end{align*}

We bound the term $\wt{\rm REGRET}^{\efce}(T)$, ${\rm BIAS}_1$, and ${\rm BIAS}_2$ in the following lemmas, whose proofs are presented in Section \ref{app:proof-tilde-regret-bound} and \ref{app:proof-bias}.

\begin{lemma}[Bound on $\wt{\rm REGRET}^{\efce}(T) $]\label{lem:regret_bound-balanced}
Assume that $\gamma \ge 2\eta H$. We have with probability at least $1-\delta/3$ that 
\begin{align*}
    \wt{\rm REGRET}^{\efce}(T) 
    \le ~ \frac{XA \log(XA^2)}{\eta} + 22 \eta H^4 T + \frac{38 \eta H^3 XA \iota}{\gamma}, %
\end{align*}
where $\iota = \log (10HXA/\delta)$ is a log factor. 
\end{lemma}

\begin{lemma}[Bound on ${\rm BIAS}_1$]\label{lem:bias1-balanced}
We have with probability at least $1-\delta/3$ that
\begin{align*}
    {\rm BIAS}_1 \le 2\gamma H^2T + 2H\sqrt{T\iota},
\end{align*}
where $\iota = \log (3/\delta)$ is a log factor. 
\end{lemma}

\begin{lemma}[Bound on ${\rm BIAS}_2$]\label{lem:bias2-balanced}
We have with probability at least $1-\delta/3$ that
\begin{align*}
    {\rm BIAS}_2 \le \frac{XA\iota }{ \gamma},
\end{align*}
where $\iota = \log (3 XA/\delta)$ is a log factor. 
\end{lemma}

By these three lemmas, whenever $\gamma \ge 2 \eta  H$, we have with probability at least $1-\delta$ that (for $\iota\defeq \log(10HXA/\delta)$)
\[
\begin{aligned}
\efceregret(T) \le&~ \frac{XA \log(XA^2)}{\eta} + 22 \eta H^4 T + \frac{38 \eta H^3 XA \iota}{\gamma} + 2 \gamma H^2 T + 2 H \sqrt{T\iota} + \frac{XA\iota }{ \gamma}. 
\end{aligned}
\]
Taking $\eta = \sqrt{XA \iota / H^4 T}$ and $\gamma = 2 \sqrt{XA \iota / H^2 T}$, we get 
\[
\efceregret(T) \le 100 \Big[\sqrt{XAH^4 T \iota} + H^2 XA \iota\Big]. 
\]

Notice that we always have the “trivial” bound $\efceregret(T) \le HT$.
For $T \ge XA\iota$,  we have $H^2XA\iota \le \sqrt{XAH^4T\iota}$, which gives $\efceregret \le 200 \sqrt{H^4XAT\iota}$; For $T \le  XA\iota $,  we have $HT \le \sqrt{XAH^4T\iota}$, which gives $\efceregret \le HT \le  \sqrt{H^4XAT\iota}$. Therefore, we always have
\begin{align*}
    \efceregret \le 200\sqrt{H^4XAT\iota}.
\end{align*}
This gives the desired bound. 
\end{proof}

The rest of this section is organized as follows. We introduce some notation in Section \ref{app:notations-and-preparations}. In Section \ref{app:proof-tilde-regret-bound}, we bound the regret term $\wt{\rm REGRET}^{\efce}(T) $. In Section \ref{app:proof-bias}, we bound the two bias terms ${\rm BIAS}_1$ and ${\rm BIAS}_2$.

\subsection{Some preparations}\label{app:notations-and-preparations}
Note that $m^t_{x_ga_g} \in \Pi^{x_g}$ is a subtree policy rooted at $x_g$, we define for any $(\wt{x}_h, \wt{a}_h) \succ x_g$
\[
m^t_{x_ga_g,g:h}(\wt{x}_h, \wt{a}_h) \defeq \prod_{h' = g}^h m^t_{x_ga_g, h'}(\wt{a}_{h'}|\wt{x}_{h'}),
\]
where $(\wt{x}_g,\wt{a}_g,\wt{x}_{g+1},\wt{a}_{g+1},\cdots, \wt{x}_{h-1},\wt{a}_{h-1})$ is the unique history leading to $(\wt{x}_h, \wt{a}_h)$, and $\wt{x}_g= x_g$.

Note that we have $\phi^t = \sum_{g, x_g, a_g} \lambda_{x_ga_g}^t (I - E_{\succeq x_g a_g} + m_{x_ga_g}^t e_{x_g a_g}^\top)$ and $\phi^t\mu^t=\mu^t$. These two equations give
\begin{align}
    \sum_{g,x_g,a_g} \lambda_{x_ga_g}^t (I - E_{\succeq x_g a_g})\mu^t = \sum_{g,x_g,a_g} \lambda_{x_ga_g}^t \mu^t_{x_ga_g} m_{x_ga_g}^t \in \R^{XA}.
\end{align}
As a consequence, for any $x_ga_g$, we have
\begin{align}
\label{equation:lambdamum}
    \lambda_{x_ga_g}^t \mu^t_{x_ga_g} m_{x_ga_g}^t \le \sum_{g,x_g,a_g} \lambda_{x_ga_g}^t\mu^t = \mu^t. 
\end{align}
Here $\lambda_{x_ga_g}^t \in \Delta_{XA}, \mu^t_{x_ga_g} = \mu_{1:g}^t(x_g, a_g)$ are two scalars, $m_{x_ga_g}^t \in \Pi^{x_g}$ and $\mu^t \in \Pi$ are two vectors of length $XA$, and the $\le$ above is understood in an entrywise sense.

We also define (recall that $\{ p_h^t \}_{h \in \{ 0\} \cup [H], t \ge 1}$ are the adversarial transition probabilities)
\begin{align}
\label{equation:ptxh}
p^t(x_h) \defeq p_0^t(x_1) \prod_{h' = 1}^{h-1} p_{h'}^t(x_{h'+1}| x_{h'}, a_{h'}).
\end{align}
Note that $p^t(x_h)\in [0,1]$. Furthermore, for any policy $\mu\in\Pi$ and any $(h, t)$, we have
\begin{equation}\label{eqn:mu_pt_1}
\sum_{x_h, a_h}\mu_{1:h}(x_h, a_h) p^t(x_h)=1,
\end{equation}
as the left-hand side is the probability of visiting \emph{some} $(x_h, a_h)$ in episode $t$ using policy $\mu$.

\subsection{Proof of Lemma \ref{lem:regret_bound-balanced}}\label{app:proof-tilde-regret-bound}
Recall that $\wt{\rm REGRET}^{\efce}(T) $ is defined as
\begin{align*}
     \wt{\rm REGRET}^{\efce}(T)  \defeq \sup_{\phi^\star\in\Phi^{\efce}} \sum_{t=1}^T \langle \phi^t - \phi^\star, \wt{M}^t\rangle  .
\end{align*}

First, we claim that
\[
\begin{aligned}
\sup_{\phi^\star\in\Phi^{\efce}} \langle  - \phi^\star, M\rangle  = \sup_{\phi^\star\in\Phi_0^{\efce}} \langle  - \phi^\star, M\rangle  \le \frac{1}{\eta} F^{\efce}_{\bal}(M). 
\end{aligned}
\]
for any $M \in \R^{XA \times XA}$. Indeed, the equality follows from $\Phi^\efce = {\rm conv}\set{ \Phi^{\trig}_{0} }$. The inequality is due to the following argument: for any fixed $M$, the maximizer $\phi_{x_g a_g \to m_{x_g a_g}} \in \Phi_0^\efce$ specifies a trigger sequence $x_g a_g$ and a deterministic subtree policy $m_{x_ga_g}\in\Pi^{x_g}$ starting from $x_g$. Replacing all the sums by this realization in the formula of $F^{\efce}_\bal$ (c.f. Eq. \eqref{equation:balanced-f-efce}) and $F^{\star}_{x_ga_g,x_h}$ (c.f. Eq. \eqref{equation:balanced-f-efce-xgag}) exactly gives $F^{\trig}(M) \ge \eta\langle -\phi_{x_g a_g \to m_{x_g a_g}}, M\rangle  = \eta \sup_{\phi^\star\in\Phi_0^{\efce}} \langle  - \phi^\star, M\rangle $. This proves the claim. 

This claim gives 
\begin{equation}\label{eqn:regret-tr-tilde-bound}
\begin{aligned}
\wt{\rm REGRET}^{\efce}(T) =&~\sup_{\phi^\star\in\Phi^{\efce}} \sum_{t=1}^T \langle \phi^t - \phi^\star, \wt{M}^t\rangle  
= \sup_{\phi^\star\in\Phi^{\efce}}  \langle  - \phi^\star, \sum_{t=1}^T \wt{M}^t\rangle  + \sum_{t=1}^T \langle \phi^t, \wt{M}^t\rangle \\
\le&~ \frac{1}{\eta} F^{\efce}_{\bal}\Big(\sum_{t=1}^T \wt{M}^t \Big) + \sum_{t=1}^T \langle \phi^t, \wt{M}^t\rangle   = \frac{1}{\eta} F^{\efce}_{\bal}(0) + \sum_{t = 1}^T D^t,
\end{aligned}
\end{equation}
where $D^t$ is given by
\begin{equation}
D^t = \frac{1}{\eta} F^{\efce}_{\bal}\Big( \eta \sum_{s = 1}^{t} \wt{M}^s \Big) - \frac{1}{\eta} F^{\efce}_{\bal}\Big( \eta \sum_{s = 1}^{t-1} \wt{M}^s \Big) + \langle  \phi^{t}, \wt{M}^{t}\rangle . 
\end{equation}

The following lemma gives bound on the initial entropy $F^{\efce}_{\bal}(0)$ with proof in Section \ref{sec:proof-initial-balanced}.
\begin{lemma}[Bound on initial entropy]\label{lem:initial-balanced}
We have 
\begin{equation}\label{eqn:F-efce-0}
F^{\efce}_{\bal}(0) = XA \log \sum_{g,x_g,a_g} \exp\{ [X_{\succeq x_g} A \log A] / XA \} \le XA \log(XA^2). 
\end{equation}
\end{lemma}

The following lemma gives a reformulation of the stability term $D^t$ with proof in Section \ref{sec:proof-reformulation-of-stability}. 
\begin{lemma}[Reformulation of stability term via incremental update]
\label{lem:reformulation-of-stability}
We have
\begin{align}
D^t  =&~  \wb{F}^{t}/\eta + \langle  \phi^{t}, \wt{M}^{t}\rangle , 
\end{align}
where we have
\begin{align}
   \overline F^{t} =&~ XA \log \sum_{g,x_g a_g} \lambda^t_{x_ga_g} \exp\Big\{ \frac{1}{XA} \big[ - \eta \langle  I - E_{\succeq x_g a_g}, \wt{M}^t \rangle  + F^{\star, t}_{x_ga_g,x_g} \big] \Big\}, \label{eqn:F-EFCE-relation-Z-t+1}\\
   F^{\star, t}_{x_ga_g, x_h} =&~ \frac{1}{\mu^{\star,h}_{g:h}(x_h, a_h)} \log  \sum_{a_h\in \cA} m^{t}_{x_ga_g, h}(a_h|x_h) \exp\Big\{ \mu^{\star,h}_{g:h}(x_h, a_h) \nonumber \\
   & \times \big[ -\eta \underbrace{\mu^t_{x_ga_g}\wt{\ell}^{t,x_ga_g}_h(x_h, a_h)}_{=\wt{M}^t_{x_ha_h, x_ga_g}}+ \sum_{x_{h+1} \in \cC(x_h, a_h)} F^{\star,t}_{x_ga_g, x_{h+1}} \big]\Big\},  ~~~ \forall (x_h, a_h) \succeq x_g, \label{eqn:logZ-EFCE-recursion-t+1} 
\end{align}
and (note that $F^\star_{x_g a_g, x_g}(\mathbf{0})$ is as defined in Eq. (\ref{equation:balanced-f-efce-xgag}) by plugging in $M = \mathbf 0$)
\begin{align}
& \lambda_{x_ga_g}^{t} \propto_{x_ga_g} \exp\Big\{ \frac{1}{XA}F^{\star}_{x_ga_g, x_g}(\mathbf{0}) +  \frac{1}{XA} \sum_{s = 1}^{t-1} \Big( - \eta\langle  I - E_{\succeq x_g a_g}, \wt{M}^s\rangle  + F^{\star, s}_{x_ga_g,x_g} \Big) \Big\}, \\
& m^{t}_{x_ga_g, h}(a_h \vert x_h) \propto_{a_h} \exp\Big\{ \mu^{\star,h}_{g:h}(x_h, a_h) \sum_{s = 1}^{t-1} \Big(-   \eta \underbrace{\mu^s_{x_ga_g} \wt{\ell}^{s,x_ga_g}_{h}(x_h, a_h)}_{=\wt{M}^s_{x_ha_h, x_ga_g}} + \sum_{x_{h+1} \in \cC(x_h, a_h)} F^{\star,s}_{x_ga_g, x_{h+1}} \Big) \Big\}. \label{eqn:m^t} 
\end{align}
\end{lemma}

To upper bound $D^t$, note that by Lemma~\ref{lem:reformulation-of-stability} we have
\begin{align*}
&~D^t = \wb{F}^t/\eta + \langle \phi^t, \wt{M}^t\rangle    \\
    & =  \langle  \phi^t, \wt{M}^t\rangle  + \frac{XA}{\eta} \log \sum_{g,x_ga_g} \lambda_{x_ga_g}^t \exp\Big\{ \frac{1}{XA} \big[ - \eta \langle  I - E_{\succeq x_ga_g} + m_{x_ga_g}^t e_{x_ga_g}^\top, \wt{M}^t\rangle  + \Delta_{x_ga_g}^t \big] \Big\},
\end{align*}
where $\Delta_{x_ga_g}^t$ is given by
\begin{align}
\label{equation:deltaxgag}
\begin{aligned}
\Delta_{x_ga_g}^t \defeq&~ F^{\star, t}_{x_ga_g, x_g} + \eta\langle  m_{x_ga_g}^t \mu^t_{x_ga_g}, \wt{\ell}^{t,x_ga_g} \rangle  \\
     =&~ F^{\star, t}_{x_ga_g, x_g} + \eta \langle  m_{x_ga_g}^t e_{x_ga_g}^\top, \wt{M}^t \rangle .
\end{aligned}
\end{align}
Note that $- \eta\langle  m_{x_ga_g}^t \mu^t_{x_ga_g}, \wt{\ell}^{t,x_ga_g} \rangle$ is the linear term in the Taylor expansion of $F^{\star, t}_{x_ga_g, x_h}$ over variable $\wt{\ell}^t$ at $0$, so $\Delta_{x_ga_g}^t$ can be understood as the nonlinear part within $F^{\star, t}_{x_ga_g, x_g}$. By convexity of $F^{\star, t}_{x_ga_g, x_h}$ as a function of $\wt \ell^t$, we have $\Delta_{x_ga_g}^t \ge 0$. 
Furthermore, we have the following almost sure upper bound of $\sup_{g, x_ga_g}\Delta_{x_ga_g}^t$ with proof in Section \ref{sec:proof-sup-delta}. 
\begin{lemma}[Bound on $\sup_{g, x_ga_g}\Delta_{x_ga_g}^t$]
\label{lemma:sup-delta}
We have for all $t\in[T]$ that, almost surely,
\begin{align*}
    \frac{1}{XA} \sup_{g,x_g a_g} \Delta_{x_g a_g}^t \le \frac{2\eta^2}{\gamma^2} H^2.
\end{align*}
\end{lemma}

Given this lemma, we further assume that $\gamma \ge 2 H \eta$ so that $\frac{1}{XA} \sup_{g,x_g a_g} \Delta_{x_g a_g}^t \le 1$. Now we use elementary inequalities $\log(1 + x) \le x$ and 
\begin{align*}
    e^{-x+c} \le 1 - (x-c) + \frac{e}{2} (x-c)^2 \le 1 - (x-c) + e(x^2 + c^2),~~~~~~ \forall x \ge 0, c \le 1,
\end{align*}
and (taking $c = \Delta_{x_g a_g}^t$ for each $(g, x_g a_g)$ below) we get
\begin{align*}
&~\log \sum_{g,x_ga_g} \lambda_{x_ga_g}^t \exp\Big\{ \frac{1}{XA} \big[ - \eta \langle  I - E_{\succeq x_ga_g} + m_{x_ga_g}^t e_{x_ga_g}^\top, \wt{M}^t\rangle  + \Delta_{x_ga_g}^t \big] \Big\}\\
\le& \sum_{g,x_ga_g} \lambda_{x_ga_g}^t \exp\Big\{ \frac{1}{XA} \big[ - \eta \langle  I - E_{\succeq x_ga_g} + m_{x_ga_g}^t e_{x_ga_g}^\top, \wt{M}^t\rangle  + \Delta_{x_ga_g}^t \big] \Big\} - 1\\
\le&~ \Big\{ \sum_{g,x_ga_g} \lambda_{x_ga_g}^t \Big( 1 + \frac{1}{XA} \big[ - \eta \langle  I - E_{\succeq x_ga_g} + m_{x_ga_g}^t e_{x_ga_g}^\top, \wt{M}^t\rangle  + \Delta_{x_ga_g}^t \big] \\
&~+ \frac{e}{X^2A^2} \big[ \eta^2 \langle  I - E_{\succeq x_ga_g} + m_{x_ga_g}^t e_{x_ga_g}^\top, \wt{M}^t\rangle^2  + (\Delta_{x_ga_g}^t)^2 \big] \Big\} - 1 \\
=&~ - \frac{\eta}{XA} \langle \phi^t, \wt{M}^t\rangle  + \frac{1}{XA}\sum_{g,x_ga_g} \lambda_{x_ga_g}^t \Delta_{x_ga_g}^t \\
&~ + \frac{e}{X^2A^2} \sum_{g,x_ga_g} \lambda_{x_ga_g}^t  \Big( \eta^2 \langle  I - E_{\succeq x_ga_g} + m_{x_ga_g}^t e_{x_ga_g}^\top, \wt{M}^t\rangle^2  + (\Delta_{x_ga_g}^t)^2 \Big). \\
\end{align*}
This gives that
\begin{equation}\label{eqn:Dt-bound-in-proof}
\begin{aligned}
D^t  \le&~ \frac{1}{\eta} \sum_{g,x_ga_g} \lambda_{x_ga_g}^t \Delta_{x_ga_g}^t + \frac{e}{\eta XA} \sum_{g,x_ga_g} \lambda_{x_ga_g}^t \paren{ \eta^2\langle  I - E_{\succeq x_ga_g} + m_{x_ga_g}^t e_{x_ga_g}^\top, \wt{M}^t\rangle ^2 + (\Delta_{x_ga_g}^t)^2 } \\
     \stackrel{(i)}{\le}&~ \frac{1}{\eta} \sum_{g,x_ga_g} \lambda_{x_ga_g}^t\Delta_{x_ga_g}^t \paren{1 + \frac{e}{XA}\sup_{g,x_ga_g}\Delta_{x_ga_g}^t} + \frac{e\eta}{XA}\sum_{g,x_ga_g}\lambda_{x_ga_g}^t \langle  I - E_{\succeq x_ga_g} + m_{x_ga_g}^t e_{x_ga_g}^\top, \wt{M}^t\rangle ^2 \\
    \stackrel{(ii)}{\le}&~ \underbrace{\frac{4}{\eta} \sum_{g,x_ga_g} \lambda_{x_ga_g}^t \Delta_{x_ga_g}^t}_{{\rm I}_t} + \underbrace{\frac{e\eta}{XA}\sum_{g,x_ga_g}\lambda_{x_ga_g}^t \langle  I - E_{\succeq x_ga_g} + m_{x_ga_g}^t e_{x_ga_g}^\top, \wt{M}^t\rangle ^2}_{{\rm II}_t}.
\end{aligned}
\end{equation}
Above, (i) used $\Delta_{x_ga_g}^t\ge 0$, and (ii) used Lemma \ref{lemma:sup-delta} and $\gamma \ge 2\eta H$.

Next, we use the following lemmas to bound $\sum_{t=1}^T {\rm I}_t$ and $\sum_{t=1}^T {\rm II}_t$, with proofs in Section \ref{sec:proof-bound-on-I-balanced} and \ref{sec:proof-bound-on-II-balanced}. 

\begin{lemma}[Bound on $\sum_{t=1}^T {\rm I}_t$]
\label{lem:bound-on-I-balanced}
With probability at least $1-\delta/10$, we have
\begin{align*}
    \sum_{t=1}^T {\rm I}_t 
    \le 16\eta H^3 T+\frac{32 \eta H^3XA\iota }{\gamma},
\end{align*}
where $\iota:=\log(10 H/\delta)$ is a log factor.
\end{lemma}

\begin{lemma}[Bound on $\sum_{t=1}^T {\rm II}_t$]
\label{lem:bound-on-II-balanced}
With probability at least $1-\delta/10$, we have
\begin{align*}
    \sum_{t=1}^T {\rm II}_t \le 6\eta HT +\frac{ 6\eta H X A \iota}{\gamma},
\end{align*}
where $\iota:=\log(10 XA/\delta)$ is a log factor.
\end{lemma}
Combining Eq. (\ref{eqn:regret-tr-tilde-bound}), Lemma \ref{lem:initial-balanced} and \ref{lem:reformulation-of-stability}, Eq. (\ref{eqn:Dt-bound-in-proof}), and Lemma \ref{lem:bound-on-I-balanced} and \ref{lem:bound-on-II-balanced}, we have
\begin{align*}
    \wt{\rm REGRET}^{\efce}(T) \le \frac{XA \log(XA^2)}{\eta} + 22 \eta H^4 T + \frac{38 \eta H^3 XA \iota}{\gamma}
\end{align*}
with probability at least $1-\delta/3$, where $\iota :=\log(10 X A H /\delta)$ is a log factor. This completes the proof of Lemma \ref{lem:regret_bound-balanced}. 
\qed

\subsubsection{Proof of Lemma \ref{lem:initial-balanced}}
\begin{proof}[Proof of Lemma \ref{lem:initial-balanced}]\label{sec:proof-initial-balanced}

By the definition of {balanced EFCE log-partition function} (see \eqref{equation:balanced-f-efce} and \eqref{equation:balanced-f-efce-xgag}), we have
\begin{align*}
F^{\trig}_{\bal}(\mathbf{0}) = XA \log \sum_{g, x_g, a_g} \exp\Big\{ \frac{1}{XA} \big[ F_{x_g a_g, x_g}^\star(\mathbf{0}) \big] \Big\},
\end{align*}
where for any $x_h \succeq x_g$,
\begin{align}\label{equation:balanced-f0-efce-xgag}
    F_{x_g a_g, x_h}^\star(\mathbf{0}) =&~ \frac{1}{\mu^{\star, h}_{g:h}(x_h, a_h)}\log \sum_{a_h} \exp\Big\{ \mu^{\star, h}_{g:h}(x_h, a_h) \big[  \sum_{x_{h+1} \in \cC(x_h,a_h)} F_{x_g a_g, x_{h+1}}^\star(\mathbf{0})\big] \Big\}.
\end{align}
So we only need to prove that $F_{x_g a_g, x_g}^\star(\mathbf{0}) = X_{\succeq x_g}A\log A$. In fact, we can use backward induction to prove the following: for any $x_g \in \cX_g$ and $x_h \in \mc{C}_{h}(x_g)$, we have
\begin{align}\label{eqn:Fstar0-general-result}
    F_{x_g a_g, x_h}^\star(\mathbf{0})= \sum_{h'=h}^H \frac{\abs{\mc{C}_{h'}(x_h)}}{\mu^{\star: h'}_{g:h-1}(x_{h-1},a_{h-1})}A\log A,
\end{align}
(with convention $\mu_{g:g-1}^{\star,h'} = 1$) where $x_{h-1},a_{h-1}$ is uniquely determined as $x_g \prec (x_{h-1},a_{h-1}) \prec x_h$. It is easy to see that, choosing $x_h = x_g$ in \eqref{eqn:Fstar0-general-result} gives $F_{x_g a_g, x_g}^\star(\mathbf{0}) = X_{\succeq x_g}A\log A$. 

Next we prove \eqref{eqn:Fstar0-general-result}. We use backward induction on $h$. When $h = H$, from \eqref{equation:balanced-f0-efce-xgag}, for any $x_H\in\mc{C}_H(x_g) $, we have
$$
F_{x_g a_g, x_H}^\star(\mathbf{0}) = \frac{1}{\mu_{g:H}^{\star,H}(x_H, a_H)}\log A = \frac{1}{\mu_{g:H-1}^{\star, H}(x_{H-1}, a_{H-1})} A\log A.
$$
This proves the base case. 

Now suppose \eqref{eqn:Fstar0-general-result} is true for $h+1$ where $h \in [g, H-1]$. By the recursive formula and the inductive  hypothesis, for any $x_{h}\in \mc{C}_h(x_g)$, we have
\begin{align*}
    &~F_{x_g a_g, x_h}^\star(\mathbf{0}) \\
    =&~ \frac{1}{\mu^{\star, h}_{g:h}(x_h, a_h)}\log \sum_{a_h} \exp\Big\{ \mu^{\star, h}_{g:h}(x_h, a_h) \big[  \sum_{x_{h+1} \in \cC(x_h,a_h)} F_{x_g a_g, x_{h+1}}^\star(\mathbf{0})\big] \Big\}\\
    =&~ \frac{1}{\mu^{\star, h}_{g:h}(x_h, a_h)}\log \sum_{a_h} \exp\Big\{ \mu^{\star, h}_{g:h}(x_h, a_h) \big[  \sum_{x_{h+1} \in \cC(x_h,a_h)} \sum_{h'=h+1}^H \frac{|\mc{C}_{h'}(x_{h+1})|}{\mu^{\star, h'}_{g:h}(x_{h},a_{h})}A\log A\big] \Big\}.
\end{align*}
Then by the definition of the balanced policies \eqref{def:balanced-exploration-policy}, we have
\begin{align*}
    &~\sum_{x_{h+1} \in \cC(x_h,a_h)}\frac{|\mc{C}_{h'}(x_{h+1})|}{\mu^{\star, h'}_{g:h}(x_{h},a_{h})}\\
    =&~\frac{\sum_{x_{h+1} \in \cC(x_h,a_h)}|\mc{C}_{h'}(x_{h+1})|}{\mu^{\star, h'}_{g:h-1}(x_{h-1},a_{h-1})}\cdot \frac{|\cC_{h'}(x_h)|}{|\cC_{h'}(x_h, a_h)|} = \frac{|\mc{C}_{h'}(x_{h})|}{\mu^{\star, h'}_{g:h-1}(x_{h-1},a_{h-1})},
\end{align*}
which is also independent of $a_h$.
So we have
\begin{align*}
    &~F_{x_g a_g, x_h}^\star(\mathbf{0}) \\
    =&~ \frac{1}{\mu^{\star, h}_{g:h}(x_h, a_h)}\log\set{ A \cdot \exp\Big\{ \mu^{\star, h}_{g:h}(x_h, a_h) \big[   \sum_{h'=h+1}^H \frac{|\mc{C}_{h'}(x_{h})|}{\mu^{\star, h'}_{g:h-1}(x_{h-1},a_{h-1})} A\log A\big] \Big\}}\\
    =&~  \frac{\log A}{\mu^{\star, h}_{g:h}(x_h, a_h)}+    \sum_{h'=h+1}^H \frac{|\mc{C}_{h'}(x_{h})|}{\mu^{\star, h'}_{g:h-1}(x_{h-1},a_{h-1})}A\log A\\
    =&~  \sum_{h'=h}^H \frac{\abs{\mc{C}_{h'}(x_h)}}{\mu^{\star: h'}_{g:h-1}(x_{h-1},a_{h-1})}A\log A.
\end{align*}
This proves \eqref{eqn:Fstar0-general-result}, and thus we proved the first equation in Eq. (\ref{eqn:F-efce-0}). The inequality in Eq. (\ref{eqn:F-efce-0}) is direct since $\exp\{ [X_{\succeq x_g} A \log A] / XA \} \le A$. This proves the lemma. 
\end{proof}

\subsubsection{Proof of Lemma \ref{lem:reformulation-of-stability}}\label{sec:proof-reformulation-of-stability}

\begin{proof}[Proof of Lemma \ref{lem:reformulation-of-stability}]

We only need to verify that 
\begin{equation}\label{eqn:def-wbF-t-in-proof}
\wb{F}^t \defeq F^\efce_\bal\Big(\eta \sum_{s=1}^t \wt{M}^s\Big) - F^\efce_\bal\Big(\eta \sum_{s=1}^{t-1} \wt{M}^s\Big)
\end{equation}
can be computed via recursive formulas \eqref{eqn:F-EFCE-relation-Z-t+1}-\eqref{eqn:m^t}.

For any $x_ga_g$, $x_h\succeq x_g$, define 
\[
G_{x_g a_g, x_h}^{\star, t} \defeq F^\star_{x_ga_g, x_h}\paren{\eta \sum_{s=1}^{t} \wt{M}^s} - F^\star_{x_ga_g,x_h}\paren{\eta \sum_{s=1}^{t-1} \wt{M}^s}.
\]
By the definition of $F^\star_{x_ga_g,x_h}$ as in Eq. (\ref{equation:balanced-f-efce-xgag}), we have
\begin{align*}
    &G_{x_g a_g, x_h}^{\star, t}
    =  \frac{1}{\mu_{g:h}^\star(x_h, a_h)} 
    \\ &\times \log \frac{\sum_{a_h \in \cA} \exp\Big\{  \mu^{\star,h}_{g:h}(x_h, a_h) \times \set{ F^\star_{x_g a_g, x_h}(\mathbf{0}) +  \sum_{s=1}^{t}\big[ - \eta \wt{M}^s_{x_ha_h, x_g a_g} +  \sum_{x_{h+1} \in \cC(x_ha_h)} G_{x_g a_g, x_{h+1}}^{\star,s}\big] } \Big\}}{\sum_{a_h \in \cA} \exp\Big\{ \mu^{\star,h}_{g:h}(x_h, a_h) \times \set{ F^\star_{x_g a_g, x_h}(\mathbf{0}) + \sum_{s=1}^{t-1}\big[ - \eta \wt{M}^s_{x_ha_h, x_g a_g} +  \sum_{x_{h+1} \in \cC(x_ha_h)} G_{x_g a_g, x_{h+1}}^{\star,s}\big]} \Big\}}\\
    = &~\frac{1}{\mu_{g:h}^\star(x_h, a_h)} \log {\sum_{a_h \in \cA} n_{x_g a_g, h}^t(a_h |x_h) \exp\Big\{ \mu^{\star,h}_{g:h}(x_h, a_h) \times \big[ - \eta \wt{M}^t_{x_ha_h, x_g a_g} +  \sum_{x_{h+1} \in \cC(x_ha_h)} G_{x_g a_g, x_{h+1}}^{\star,t}\big] \Big\}},
\end{align*}
where
\begin{align*}
     n^{t}_{x_ga_g, h}(a_h \vert x_h) &\propto_{a_h} \exp\Big\{ \mu^{\star,h}_{g:h}(x_h, a_h) F^\star_{x_g a_g, x_h}(\mathbf{0}) \\
     &~+ \mu^{\star,h}_{g:h}(x_h, a_h) \sum_{s = 1}^{t-1} \Big(-   \eta \wt{M}^s_{x_ha_h, x_ga_g}+ \sum_{x_{h+1} \in \cC(x_h, a_h)} G^{\star,s}_{x_ga_g, x_{h+1}} \Big) \Big\}.
\end{align*}
Because $\mu^{\star,h}_{g:h}(x_h,a_h)$ and $F_{x_g a_g, x_h}^\star(\mathbf{0})$ are independent of $a_h$ for any $(x_h, a_h) \succeq (x_g, a_g)$ (see proof of Lemma \ref{lem:initial-balanced}), we have 
$$
n^{t}_{x_ga_g, h}(a_h \vert x_h)\propto_{a_h} \exp\Big\{  \mu^{\star,h}_{g:h}(x_h, a_h) \sum_{s = 1}^{t-1} \Big(-   \eta \wt{M}^s_{x_ha_h, x_ga_g}+ \sum_{x_{h+1} \in \cC(x_h, a_h)} G^{\star,s}_{x_ga_g, x_{h+1}} \Big) \Big\}.
$$
So $G_{x_g a_g, x_h}^{\star,t}$ and $F_{x_g a_g, x_h}^{\star,t}$ (c.f. Eq. (\ref{eqn:logZ-EFCE-recursion-t+1})) have the same recursive formula, which means that $G_{x_g a_g, x_h}^{\star,t} = F_{x_g a_g, x_h}^{\star,t}$ for any $x_g a_g$ and $x_h \succeq x_g$.

Finally, by the definition of $\wb{F}^t$ as in Eq. (\ref{eqn:def-wbF-t-in-proof}) and the definition of $F^\efce_\bal$ as in Eq. (\ref{equation:balanced-f-efce}), we have
\begin{align*}
    &~\wb{F}^t \\
    =&~ XA\log \frac{\sum_{g, x_g, a_g} \exp\Big\{ \frac{1}{XA} \big[- \langle I - E_{\succeq x_ga_g}, \eta \sum_{s=1}^{t} \wt{M}^s\rangle  + F_{x_g a_g, x_g}^\star(\eta \sum_{s=1}^{t} \wt{M}^s) \big] \Big\}}{\sum_{g, x_g, a_g} \exp\Big\{ \frac{1}{XA} \big[- \langle I - E_{\succeq x_ga_g}, \eta \sum_{s=1}^{t-1} \wt{M}^s\rangle  + F_{x_g a_g, x_g}^\star(\eta \sum_{s=1}^{t-1} \wt{M}^s) \big] \Big\}}\\
    =&~ XA\log \frac{\sum_{g, x_g, a_g} \exp\set{\frac{1}{XA}F^\star_{x_ga_g,x_g}(\mathbf{0})} \exp\Big\{ \frac{1}{XA} \big[\sum_{s=1}^{t} \paren{- \langle I - E_{\succeq x_ga_g}, \eta \ \wt{M}^s\rangle  + F_{x_g a_g, x_g}^{\star,s} }\big] \Big\}}{\sum_{g, x_g, a_g} \exp\set{\frac{1}{XA}F^\star_{x_ga_g,x_g}(\mathbf{0})} \exp\Big\{ \frac{1}{XA} \big[\sum_{s=1}^{t-1} \paren{- \langle I - E_{\succeq x_ga_g}, \eta \ \wt{M}^s\rangle  + F_{x_g a_g, x_g}^{\star,s} }\big] \Big\}}\\
    =&~ XA \log \sum_{g, x_ga_g} \lambda_{x_g a_g}^t \exp\set{\frac{1}{XA} [-\eta \langle I - E_{\succeq x_ga_g}, \wt{M}^t\rangle  + F_{x_g a_g, x_g}^{\star,t}]},
\end{align*}
where 
$$
\lambda_{x_ga_g}^{t} \propto_{x_ga_g} \exp\Big\{\frac{1}{XA}F^\star_{x_ga_g,x_g}(\mathbf{0}) + \frac{1}{XA} \sum_{s = 1}^{t-1} \Big( - \eta\langle  I - E_{\succeq x_g a_g}, \wt{M}^s\rangle  + F^{\star, s}_{x_ga_g,x_g} \Big) \Big\}.
$$
This proves the lemma. 
\end{proof}

\subsubsection{Bound on $\Delta_{x_ga_g}^t$ via Hessian}

The following lemma can be proved similar to Lemma D.11 in \citet{bai2022near} by calculating the Hessian of $\Delta_{x_g a_g}^t$ with respect to $\wt \ell^{t, x_g a_g}$. This result is the starting point of both Lemma~\ref{lemma:sup-delta} and Lemma~\ref{lem:bound-on-I-balanced}. 
\begin{lemma}[Bound on $\Delta_{x_ga_g}^t$]
\label{lemma:bound-delta-by-hessian}
We have, almost surely,
\begin{align*}
    & \quad \Delta_{x_ga_g}^t \le 2\eta^2 \sum_{g\le h\le h'\le H}  \sum_{h''=g}^{h}  \mu^{\star, h''}_{g:h''}(x_{h''}^t, a_{h''}^t) m^t_{x_ga_g, h''+1:h'}(x_{h'}^t, a_{h'}^t) m^t_{x_ga_g, g:h}(x_{h}^t, a_{h}^t) \\
    & \qquad\qquad \times \wt{\ell}^{t, x_ga_g}_h(x_h^t, a_h^t) \wt{\ell}^{t, x_ga_g}_{h'}(x_{h'}^t, a_{h'}^t) \cdot (\mu^t_{x_g a_g})^2 \indic{(x_{h'}^t, a_{h'}^t) \succeq (x_h^t, a_h^t) \succeq x_g} \\
    & \le 2\eta^2 \sum_{g\le h\le h'\le H} \sum_{h''=g}^{h} \sum_{x_{h'}, a_{h'}} \mu^{\star, h''}_{g:h''}(x_{h''}, a_{h''}) m^t_{x_ga_g, h''+1:h'}(x_{h'}, a_{h'}) m^t_{x_ga_g, g:h}(x_{h}, a_{h}) (\mu^t_{x_ga_g})^2 \\
    & \qquad\qquad \times \frac{\indic{x_{h}^t, a_{h}^t = x_{h}, a_{h}}}{(\mu^t_{1:h}(x_h, a_h) + \gamma (\mu^{\star, h}_{1:h}(x_h, a_h) + \mu^t_{x_ga_g} m^t_{x_ga_g, g:h}(x_h, a_h)\indic{x_h\succeq x_g}))} \\
    & \qquad\qquad \times \frac{\indic{x_{h'}^t, a_{h'}^t = x_{h'}, a_{h'}}}{(\mu^t_{1:{h'}}(x_{h'}, a_{h'}) + \gamma (\mu^{\star, h'}_{1:h'}(x_{h'}, a_{h'}) + \mu^t_{x_ga_g} m^t_{x_ga_g, g:h'}(x_{h'}, a_{h'})\indic{x_{h'}\succeq x_g}))} \\
    & \qquad\qquad \times \indic{(x_{h'},a_{h'}) \succeq (x_h, a_h) \succeq x_g}.
\end{align*}
\end{lemma}
\begin{proof}
First, notice that $F^{\star, t}_{x_ga_g, x_h}$ has a similar form as $\Xi _{1}^{t}$ in Appendix D.6 of \cite{bai2022near} with three minor differences:
\begin{itemize}
    \item $m^{t}_{x_ga_g, h}$ is used instead of their $\mu^t_{h}$ for layer $h$,
    \item $\mu^t_{x_ga_g} \wt{\ell}^{t, x_ga_g}_h$ is used instead of their $\wt{\ell}^{t}_h$ for layer $h$,
    \item We terminate the induction argument earlier at $(x_g,a_g)$ instead of the root of the full tree.
\end{itemize}
Therefore $\Delta_{x_ga_g}^t$ defined as below (cf.~\eqref{equation:deltaxgag}) is exactly the non-linear part (i.e. remainder term of the first-order Taylor expansion with respect to $\mu^t_{x_ga_g}\wt{\ell}^{t,x_ga_g}$ around $0$) within $F^{\star, t}_{x_ga_g, x_g}$:
\begin{align*}
\Delta_{x_ga_g}^t \defeq&~ F^{\star, t}_{x_ga_g, x_h} + \eta\langle  m_{x_ga_g}^t \mu^t_{x_ga_g}, \wt{\ell}^{t,x_ga_g} \rangle.
\end{align*}
Further taking the above differences into account and following exactly the same analysis as in Appendix D.6.1 of \cite{bai2022near}, we get the first inequality as claimed. The second inequality follows by definition of the loss estimator $\wt{\ell}^{t,x_ga_g}_h$ in~\eqref{equation:adaptive-bandit-estimator}.
\end{proof}

\subsubsection{Proof of Lemma \ref{lemma:sup-delta}}\label{sec:proof-sup-delta}

\begin{proof}[Proof of Lemma \ref{lemma:sup-delta}]
By Lemma~\ref{lemma:bound-delta-by-hessian}, for any $x_ga_g$, we have
\begin{align*}
    & \quad  \Delta_{x_ga_g}^t  \\
    & \le 2\eta^2 \sum_{g\le h\le h'\le H} \sum_{h''=g}^{h} \sum_{x_{h'}, a_{h'}} \mu^{\star, h''}_{g:h''}(x_{h''}, a_{h''}) m^t_{x_ga_g, h''+1:h'}(x_{h'}, a_{h'}) m^t_{x_ga_g, g:h}(x_{h}, a_{h}) (\mu^t_{x_ga_g})^2 \\
    & \qquad\qquad \times \frac{\indic{x_{h}^t, a_{h}^t = x_{h}, a_{h}}}{(\mu^t_{1:h}(x_h, a_h) + \gamma (\mu^{\star, h}_{1:h}(x_h, a_h) + \mu^t_{x_ga_g} m^t_{x_ga_g, g:h}(x_h, a_h)\indic{x_h\succeq x_g}))} \\
    & \qquad\qquad \times \frac{\indic{x_{h'}^t, a_{h'}^t = x_{h'}, a_{h'}}}{(\mu^t_{1:{h'}}(x_{h'}, a_{h'}) + \gamma (\mu^{\star, h'}_{1:h'}(x_{h'}, a_{h'}) + \mu^t_{x_ga_g} m^t_{x_ga_g, g:h'}(x_{h'}, a_{h'})\indic{x_{h'}\succeq x_g}))} \\
    & \qquad\qquad \times \indic{(x_{h'},a_{h'}) \succeq (x_h, a_h) \succeq x_g}\\
    & \le 2\eta^2 \sum_{g\le h\le h'\le H} \sum_{h''=g}^{h} \sum_{x_{h'}, a_{h'}} \mu^{\star, h''}_{g:h''}(x_{h''}, a_{h''}) m^t_{x_ga_g, h''+1:h'}(x_{h'}, a_{h'}) m^t_{x_ga_g, g:h}(x_{h}, a_{h}) (\mu^t_{x_ga_g})^2 \\
    & \qquad \times \frac{\indic{x_{h}^t, a_{h}^t = x_{h}, a_{h}}}{ \gamma \mu^t_{x_ga_g} m^t_{x_ga_g, g:h}(x_{h}, a_{h}) } \times \frac{\indic{x_{h'}^t, a_{h'}^t = x_{h'}, a_{h'}}}{ \gamma \mu^{\star, h'}_{1:h'}(x_{h'}, a_{h'}) } \times \indic{(x_{h'},a_{h'}) \succeq (x_h, a_h) \succeq x_g}\\
    & \le \frac{2\eta^2}{\gamma^2} \sum_{g\le h\le h'\le H} \sum_{h''=g}^{h} \sum_{x_{h'}, a_{h'}}\frac{\mu^t_{x_ga_g}\mu^{\star, h''}_{g:h''}(x_{h''}, a_{h''}) m^t_{x_ga_g, h''+1:h'}(x_{h'}, a_{h'})} {\mu^{\star, h'}_{1:h'}(x_{h'}, a_{h'})}\indic{(x_{h'},a_{h'}) \succeq x_g}\\
    & \overset{(i)}{\le} \frac{2\eta^2}{\gamma^2} \sum_{g\le h\le h'\le H} \sum_{h''=g}^{h} \sum_{x_{h'}, a_{h'}} \\
    & \qquad \frac{\mu^t_{1:g-1}(x_{g-1}, a_{g-1})\mu^{\star, h''}_{g:h''}(x_{h''}, a_{h''}) m^t_{x_ga_g, h''+1:h'}(x_{h'}, a_{h'}) \indic{(x_{h'},a_{h'}) \succeq x_g}} {\mu^{\star, h'}_{1:h'}(x_{h'}, a_{h'})}\\
    & \overset{(ii)}{\le} \frac{2\eta^2}{\gamma^2} \sum_{g\le h\le h'\le H} \sum_{h''=g}^{h} X_{h'}A\\
    & \le  \frac{2\eta^2}{\gamma^2}H^2XA.
    \end{align*}
Here, (i) uses that 
\[
\mu_{x_g a_g}^t = \mu^t_{1:g}(x_g, a_g) \le \mu^t_{1:g-1}(x_{g-1}, a_{g-1}),
\]
and (ii) uses the property of the balanced policy as in Lemma \ref{lemma:balancing}, and observing that 
\[
\mu^t_{1:g-1}(x_{g-1}, a_{g-1})\mu^{\star, h''}_{g:h''}(x_{h''}, a_{h''}) m^t_{x_ga_g, h''+1:h'}(x_{h'}, a_{h'}) \indic{(x_{h'},a_{h'})\succeq x_g}
\]
is bounded by some sequence form policy over steps $1:h'$.

Taking supremum over $x_ga_g$, we get 
\begin{align*}
    \frac{1}{XA} \sup_{g, x_g a_g} \Delta_{x_g a_g}^t \le \frac{2\eta^2}{\gamma^2} H^2.
\end{align*}
This proves the lemma. 
\end{proof}

\subsubsection{Proof of Lemma \ref{lem:bound-on-I-balanced}}\label{sec:proof-bound-on-I-balanced}

\begin{proof}[Proof of Lemma \ref{lem:bound-on-I-balanced}]
We first upper bound ${\rm I}_t$:  
\begin{align*}
  {\rm I}_t \le & 8\eta \sum_{g\le h\le h'\le H} \sum_{h''=g}^{h} \sum_{x_{h'}, a_{h'}} \sum_{x_ga_g} \lambda_{x_ga_g}^t \mu^{\star, h''}_{g:h''}(x_{h''}, a_{h''}) m^t_{x_ga_g, h''+1:h'}(x_{h'}, a_{h'}) m^t_{x_ga_g, g:h}(x_{h}, a_{h}) (\mu^t_{x_ga_g})^2 \\
    & \qquad\qquad \times \frac{\indic{x_{h}^t, a_{h}^t = x_{h}, a_{h}}}{(\mu^t_{1:h}(x_h, a_h) + \gamma (\mu^{\star, h}_{1:h}(x_h, a_h) + \mu^t_{x_ga_g} m^t_{x_ga_g, g:h}(x_h, a_h)\indic{x_h\succeq x_g}))} \\
    & \qquad\qquad \times \frac{\indic{x_{h'}^t, a_{h'}^t = x_{h'}, a_{h'}}}{(\mu^t_{1:{h'}}(x_{h'}, a_{h'}) + \gamma (\mu^{\star, h'}_{1:h'}(x_{h'}, a_{h'}) + \mu^t_{x_ga_g} m^t_{x_ga_g, g:h'}(x_{h'}, a_{h'})\indic{x_{h'}\succeq x_g}))} \\
    & \qquad\qquad \times \indic{(x_{h'},a_{h'}) \succeq (x_h, a_h) \succeq x_g}\\ 
    \stackrel{(i)}{\le} &8\eta \sum_{g\le h\le h'\le H} \sum_{h''=g}^{h} \sum_{x_{h'}, a_{h'}} \sum_{x_ga_g}  \mu^{\star, h''}_{g:h''}(x_{h''}, a_{h''}) m^t_{x_ga_g, h''+1:h'}(x_{h'}, a_{h'})  \mu^t_{x_ga_g} \mu^t_{1:{h}}(x_{h}, a_{h})\\
    & \qquad\qquad \times \frac{\indic{x_{h}^t, a_{h}^t = x_{h}, a_{h}}}{(\mu^t_{1:h}(x_h, a_h) + \gamma (\mu^{\star, h}_{1:h}(x_h, a_h) + \mu^t_{x_ga_g} m^t_{x_ga_g, g:h}(x_h, a_h)\indic{x_h\succeq x_g}))} \\
    & \qquad\qquad \times \frac{\indic{x_{h'}^t, a_{h'}^t = x_{h'}, a_{h'}}}{(\mu^t_{1:{h'}}(x_{h'}, a_{h'}) + \gamma (\mu^{\star, h'}_{1:h'}(x_{h'}, a_{h'}) + \mu^t_{x_ga_g} m^t_{x_ga_g, g:h'}(x_{h'}, a_{h'})\indic{x_{h'}\succeq x_g}))} \\
    & \qquad\qquad \times \indic{(x_{h'},a_{h'}) \succeq (x_h, a_h) \succeq x_g}\\
    \le &8\eta \sum_{g\le h\le h'\le H} \sum_{h''=g}^{h} \sum_{x_{h'}, a_{h'}} \sum_{x_ga_g}  \mu^{\star, h''}_{g:h''}(x_{h''}, a_{h''}) m^t_{x_ga_g, h''+1:h'}(x_{h'}, a_{h'})  \mu^t_{x_ga_g} \\
    & \qquad\qquad \times \frac{\indic{x_{h'}^t, a_{h'}^t = x_{h'}, a_{h'}} \times \indic{(x_{h'},a_{h'}) \succeq (x_h, a_h) \succeq x_g}}{(\mu^t_{1:{h'}}(x_{h'}, a_{h'}) + \gamma (\mu^{\star, h'}_{1:h'}(x_{h'}, a_{h'}) + \mu^t_{x_ga_g} m^t_{x_ga_g, g:h'}(x_{h'}, a_{h'})\indic{x_{h'}\succeq x_g}))} \\
    \le & 8\eta H \sum_{g\le h'\le H} \sum_{h''=g}^{h'} \sum_{x_{h'}, a_{h'}} \sum_{x_ga_g}  \mu^{\star, h''}_{g:h''}(x_{h''}, a_{h''}) m^t_{x_ga_g, h''+1:h'}(x_{h'}, a_{h'})  \mu^t_{x_ga_g} \\
    & \qquad\qquad \times \frac{\indic{x_{h'}^t, a_{h'}^t = x_{h'}, a_{h'}} \times \indic{x_{h'} \succeq x_g}}{\mu^t_{1:{h'}}(x_{h'}, a_{h'}) + \gamma (\mu^{\star, h'}_{1:h'}(x_{h'}, a_{h'}) + \mu^t_{x_ga_g} m^t_{x_ga_g, g:h'}(x_{h'}, a_{h'}))}\\
    \stackrel{(ii)}{=} &8 \eta H \sum_{g\le h'\le H} \sum_{h''=g}^{h'}  \wt{\Delta}^t_{g,h',h''}.
\end{align*}
Here, (i) used the fact that $\lambda_{x_g a_g}^t m_{x_g a_g, g:h}^t(x_h, a_h) \mu_{x_g a_g}^t \le \mu_{1:h}^t(x_h, a_h)$ as shown in Eq. (\ref{equation:lambdamum}). Moreover, in (ii), we define for any $g\le h''\le h'$:
\begin{align*}
\wt{\Delta}^t_{g,h',h''} \defeq &~ \sum_{x_{h'}, a_{h'}} \sum_{x_ga_g}  \mu^{\star, h''}_{g:h''}(x_{h''}, a_{h''}) m^t_{x_ga_g, h''+1:h'}(x_{h'}, a_{h'})  \mu^t_{x_ga_g} \\
    & \qquad\qquad \times \frac{\indic{x_{h'}^t, a_{h'}^t = x_{h'}, a_{h'}} \times \indic{x_{h'} \succeq x_g}}{\mu^t_{1:{h'}}(x_{h'}, a_{h'}) + \gamma (\mu^{\star, h'}_{1:h'}(x_{h'}, a_{h'}) + \mu^t_{x_ga_g} m^t_{x_ga_g, g:h'}(x_{h'}, a_{h'}))}.
\end{align*}
Observe that the random variable $\wt{\Delta}^t_{g,h',h''}$ satisfies the following properties:
\begin{itemize}[leftmargin=1.5pc]
    \item $\wt{\Delta}^t_{g,h',h''} \le X_{h'}A /\gamma $ almost surely: We have 
    \begin{align*}
    \wt{\Delta}^t_{g,h',h''} \le   \frac{1}{\gamma} \sum_{x_{h'}, a_{h'}} \frac{\sum_{x_ga_g}\mu^{\star, h''}_{g:h''}(x_{h''}, a_{h''}) m^t_{x_ga_g, h''+1:h'}(x_{h'}, a_{h'})  \mu^t_{x_ga_g}\indic{x_{h'} \succeq x_g}}{\mu_{1:h'}^{\star,h'}(x_{h'},a_{h'})}. 
    \end{align*}
    Notice that (for this fixed $g, h''$)
    \begin{align}
    \label{equation:concatenated-sequence-form}
        \sum_{x_ga_g}\mu^{\star, h''}_{g:h''}(x_{h''}, a_{h''}) m^t_{x_ga_g, h''+1:h'}(x_{h'}, a_{h'})  \mu^t_{x_ga_g}\indic{x_{h'} \succeq x_g}
    \end{align}
    is the sequence-form of a certain policy at $(x_{h'}, a_{h'})$, where the policy is defined as follows: First, take policy $\mu^t_{1:g}$ and arrive at some $x_g\in\cX_g$. Let $a_g$ be the action sampled from $\mu^t_g(\cdot|x_g)$. Then, starting from $x_g$, discard $a_g$ and instead take policy $\mu^{\star, h''}_{g:h''}m^t_{x_ga_g, h''+1:H}$  until the end of the game. One may check that the sequence-form of this policy is indeed given by~\eqref{equation:concatenated-sequence-form}. Therefore, we have $\wt{\Delta}^t_{g,h,h''}\le X_{h'}A/\gamma$ by the balancing property of $\mu^{\star,h'}_{1:h'}$ (Lemma \ref{lemma:balancing}). %
    \item $\E[\wt{\Delta}^t_{g,h',h''}|\cF_{t-1}] \le 1$: We have
    \begin{align*}
      &\E[\wt{\Delta}^t_{g,h',h''}|\cF_{t-1}]\\
      \le&  \sum_{x_{h'}, a_{h'}} \sum_{x_ga_g}  \mu^{\star, h''}_{g:h''}(x_{h''}, a_{h''}) m^t_{x_ga_g, h''+1:h'}(x_{h'}, a_{h'})  \mu^t_{x_ga_g} p^t(x_{h'})\indic{x_{h'} \succeq x_g} = 1.
    \end{align*}
    Above, the last equality used again the fact that~\eqref{equation:concatenated-sequence-form} is the sequence-form of a policy.
    \item $\E[(\wt{\Delta}^t_{g,h',h''})^2|\cF_{t-1}] \le X_{h'}A/\gamma$: Note that $\wt{\Delta}^t_{g,h',h''}$ is non-negative, so by the almost sure bound that $\wt{\Delta}^t_{g,h',h''} \le X_{h'}A/\gamma$, we have
    \[
    \E[(\wt{\Delta}^t_{g,h',h''})^2|\cF_{t-1}] \le \E[\wt{\Delta}^t_{g,h',h''}|\cF_{t-1}] \cdot X_{h'}A/\gamma \le X_{h'}A/\gamma. 
    \]
\end{itemize}

By Freedman's inequality (Lemma \ref{lemma:freedman}) and taking the union bound, with probability at least $1-\delta / (10 H^3)$ and some fixed $\lambda \le \gamma/(X_{h'}A)$, we get %
\begin{align*}
    \sum_{t=1}^T  \wt{\Delta}^t_{g,h',h''} \le&~ T + \frac{\lambda X_{h'}A T}{\gamma}+ \frac{4 \log(10H/\delta)}{\lambda}.
\end{align*}
Taking $\lambda = \gamma/(X_{h'}A)$, we have
\begin{align*}
    \sum_{t=1}^T  \wt{\Delta}^t_{g,h',h''} \le&~ 2T + \frac{4X_{h'}A\log(10H/\delta)}{\gamma}.
\end{align*}
Finally summing up $\wt{\Delta}^t_{g,h',h''}$ over $g, h, h''$ and taking the union bound, we have with probability at least $1-\delta/10$, we have
\begin{equation*}
    \sum_{t = 1}^T {\rm I}_t \le 16\eta H^4 T+\frac{32 \eta H^3XA\iota }{\gamma},
\end{equation*}
where $\iota := \log(10 H/\delta)$. This proves the lemma. 
\end{proof}

\subsubsection{Proof of Lemma \ref{lem:bound-on-II-balanced}}\label{sec:proof-bound-on-II-balanced}

\begin{proof}
First, recall that the matrix loss estimator is defined as $\wt{M}^t = \sum_{g, x_g a_g} \mu^t_{g,x_ga_g} \wt{\ell}^{t, x_ga_g} e_{x_g a_g}^\top$, and the vector loss estimator is computed by 
\begin{align*}
    \wt{\ell}^{t, x_g a_g}_{h}(x_h, a_h) = \frac{\indic{(x_h^t, a_h^t) = (x_h, a_h)} (1 - r_h^t)}{\mu^t_{1:h}(x_h, a_h) + \gamma (\mu^{\star, h}_{1:h}(x_h, a_h) + \mu^t_{x_ga_g} m^t_{x_ga_g, g:h}(x_h, a_h)\indic{x_h\succeq x_g})}.
\end{align*}
We define a vector $\wt{\ell}^t = \{ \wt{\ell}_h^t(x_h,a_h) \}_{(x_h, a_h) \in \cX \times \cA} \in \R_{\ge 0}^{XA}$ as
\begin{equation}\label{equation:loss-estimator-old}
\wt{\ell}^t_h(x_h,a_h)  \defeq \frac{\indic{(x_h^t, a_h^t) = (x_h, a_h)} (1 - r_h^t)}{\mu^t_{1:h}(x_h, a_h) + \gamma \mu^{\star, h}_{1:h}(x_h, a_h) }.
\end{equation}
Then we have for any $(t, x_ga_g)$, $(x_h, a_h)$ that
\begin{align*}
    \wt{\ell}^{t, x_g a_g}_{h}(x_h, a_h) \le \wt{\ell}^t_h(x_h,a_h). 
\end{align*}

Then $\langle  I - E_{\succeq x_ga_g} + m_{x_ga_g}^t e_{x_ga_g}^\top, \wt{M}^t\rangle $ can be upper bounded as follows:
\begin{align*}
    &\langle  I - E_{\succeq x_ga_g} + m_{x_ga_g}^t e_{x_ga_g}^\top, \wt{M}^t\rangle  \\
    =&~ \langle  I - E_{\succeq x_ga_g} + m_{x_ga_g}^t e_{x_ga_g}^\top,  \sum_{h, x_h a_h} \mu^t_{x_ha_h}  \wt{\ell}^{t, x_g a_g} e_{x_h a_h}^\top\rangle  \\
    \le &~ \langle  I - E_{\succeq x_ga_g} + m_{x_ga_g}^t e_{x_ga_g}^\top, \wt{\ell}^t \sum_{h, x_h, a_h} \mu^t_{x_ha_h}   e_{x_h a_h}^\top\rangle \\
   = &~   \langle  \phi_{x_g a_g \to m^t_{x_g a_g}} , \wt{\ell}^t (\mu^t)^\top    \rangle  = \langle \phi_{x_g a_g \to m^t_{x_g a_g}} \mu^t, \wt{\ell}^t \rangle \\
   = &~ \sum_{h=1}^H \sum_{h, x_h a_h} (\phi_{x_g a_g \to m^t_{x_g a_g}} \mu^t)_{1:h}(x_h,a_h) \wt\ell^t_h(x_h,a_h),
\end{align*}
where we have used $\phi_{x_g a_g \to m^t_{x_g a_g}} \defeq I - E_{\succeq x_ga_g} + m_{x_ga_g}^t e_{x_ga_g}^\top$ to denote the EFCE modification triggered at $(x_g, a_g)$ and then playing the policy $m^t_{x_g a_g}$. Also, $\langle  I - E_{\succeq x_ga_g} + m_{x_ga_g}^t e_{x_ga_g}^\top, \wt{M}^t\rangle \ge 0$ as both matrices have non-negative entries. As a result, we get that 
\begin{align*}
    & \sum_{t=1}^T\sum_{g, x_ga_g}\lambda_{x_ga_g}^t \langle  I - E_{\succeq x_ga_g} + m_{x_ga_g}^t e_{x_ga_g}^\top, \wt{M}^t\rangle ^2\\
    \le&~ \sum_{t=1}^T\sum_{g, x_ga_g}\lambda_{x_ga_g}^t \paren{ \langle   \phi_{x_g a_g \to m^t_{x_g a_g}}  \mu^t, \wt\ell^t\rangle }^2 \\
    \le &~ 2\sum_{t=1}^T\sum_{g, x_g a_g} \lambda_{x_ga_g}^t \sum_{1\le h\le h'\le H} \sum_{x_h,a_h} \sum_{(x_{h'},a_{h'})\in \cC_{h'}(x_h, a_h)}\\
    & \qquad \frac{(\phi_{x_g a_g \to m^t_{x_g a_g}} \mu^t)_{1:h}(x_h,a_h)\indic{(x_{h'}^t, a_{h'}^t)=(x_{h'}, a_{h'}) } \cdot (1 - r_h^t)\cdot (1 - r_{h'}^t)}{(\mu _{1:h}^{t}(x_h,a_h)+\gamma \mu^{\star, h}_{1:h}(x_h, a_h)) (\mu _{1:h'}^{t}(x_{h'},a_{h'})+\gamma \mu^{\star, h'}_{1:h'}(x_{h'}, a_{h'}))} \\
    \le &~ 2\sum_{1\le h\le h'\le H} \sum_{t=1}^T \sum_{x_h,a_h} \sum_{(x_{h'},a_{h'})\in \cC_{h'}(x_h, a_h)} \sum_{g, x_g a_g} \frac{\lambda_{x_ga_g}^t(\phi_{x_g a_g \to m^t_{x_g a_g}}  \mu^t)_{1:h}(x_h,a_h)}{\mu _{1:h}^{t}(x_h,a_h)} \wt{\ell}_{h'}^t(x_{h'},a_{h'}) \\
    \overset{(i)}{=}&~ 2\sum_{1\le h\le h'\le H} \sum_{t=1}^T \sum_{x_h,a_h} \sum_{(x_{h'},a_{h'})\in \cC_{h'}(x_h, a_h)} \wt{\ell}_{h'}^t(x_{h'},a_{h'}) \\
    \le &~ 2H \sum_{t=1}^T \sum_{h',x_{h'},a_{h'}} \wt{\ell}_{h'}^t(x_{h'},a_{h'}) \\
    \overset{\left( ii \right)}{\le} &~ 2H \sum_{t=1}^T \sum_{h',x_{h'},a_{h'}} \underbrace{\ell_{h'}^t(x_{h'},a_{h'})}_{\le 1} + 2H \sum_{h',x_{h'},a_{h'}} \frac{\log( 10 XA / \delta)}{\gamma \mu_{1:h'}^{\star,h' }(x_{h'}, {a_{h'}}) } \\
    \overset{(iii)}{\le}&~ 2HXAT + 2H  \sum_{h',x_{h'},a_{h'}} \iota \cdot X_{h'}A/\gamma \\
    \le &~ 2HXAT + 2HX^2A^2\iota /\gamma.
\end{align*}
Above, $(i)$ uses that $\mu^t$ is the solution of the fixed point equation $\mu = \sum_{g,x_g a_g} \lambda_{x_g, a_g}^t (I - E_{\succeq x_g a_g} + m^t_{x_ga_g} e_{x_g a_g}^\top) \mu$;  $(ii)$ is by \citep[Corollary D.6]{bai2022near} for each $(h', x_{h'}, a_{h'})$ with probability $1-\delta/(10XA)$ and a union bound; (iii) uses $\mu_{1:h'}^{\star,h' }(x_{h'}, {a_{h'}}) \ge 1/(X_{h'}A)$ by Corollary~\ref{cor:balanced-policy-lower-bound}.
Therefore, we have with probability at least $1-\delta/10$ that
\begin{align*}
    & \quad \sum_{t=1}^T {\rm II}_t = \frac{e\eta}{XA} \cdot \sum_{t=1}^T\sum_{g, x_ga_g}\lambda_{x_ga_g}^t \langle  I - E_{\succeq x_ga_g} + m_{x_ga_g}^t e_{x_ga_g}^\top, \wt{M}^t\rangle ^2 \\
    & \le \frac{e\eta}{XA}\paren{ 2HXAT+ 2HX^2A^{2}\iota/\gamma} \le 6\eta HT + 6\eta H X A \iota/\gamma.
\end{align*}
This proves the lemma. 
\end{proof}

\subsection{Bound on two bias terms} \label{app:proof-bias}

\begin{proof}[Proof of Lemma \ref{lem:bias1-balanced}]
First, recall that the matrix loss estimator gives $\wt{M}^t = \sum_{g,x_g, a_g} \mu^t_{x_ga_g} \wt{\ell}^{t, x_ga_g} e_{x_g a_g}^\top$ and the vector loss estimator is computed by 
\begin{align*}
    \wt{\ell}^{t, x_g a_g}_{h}(x_h, a_h) = \frac{\indic{(x_h^t, a_h^t) = (x_h, a_h)} (1 - r_h^t)}{\mu^t_{1:h}(x_h, a_h) + \gamma (\mu^{\star, h}_{1:h}(x_h, a_h) + \mu^t_{x_ga_g} m^t_{x_ga_g, g:h}(x_h, a_h)\indic{x_h\succeq x_g})}.
\end{align*}
Then we decompose ${\rm BIAS}_1$ as
\begin{align*}
    {\rm BIAS}_1 =&~ \sum_{t=1}^T \langle  \phi^t, \ell^t (\mu^t)^\top - \wt{M}^t \rangle \\
    =&~ \underbrace{\sum_{t=1}^T \langle  \phi^t, \ell^t (\mu^t)^\top - \E\brac{\wt{M}^t | \cF_{t-1}} \rangle }_{(A)} + \underbrace{\sum_{t=1}^T \langle  \phi^t, \E\brac{\wt{M}^t | \cF_{t-1}} - \wt{M}^t \rangle }_{(B)}.
\end{align*}
We first the second term $(B)$ by Azuma-Hoeffding inequality. Recall the definition of $\wt\ell^t$ in \eqref{equation:loss-estimator-old}. We immediately have $\wt{\ell}^{t, x_ga_g} \le \wt\ell^t$ pointwisely, so we can upper bound $\langle  \phi^t, \wt{M}^t\rangle $ by 
\begin{align*}
    \langle  \phi^t, \wt{M}^t \rangle  \le \langle  \phi^t, \sum_{g,x_g, a_g} \mu^t_{x_g a_g} \wt\ell^t e_{x_g a_g}^\top \rangle  = \langle  \phi^t \mu^t , \wt\ell^t\rangle  = \langle  \mu^t, \wt\ell^t \rangle  ,
\end{align*}
where the last equality comes from fixed point equation $\mu^t = \phi^t \mu^t$. Then we have
\begin{align*}
     & \langle  \phi^t, \wt M^t \rangle  \le \langle  \mu^t, \wt\ell^t \rangle \\
     =&~{\sum_{h=1}^H{\sum_{x_h,a_h}{\mu  _{1:h}^{t}(x_h,a_h)\frac{\indic{(x_h^t, a_h^t)=(x_h, a_h) } \cdot (1 - r_h^t)}{\mu _{1:h}^{t}(x_h,a_h)+\gamma \mu_{1:h}^{\star, h}(x_h, a_h)} }}}
\\
\le&~ \sum_{h=1}^H{\sum_{x_h,a_h}{\ones\left\{ x_h=x_{h}^{t},a_h=a_{h}^{t} \right\}}}=\sum_{h=1}^H{1}=H.
\end{align*}
As a consequence, by Azuma-Hoeffding inequality, with probability at least $1-\delta/10$, we have
$$
\sum_{t=1}^T{\langle  \phi^t, \E\brac{\wt{M}^t | \cF_{t-1}} - \wt{M}^t \rangle }\le H\sqrt{2T\log(10/\delta)} \le H\sqrt{2T\iota}.
$$

Then we turn to bound the first term $(A)$. Denote $\ell^{t, x_ga_g} = \E\brac{\wt\ell^{t,x_ga_g}|\cF_{t-1}}$ and plug in the definition of $\wt\ell^{t, x_g a_g}$, we get
\begin{align*}
    &\langle  \phi^t, \ell^t (\mu^t)^\top - \E\brac{\wt{M}^t | \cF_{t-1}} \rangle  \\
    =& \langle  \phi^t, \ell^t (\mu^t)^\top \rangle -  \sum_{g, x_g, a_g} \langle  \phi^t, \mu_{x_g a_g}^t \ell^{t, x_ga_g} e_{x_ga_g}^\top \rangle \\
    = & \sum_{g, x_g, a_g} \langle  \phi^t e_{x_g a_g} \mu_{x_g a_g}^t, \ell^t - \ell^{t, x_ga_g} \rangle .
\end{align*}
Note that by the definition of the loss estimator as in Eq. (\ref{equation:l-definition}), we have
\[
\ell_{h}^t(x_{h}, a_{h}) = p^t(x_h) [1 - \bar{R}_{h}^t(x_h, a_h)] \le p^t(x_h),
\]
where we recall the definition of $p^t(x_h)$ in~\eqref{equation:ptxh}.

Moreover, the $\ell^{t, x_ga_g}(x_h, a_h)$ is related to $\ell_h^t(x_h, a_h)$ by a rescaling
\begin{align*}
    \ell^{t, x_ga_g}(x_h, a_h) = \frac{\mu_{1:h}^t(x_h, a_h) \ell_h^t(x_h, a_h)}{\mu^t_{1:h}(x_h, a_h) + \gamma (\mu^{\star, h}_{1:h}(x_h, a_h) + \mu^t_{x_ga_g} m^t_{x_ga_g, g:h}(x_h, a_h)\indic{x_h\succeq x_g})}.
\end{align*}
So we get 
\begin{align*}
    &\langle  \phi^t, \ell^t (\mu^t)^\top - \E\brac{\wt{M}^t | \cF_{t-1}} \rangle  \\
    = & \sum_{g, x_g, a_g} \sum_{h, x_h, a_h} \mu_{x_ga_g}^t (\phi^t e_{x_g a_g})_{1:h}(x_h, a_h)  \\
    &~~~ \times \frac{\gamma (\mu^{\star, h}_{1:h}(x_h, a_h) + \mu^t_{x_ga_g} m^t_{x_ga_g, g:h}(x_h, a_h)\indic{x_h\succeq x_g}) \ell_h^t (x_h, a_h)}{\mu^t_{1:h}(x_h, a_h) + \gamma (\mu^{\star, h}_{1:h}(x_h, a_h) + \mu^t_{x_ga_g} m^t_{x_ga_g, g:h}(x_h, a_h)\indic{x_h\succeq x_g})} \\
    \le & \gamma\sum_{g, x_g, a_g} \sum_{h, x_h, a_h}   \frac{\mu_{x_ga_g}^t (\phi^t e_{x_g a_g})_{1:h}(x_h, a_h)  \mu^{\star, h}_{1:h}(x_h, a_h) p^t(x_h)}{\mu^t_{1:h}(x_h, a_h) + \gamma (\mu^{\star, h}_{1:h}(x_h, a_h) + \mu^t_{x_ga_g} m^t_{x_ga_g, g:h}(x_h, a_h)\indic{x_h\succeq x_g})}  \\
    & + \gamma \sum_{g, x_g, a_g} \sum_{h, x_h, a_h}   \frac{\mu_{x_ga_g}^t (\phi^t e_{x_g a_g})_{1:h}(x_h, a_h)  \mu^t_{x_ga_g} m^t_{x_ga_g, g:h}(x_h, a_h)\indic{x_h\succeq x_g}  p^t(x_h)}{\mu^t_{1:h}(x_h, a_h) + \gamma (\mu^{\star, h}_{1:h}(x_h, a_h) + \mu^t_{x_ga_g} m^t_{x_ga_g, g:h}(x_h, a_h)\indic{x_h\succeq x_g})}.
\end{align*}
The first term admits an upper bound 
\begin{align*}
    & \gamma \sum_{h, x_h, a_h}  \sum_{g, x_g, a_g} \frac{\mu_{x_ga_g}^t (\phi^t e_{x_g a_g})_{1:h}(x_h, a_h) }{\mu^t_{1:h}(x_h, a_h) } \mu^{\star, h}_{1:h}(x_h, a_h) p^t(x_h) \\
    \overset{(i)}{=}~ & \gamma \sum_{h, x_h, a_h}   \mu^{\star, h}_{1:h}(x_h, a_h) p^t(x_h) 
    \stackrel{(ii)}{=} \gamma H.
\end{align*}
Here, $(i)$ uses $\mu^t = \phi^t \mu^t$ and $(ii)$ uses Eq. (\ref{eqn:mu_pt_1}). 

The second term can be upper bounded by
\begin{align*}
    &\gamma  \sum_{h, x_h, a_h}  \sum_{g, x_g, a_g, x_h \succeq x_g} \frac{\mu_{x_ga_g}^t (\phi^t e_{x_g a_g})_{1:h}(x_h, a_h)  \mu^t_{x_ga_g} m^t_{x_ga_g, g:h}(x_h, a_h)  p^t(x_h)}{\mu^t_{1:h}(x_h, a_h) } \\
    \le ~& \gamma  \sum_{h, x_h, a_h}  \paren{\sum_{g, x_g, a_g, x_h \succeq x_g} \frac{\mu_{x_ga_g}^t (\phi^t e_{x_g a_g})_{1:h}(x_h, a_h)  }{\mu^t_{1:h}(x_h, a_h) }} \cdot \paren{\sum_{g, x_g, a_g, x_h \succeq x_g} \mu^t_{x_ga_g} m^t_{x_ga_g, g:h}(x_h, a_h)  p^t(x_h)} \\
    \overset{(i)}{\le} ~& \gamma  \sum_{h, x_h, a_h}\sum_{g, x_g, a_g, x_h \succeq x_g} \mu^t_{x_ga_g} m^t_{x_ga_g, g:h}(x_h, a_h)  p^t(x_h) \\
    = ~& \gamma \sum_{h, x_h, a_h} \sum_{g = 1}^h \sum_{x_ga_g} \mu^t_{x_ga_g} m^t_{x_ga_g, g:h}(x_h, a_h) \indic{x_h\succeq x_g} \cdot p^t(x_h) \\
    = ~& \gamma \sum_{1\le g\le h\le H} \sum_{x_h,a_h} \sum_{x_ga_g} \mu^t_{x_ga_g} m^t_{x_ga_g, g:h}(x_h, a_h) \indic{x_h\succeq x_g} \cdot p^t(x_h) \\
    \overset{(ii)}{\le} ~& \gamma H^2.
\end{align*}
Here, the inequality in $(i)$ also uses $\mu^t = \phi^t \mu^t$; (ii) used the fact for any fixed $g$, $\sum_{x_ga_g} \mu^t_{x_ga_g} m^t_{x_ga_g, g:h}(x_h, a_h) \indic{x_h\succeq x_g}$ is the sequence-form of a policy, similar as~\eqref{equation:concatenated-sequence-form}.

Taking summation over $t=1, 2,\cdots, T$, we have
\begin{align*}
    (A) \le 2 \gamma H^2 T.
\end{align*}
Combined with the bound on $(B)$, we have with probability at least $1-\delta/10$ that
\begin{align*}
    {\rm BIAS}_1 \le 2 \gamma H^2 T + 2H \sqrt{T\iota}.
\end{align*}
This completes the proof of this lemma. 
\end{proof}

\begin{proof}[Proof of Lemma \ref{lem:bias2-balanced}]
Recall the definition of $\wt\ell^t$ in \eqref{equation:loss-estimator-old}. We have $\wt{\ell}^{t, x_ga_g} \le \wt\ell^t$ pointwisely, so we have
\begin{align*}
    \langle \phi^\star, \wt{M}^t - \ell^t(\mu^t)^\top \rangle =&~ \langle \phi^\star,  \sum_{g, x_ga_g} \wt\ell^{t, x_ga_g} \mu_{x_g a_g}^t e_{x_ga_g}^\top - \ell^t(\mu^t)^\top \rangle  \\
    \le&~ \langle  \phi^\star, \wt\ell^t (\mu^t)^\top - \ell^t(\mu^t)^\top\rangle  = \langle  \phi^\star \mu^t, \wt\ell^t - \ell^t \rangle. 
\end{align*}
Then we can get that with probability at least $1 - \delta/3$
\begin{align*}
    & {\rm BIAS}_2 \le \max_{\phi^\star \in\Phi^\efce} \sum_{t=1}^T{\left< \phi^\star \mu^t ,\widetilde{\ell }^t-\ell ^t \right>}
\\
=& \max_{\phi^\star\in\Phi^\efce} \sum_{t=1}^T{\sum_{h=1}^H{\sum_{x_h,a_h}{(\phi^\star \mu^t) _{1:h}(x_h,a_h)\left[ \wt{\ell}_h^t(x_h,a_h)-\ell_h ^t(x_h,a_h) \right]}}}
\\
=& \max_{\phi^\star \in\Phi^\efce}  \sum_{t=1}^T{\sum_{h=1}^H{\sum_{x_h,a_h}{\frac{(\phi^\star \mu^t) _{1:h}(x_h,a_h)}{\gamma \mu^{\star,h}_{1:h}(x_h, a_h) }\gamma \mu^{\star,h}_{1:h}(x_h, a_h) \left[ \wt{\ell}_h^t(x_h,a_h)-\ell_h ^t(x_h,a_h) \right]}}}
\\
=& \max_{\phi^\star\in\Phi^\efce}  \sum_{h=1}^H{\sum_{x_h,a_h}{\frac{(\phi^\star \mu^t) _{1:h}(x_h,a_h)}{\gamma \mu^{\star,h}_{1:h}(x_h, a_h) }\sum_{t=1}^T{\gamma \mu^{\star,h}_{1:h}(x_h, a_h) \left[ \wt{\ell}_h^t(x_h,a_h)-\ell_h ^t(x_h,a_h) \right]}}}
\\
\overset{\left( i \right)}{\le} &~\frac{\log \left( 3 XA/\delta \right)}{\gamma }\max_{\phi^\star\in\Phi^\efce}\sum_{h=1}^H   {\sum_{x_h,a_h}\frac{(\phi^\star \mu^t)_{1:h}(x_h,a_h)}{ \mu^{\star,h}_{1:h}(x_h, a_h)}}
\\
\overset{(ii)}{=} &~\frac{\iota}{\gamma}\sum_{h=1}^H{X_hA}=XA\iota/\gamma ,
\end{align*}
where (i) is a high probability bound by applying Corollary D.6 in \citet{bai2022near} for each $(x_h, a_h)$ and taking union bound, and (ii) is by the balancing property of $\mu^{\star,h}$. This proves the lemma. 
\end{proof}

\section{Equivalence between Vertex MWU and OMD}\label{app:equivalence-vertex-hedge}

\subsection{Proof of Theorem~\ref{theorem:equivalence-ftrl}}
\label{appendix:proof-equivalence}

In this section we prove Theorem~\ref{theorem:equivalence-ftrl}. Our proof is based on Algorithm~\ref{algorithm:vertex-ftrl}, which is just (the efficient implementation of) the standard OMD algorithm with dilated entropy regularizer in FTRL form~\citep{kroer2015faster}. Indeed, Lemma~\ref{lem:dilated-ent-optimization} show that its output policy $\set{\mu^t}_{t\ge 1}$ is the same as~\eqref{equation:dilated-ent-omd}. Then, Lemma~\ref{lem:vertex-free-energy-trick} \&~\ref{lem:evaluate-partition-function-gradient-vertex} show that its output policy $\set{\mu^t}_{t\ge 1}$ is the same as~\eqref{equation:vertex-ftrl}. These together imply the equivalence of~\eqref{equation:dilated-ent-omd} and~\eqref{equation:vertex-ftrl}, thereby proving Theorem~\ref{theorem:equivalence-ftrl}. 

The rest of this subsection is devoted to stating and proving Lemma~\ref{lem:vertex-free-energy-trick}-\ref{lem:dilated-ent-optimization}.

\paragraph{Remark on optimistic algorithms}
As pointed out in \cite{farina2022kernelized}, Theorem~\ref{theorem:equivalence-ftrl} does not depend on the concrete values of $\set{\ell^t}_{t\ge 1}$. As a result, the equivalence also holds for the optimistic version of the algorithms (where the algorithms are fed with loss functions $\set{2\ell^t-\ell^{t-1}}_{t\ge 1}$, with $\ell^0\defeq 0$) which achieves an faster $\cO({\rm poly}(\log T))$ regret. In words: The Kernelized OMWU algorithm of~\citet{farina2022kernelized} is equivalent to an Optimistic OMD algorithm with the dilated KL distance.

\begin{algorithm}[t]
\caption{OMD (FTRL form)}
\label{algorithm:vertex-ftrl}
\begin{algorithmic}[1]
\REQUIRE Learning rate $\eta>0$. 
\FOR{$t = 1, 2, \ldots, T$}
\STATE Compute $\mu^{t}_{h}(a_h \vert x_h)$ and $F^{t}_{x_h}$ in the bottom-up order over $x_h\in\cX$:
\begin{align}
& \textstyle \mu^{t}_{h}(a_h \vert x_h) \propto_{a_h} \exp\Big\{ - \eta\sum_{s=1}^{t-1} \ell^s_{h}(x_h, a_h) + \sum_{x_{h+1} \in \cC(x_h, a_h)} F^{t}_{ x_{h+1}} \Big\},  \label{equation:muht} \\
&~\textstyle F^{t}_{x_h} = \log  \sum_{a_h}  \exp\Big\{ -\eta  \sum_{s=1}^{t-1}\ell^s_{h}(x_h, a_h) + \sum_{x_{h+1} \in \cC(x_h, a_h)} F^{t}_{ x_{h+1}} \Big\}. \label{equation:fxht}
\end{align}
\STATE Receive loss $\ell^t = \{\ell^t_h(x_h, a_h)\}_{(x_h, a_h) \in \cX \times \cA} \in \R^{XA}_{\ge 0}$.
    \ENDFOR
 \end{algorithmic}
\end{algorithm}

\begin{lemma}[Conversion to log-partition function]
\label{lem:vertex-free-energy-trick}
Define the log-partition function $F^{\mc{V}}: \R^{XA} \to \R$
\begin{align}
\textstyle F^{\mc{V}}(\ell) \defeq \log \sum_{v\in\mc{V}} \exp\{ - \<v, \ell\>\}.
\end{align}
Then update~\eqref{equation:vertex-ftrl} has a closed-form update for all $t\ge 1$:
\begin{align}\label{eqn:mu-update-vertex-mwu}
\textstyle
\mu^{t} = -\grad F^{\mc{V}}\paren{ \eta\sum_{s=1}^{t-1} \ell^s } = - \frac{ \sum_{v\in\mc{V}} \exp\big\{-\eta \<v, \sum_{s=1}^{t-1} \ell^s\>\big\} v }{\sum_{v\in\mc{V}} \exp\big\{-\eta\<v, \sum_{s=1}^{t-1} \ell^s\>\big\}}.
\end{align}
\end{lemma}

\begin{proof}
By \eqref{equation:vertex-ftrl},
$$
\mu^{t} = \sum_{v}p_v^{t}v = \frac{\sum_{\phi}\exp\{ - \eta \< v,\sum_{s=1}^{t-1}\ell^s \> \}v}{\sum_{v}\exp\{ - \eta \< v,\sum_{s=1}^{t-1}\ell^s \> \}} = -\grad F^{\mc{V}}\paren{ \eta\sum_{s=1}^{t-1} \ell^s }.
$$
\end{proof}

\begin{lemma}[Recursive expression of $F^{\mc{V}}$ and $\grad F^{\mc{V}}$]
\label{lem:evaluate-partition-function-gradient-vertex}
For any loss matrix $\ell \in \R^{XA}$, the \emph{log-partition function} can be written as $ F^{\mc{V}}(\ell) =  F_{\emptyset}(\ell)$ where $F_{x_h}(\ell):= \log \sum_{v\in\mc{V}^{x_h}} \exp\{ - \<v, \ell\>\}$ can be computed recurrently by $F_{x_{H+1}}(\cdot) = 0$ and
\begin{equation}%
\label{equation:vertex-log-partition-recursive}
 \textstyle   F_{x_h}(\ell) \defeq \log \sum_{a_h} \exp\Big\{  - \ell_{h}(x_h, a_h) +  \sum_{x_{h+1} \in \cC(x_h,a_h)} F_{x_{h+1}}(\ell) \Big\}. 
\end{equation}
Furthermore, define a (sequence form) policy $\mu $ by
\begin{align}
\textstyle  \mu(a_h \vert x_h) \propto_{a_h}\exp\Big\{  - \ell_{h}(x_h, a_h) +  \sum_{x_{h+1} \in \cC(x_h,a_h)} F_{x_{h+1}}(\ell) \Big\},%
\end{align}
then we have
\begin{align}
\label{equation:vertex-grad-log-partition}
-\grad F^{\mc{V}}(\ell) = \mu.
\end{align}
\end{lemma}
\begin{proof}
We first show~\eqref{equation:vertex-log-partition-recursive}. 
Using the structure of $\mc{V}^{x_h}$,
\begin{align*}
   F_{x_h}(\ell) =&\log \sum_{v\in\mc{V}^{x_h}} \exp\{ - \<v, \ell\>\} \\ 
   =&\log \sum_{a_h \in \cA_{x_h}} \exp\Big\{ - \ell_{h}(x_h, a_h) + \sum_{x_{h+1} \in \cC(x_ha_h)} \sum_{v\in\mc{V}^{x_{h+1}}} \exp\{ - \<v, \ell\>\} \Big\}.\\
   =&\log \sum_{a_h \in \cA_{x_h}} \exp\Big\{ - \ell_{h}(x_h, a_h) + \sum_{x_{h+1} \in \cC(x_ha_h)} F_{x_{h+1}}(\ell) \Big\}.
\end{align*}

Next we show~\eqref{equation:vertex-grad-log-partition}. Taking the gradient,
\begin{align*}
    & \quad -\grad F_{x_h}(\ell) \\
    & = \frac{\sum_{a_h \in \cA_{x_h}} \exp\Big\{ - \ell_{h}(x_h, a_h) + \sum_{x_{h+1} \in \cC(x_ha_h)} F_{x_{h+1}}(\ell) \Big\}[e_{x_ha_h}-\sum_{x_{h+1} \in \cC(x_ha_h)} \grad F_{x_{h+1}}(\ell)]}{\sum_{a_h \in \cA_{x_h}} \exp\Big\{ - \ell_{h}(x_h, a_h) + \sum_{x_{h+1} \in \cC(x_ha_h)} F_{x_{h+1}}(\ell) \Big\}} \\
    & = \sum_{a_h \in \cA_{x_h}} \mu_h(a_h \vert x_h) [e_{x_ha_h}+\sum_{x_{h+1} \in \cC(x_ha_h)} (-\grad F_{x_{h+1}})(\ell)].
\end{align*}

For any $x_h$, repeat the above process along the treeplex, the contribution of the production of $\mu(\cdot \vert \cdot)$ will be the sequence form. As a result, $-\grad F^{\mc{V}}(\ell) = \mu$, which completes the proof.
\end{proof}

\begin{lemma}
\label{lem:dilated-ent-optimization}
The policy $\mu^t$ in Algorithm~\ref{algorithm:vertex-ftrl} is the optimizer of the optimization problem
$$
\textstyle
\argmin_{\mu \in \Pi^{x_h}} \brac{ \eta \<\mu, \sum_{s=1}^{t-1}\ell^{s}\> + H_{x_h}(\mu) } 
$$
for all $x_h$ simultaneously. Furthermore, 
$$
\textstyle
\min_{\mu \in \Pi^{x_h}} \brac{ \eta \<\mu, \sum_{s=1}^{t-1}\ell^{s}\> + H_{x_h}(\mu) } = -F_{x_h}\paren{\sum_{s=1}^{t-1}\ell^{s}} .
$$
\end{lemma}

The result is known in the literature (e.g.~\cite{kroer2015faster}) and its proof is similar to Appendix B of~\cite{kozuno2021model}, which focuses on the special case when the loss function is the bandit-based loss estimator~\eqref{eq:bandit-loss-estimator}: 
\begin{align*}
 \ell_h^{t}(x_h, a_h) \defeq \frac{\indic{(x_h^t, a_h^t)=(x_h, a_h)} (1 - r_h^t)}{\mu_{1:h}^t (x_h, a_h)+\gamma} .
\end{align*}

For completeness, we include a proof for generic loss here.

\begin{proof}
We prove by induction for $h=H,\cdots,1$. For $h=H$, since there is no further decisions to be make, this is just a linear optimization problem with entropy regularizer on simplex. As a result, $\mu_H(a_H \vert x_H) \propto_{a_H}\exp\{  - \ell_{H}(x_H, a_H) \}$ as desired and the minimum is $- \log \sum_{a_H} \exp\Big\{  - \ell_{H}(x_H, a_H) \Big\} =-F_{x_H}(\ell)$.

If the claim holds for levels after $h+1$, consider the $h$-th level. Plug in the optimizer after the $h+1$-th level, the optimization problem in the sub-tree rooted at $x_h$ becomes
$$
\argmin_{\mu \in \Pi^{x_h}} \brac{ \sum_{a_h}\mu_h(a_h|x_h)(\ell_{h}(x_h,a_h)-\sum_{x_{h+1} \in \cC(x_ha_h)} F_{x_{h+1}}(\ell)) } ,
$$
which is again a linear optimization problem with entropy regularizer on simplex. As a result, $\mu_h(a_h \vert x_h) \propto_{a_h}\exp\{  - \ell_{h}(x_h, a_h)+\sum_{x_{h+1} \in \cC(x_ha_h)} F_{x_{h+1}}(\ell)) \}$ as desired and the minimum is $- \log \sum_{a_h} \exp\Big\{  - \ell_{h}(x_h, a_h)+\sum_{x_{h+1} \in \cC(x_ha_h)} F_{x_{h+1}}(\ell)) \Big\} =-F_{x_h}(\ell)$.
\end{proof}

\subsection{Equivalence between OMD implementation and ``Kernelized MWU'' implementation of~\citet{farina2022kernelized}}
\citet{farina2022kernelized} design another efficient implementation of the Vertex MWU algorithm~\eqref{equation:vertex-ftrl} via the ``kernel trick'', which they term as the Kernelized MWU algorithm (Algorithm 1 in~\citep{farina2022kernelized}\footnote{Their Algorithm 1 is an optimistic algorithm with a ``prediction vector''. Here we are referring to their non-optimistic version where the prediction vectors are set to zero.}). As Theorem~\ref{theorem:equivalence-ftrl} shows that Vertex MWU is equivalent to standard OMD, the Kernelized MWU algorithm is also equivalent to standard OMD.

In this section, we further show that the implementation in Kernelized MWU is also ``equivalent'' to the standard linear-time implementation of OMD (Algorithm~\ref{algorithm:vertex-ftrl}), by showing that the key intermediate quantities in both implementations are also equivalent.

Since the notation used in \cite{farina2022kernelized} is slightly different from ours, we first describe their key intermediate quantities using our notation. Their exponential weight $b^t \in \R^{XA}$ is defined by 
$$
b^t(x_h,a_h) = \exp\{-\eta \sum_{s=1}^t\ell_h^s(x_h,a_h)\}.
$$
Then, their kernel function $K:\R^{XA} \times \R^{XA} \rightarrow \R$ is defined by 
$$
K_{x_g}(b,b') = \sum_{v\in \mc{V}^{x_g}} \sum_{(x_h,a_h)\in v}b(x_h,a_h)b'(x_h,a_h),
$$
where $(x_h,a_h)\in v$ is a shorthand notation meaning that $(x_h, a_h)$ is such that $v_{1:h}(x_h,a_h)=1$.

We will also use $\ones\in\R^{XA}$ to denote the all-ones vector in $\R^{XA}$. \citep[Proposition 5.3]{farina2022kernelized}~shows that the output policy $\mu^t$ of kernelized OMWU can be written in conditional-form as 
$$
\mu^t(a_h|x_h) = \frac{b^{t-1}(x_h,a_h)\prod_{x_{h+1} \in \cC(x_ha_h)}K_{x_{h+1}}(b^{t-1},\ones)}{K_{x_h}(b^{t-1},\ones)}.
$$
The key step within \citet{farina2022kernelized}'s Kernelized MWU implementation is the recursive evaluation of the quantity $K_{x_h}(b^{t-1}, \ones)$ in the bottom-up order over $x_h\in\cX$, whereas our Algorithm~\ref{algorithm:vertex-ftrl}'s key step is the recursive evaluation of $F_{x_h}^t$ in the bottom-up order over $x_h\in\cX_h$ in~\eqref{equation:fxht}. 

The following proposition shows that these two quantities are exactly equivalent, thereby showing the equivalence of the two implementations. 
\begin{proposition}
\label{proposition:intermediate-quantity-equivalence}
We have for all $x_h\in\cX$ and all $t\ge 1$ that
\begin{align*}
    K_{x_h}(b^{t-1},\ones) = \exp\{F^{t}_{x_h}\}.
\end{align*}
\end{proposition}
\begin{proof}
We prove this by induction for $h=H+1,\cdots, 1$. For $h=H+1$, $K_{x_h}(b^{t-1},\ones) = 1$ and $F^{t}_{x_h} = 0$ by definition. If the claim holds for $h+1$, then by Theorem 5.2 of \cite{farina2022kernelized}, \begin{align*}
 K_{x_h}(b^{t-1},\ones) =& \sum_{a_h}\exp\{-\eta \sum_{s=1}^{t-1} \ell^s(x_h,a_h)\}\prod_{x_{h+1} \in \cC(x_h,a_h)} K_{x_{h+1}}(b^{t-1},\ones) \\ 
 =& \sum_{a_h}\exp\{-\eta \sum_{s=1}^{t-1} \ell^s(x_h,a_h)+\sum_{x_{h+1} \in \cC(x_h,a_h)} F_{x_{h+1}}^t\}\\
=& \exp\{F_{x_{h}}^t\}.
\end{align*}
\end{proof}

\end{document}